\documentclass{article}
\ifdefined\pdfsuppressptexinfo\pdfsuppressptexinfo=-1\fi
\ifdefined\pdfinfoomitdate\pdfinfoomitdate=1\fi
\usepackage{arxiv}
\usepackage{times}
\usepackage{natbib}

\usepackage{amsmath,amsfonts,bm}

\usepackage{color, xspace}

\def\eqref#1{equation~\ref{#1}}

\def\1{\bm{1}}

\DeclareMathAlphabet{\mathsfit}{\encodingdefault}{\sfdefault}{m}{sl}
\SetMathAlphabet{\mathsfit}{bold}{\encodingdefault}{\sfdefault}{bx}{n}

\DeclareMathOperator{\sign}{sign}

\usepackage{amsmath,amssymb,amsthm}
\usepackage{booktabs,tabularx,longtable,array,calc,colortbl}
\usepackage{graphicx}
\usepackage{xcolor}
\usepackage{microtype}
\usepackage{enumitem}

\usepackage{tikz}
\usetikzlibrary{arrows.meta,positioning}
\usepackage{placeins}
\usepackage{flafter}
\usepackage[hidelinks]{hyperref}
\hypersetup{pdftitle={Persona Following Is Not Selective Control: The Neutrality Gap in LLM User Simulation},pdfauthor={Jiashen Ren, Wenlin Zhang, Bohan Zhang, Xiaopeng Li, Zichuan Fu, Wanyu Wang, Junyi Li, Xiangyu Zhao}}
\usepackage{url}
\DeclareUnicodeCharacter{2212}{\ensuremath{-}}
\DeclareUnicodeCharacter{00D7}{\ensuremath{\times}}
\DeclareUnicodeCharacter{00B1}{\ensuremath{\pm}}
\DeclareUnicodeCharacter{2264}{\ensuremath{\le}}
\DeclareUnicodeCharacter{2265}{\ensuremath{\ge}}
\DeclareUnicodeCharacter{2019}{'}
\DeclareUnicodeCharacter{201C}{``}
\DeclareUnicodeCharacter{201D}{''}

\graphicspath{{figs/}}
\newcolumntype{Y}{>{\raggedright\arraybackslash}X}
\newcommand{\logit}{\operatorname{logit}}
\newcommand{\doP}{\operatorname{do}_{\mathcal P}}
\newcommand{\modelid}[1]{#1}

\theoremstyle{definition}
\newtheorem{defn}{Definition}
\theoremstyle{plain}
\newtheorem{prop}{Proposition}
\newtheorem{assumption}{Assumption}
\theoremstyle{remark}
\newtheorem{remark}{Remark}

\definecolor{RiskRed}{HTML}{E66D50}
\definecolor{ScopeTeal}{HTML}{299D8F}
\definecolor{OrdinaryAmber}{HTML}{E7C66B}
\definecolor{BoundedSlate}{HTML}{297270}
\definecolor{DiDBlue}{HTML}{274753}
\definecolor{DiDPurple}{HTML}{8AB07C}
\definecolor{ControlGray}{HTML}{9A9C9A}
\definecolor{TextGray}{HTML}{274753}

\title{Persona Following Is Not Selective Control:\\The Neutrality Gap in LLM User Simulation}

\author{Jiashen Ren$^{1}$, Wenlin Zhang$^{1}$, Bohan Zhang$^{2}$, Xiaopeng Li$^{1}$,\\
\bfseries Zichuan Fu$^{1}$, Wanyu Wang$^{1}$, Junyi Li$^{1}$, Xiangyu Zhao$^{1,*}$\\[3pt]
{\normalfont\small $^{1}$City University of Hong Kong\quad $^{2}$Stanford University}\\[2pt]
{\normalfont\small \texttt{jiashen.ren@my.cityu.edu.hk}\quad \texttt{xy.zhao@cityu.edu.hk}}\\
{\normalfont\small $^{*}$Corresponding author}}
\renewcommand{\shorttitle}{Persona Following Is Not Selective Control}
\date{}

\begin{document}

\maketitle

\begin{abstract}
Persona prompting is widely used in LLM user simulation, yet it relies on a largely untested assumption. Specifying a target user attribute should shift target behavior while leaving other attributes held fixed in the experiment unchanged. Across eight black-box and five open-weight models, however, changing a target attribute alone also shifts unspecified non-target choices, such as risk preference shifting color choice. We ask whether explicitly specifying non-target attributes restores selective control across $477$ held-out synthetic worlds and $200$ independent items under unspecified, directional, and neutral states. Assigning a clear non-target direction nearly eliminates cross-attribute shifts. Explicitly declaring no preference, however, still leaves $37.7\%$--$76.3\%$ of held-out worlds and $51\%$--$81\%$ of independent non-target items shifting with the target attribute while preserving target efficacy. We call this disparity the \emph{neutrality gap}. Because directional tests can be passed simply by following a stated preference, the neutral state reveals failures of selective control that directional tests miss. Producing behavior consistent with a stated persona therefore does not imply keeping unrelated choices invariant. In simulated investment decisions, attaching incidental brochure color labels shifts estimated target risk treatment effects by $25.5$ percentage points, even when models accurately report having no color preference. Counterfactual attention interventions show that unwanted influence enters primarily at query time, when option evaluation reads the target persona directly. Blocking this query-time attention cuts non-target drift by $72.1\%$--$87.8\%$, whereas directional declarations suppress drift primarily by overwriting option values rather than rerouting attention. Because target and non-target choices share query-time access, blocking it weakens target efficacy while value overwrite imposes an artificial direction, leaving task-dependent selective access as a central challenge.
\end{abstract}

\section{Introduction}\label{sec:intro}

Large language models (LLMs) are widely used to simulate users with specific attributes, identities, and behavioral dispositions \citep{silicon,dillion2023replace,park2023generativeagents,impersonation} in economic decision-making \citep{homosilicus,turingexperiments} and social science or behavioral studies \citep{sociodemprompting,personallm,hewitt2026predicting,hullman2026simulation}. In these settings, researchers manipulate a target attribute through persona prompting and interpret subsequent behavioral changes as effects of that attribute \citep{hu2024personaeffect,psychometric2025,hartley2025personality}. This interpretation requires \textbf{selective persona control}: changing the target attribute should affect the intended target behavior without altering other attributes that the experiment aims to hold fixed \citep{chester2021manipulation,gui2023causalsim}. Otherwise, if the same manipulation also alters other attributes that affect behavior, observed differences mix effects from multiple attributes and cannot be cleanly attributed to the target attribute alone.

Specifying only the target attribute in a prompt does not guarantee that other attributes stay unchanged. Because LLMs learn statistical associations among concepts and user traits from training data and context, a stated attribute can shift responses on other unspecified attributes \citep{impersonation,whoseopinions,personabias,liu2024personasteered}. When a prompt changes only a user's risk preference, for example, the model's choice of phone-case color can shift as well (Figure~\ref{fig:model}a). Such cross-attribute shifts are expected when the non-target attribute is left unspecified and do not violate any control instruction. Prior work has documented trait correlations, cross-trait interference under activation steering, and the effects of user backstories on simulated behavior \citep{psychometric2025,moon2024anthology,kang2025deepbinding,feng2026persona,bhandari2026traitsinterfere,luzdearaujo2025principled}. Recent studies show that treatments in LLM simulations can induce user drift in unspecified attributes, confounding effect estimates \citep{gui2023causalsim,lin2026illusion}. These studies propose explicitly specifying relevant covariates to mitigate this drift. Explicit attribute specification offers a natural control strategy, yet simulated populations include individuals with directional preferences as well as neutral individuals with no particular preference. Covariate specification therefore protects attribution only if it works for both groups. When indifferent individuals are present, residual drift in their choices contaminates non-target comparisons. This drift makes estimated treatment effects sensitive to incidental presentation cues, even when a self-report manipulation check passes. The core question is thus not whether unspecified attributes shift. \textbf{The key unresolved issue is whether explicit non-target specifications, especially declarations of no preference, can stabilize non-target choices without weakening target efficacy.}

\begin{figure*}[!htb]
\centering
\includegraphics[width=\textwidth]{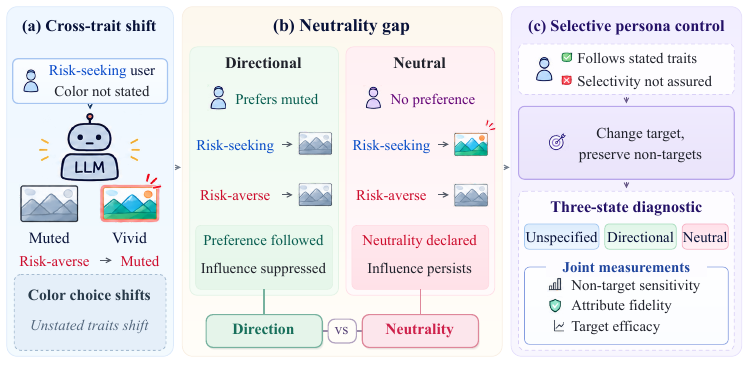}
\caption{\textbf{Evaluating selective persona control across three information states.}
\textbf{(a)} When the non-target attribute is left unspecified, changing the target attribute can also change non-target choices.
\textbf{(b)} Specifying a directional preference suppresses cross-attribute shifts more effectively than specifying no preference, revealing a \emph{neutrality gap}.
\textbf{(c)} Our three-state evaluation separates shifts in unspecified attributes from explicit persona control and distinguishes selective persona control from persona following.
Highlighted choices indicate qualitative tendencies only.}
\label{fig:model}
\end{figure*}

\textbf{Persona following does not equal selective persona control.} A model can faithfully portray the described user while still shifting attributes outside the intended scope, and it may pass a directional preference test simply by choosing the preferred option. Explicitly stating no preference removes this shortcut and lets us test whether non-target choices remain stable when the target attribute changes. Existing studies on persona prompting, activation steering, and covariate control typically evaluate assigned trait directions, steered trait interference, or concrete covariate values \citep{personallm,sociodemprompting,psychometric2025,bhandari2026traitsinterfere,luzdearaujo2025principled,gui2023causalsim,lin2026illusion}. Invariance under explicit neutrality remains largely untested. A related limitation arises in behavioral diagnosis: even when a model accurately restates a specified persona state, that state does not necessarily constrain subsequent choices \citep{hauser2018manipulation,chester2021manipulation}. Checking only whether the model recalls the stated persona cannot verify choice stability. Furthermore, self-reports cannot distinguish whether residual shifts arise within the prompt prefix or when the model reads the target attribute to evaluate the options.

We address these questions with a three-state evaluation of selective persona control across $477$ held-out synthetic worlds and $200$ independently authored items spanning four attribute pairs (Figure~\ref{fig:model}b--c). Holding question and options fixed, we evaluate three non-target states: unspecified, directional preference, and explicit neutrality. Across these states, we jointly measure non-target choice stability, directional adherence, and retention of the intended target effect. Comparing these states reveals a pronounced \textbf{neutrality gap}. Directional preferences reduce above-threshold non-target shifts to $0\%$--$3.6\%$ of held-out worlds, whereas explicit statements of no preference still leave $37.7\%$--$76.3\%$ of worlds shifting when the target attribute changes. Neutrality does not require an even choice split. It requires only that the non-target choice distribution remain invariant when the target attribute varies. Comparing preference reports with binary choices shows that models accurately report neutrality yet shift choices, distorting treatment effects across incidental cues. Internal interventions trace this failure primarily to query-time attention. Directional declarations suppress drift through value overwrite, whereas blocking query-time access impairs target efficacy.

Our main contributions are as follows:

\noindent\textbf{(1) Three-state evaluation, the neutrality gap, and downstream bias.} We introduce a three-state evaluation that separates shifts in unspecified attributes from explicit directional and neutral control while jointly tracking target efficacy. Using this evaluation, we identify a widespread neutrality gap that biases non-target comparisons and distorts estimated treatment effects by $25.5$ percentage points in simulated investment decisions.

\noindent\textbf{(2) Behavioral and mechanistic isolation of the neutrality failure.} We establish a persistent dissociation between accurate neutral self-reports and shifting binary choices, showing that standard manipulation checks do not guarantee choice invariance. Through counterfactual attention interventions, we trace this failure primarily to query-time reading of the target persona during option evaluation and show how directional value signals override it.

\noindent\textbf{(3) Control boundaries and the shared-pathway trade-off.} We show why standard prompt constraints and rating-scale subspace erasure fail to restore neutral control in discrete binary decisions. Our interventions reveal that non-target drift and target control rely on the same query-time pathway, establishing task-dependent selective access as a core requirement for controlled simulation.

\section{Selective Persona Control and Experimental Setup}
\label{sec:framework}

\subsection{Selective Persona Control}
\label{sec:framework:selective}

Testing whether a target persona manipulation acts selectively requires more than checking whether non-target choices remain stable. Consider risk preference and color preference: a model that ignores the risk description altogether would also produce identical answers on color questions. We therefore measure non-target choice stability alongside the target attribute's effect on independent target decisions. In Table~\ref{tab:example_tasks}, both tasks share the same persona of high versus low risk preference. The non-target color preference is left unstated, specified directionally, or stated as neutral. Verbatim prompts appear in Appendices~\ref{app:prompts:battery} and~\ref{app:prompts:tasks}.

\begin{table}[!htb]
\centering
\caption{Two independent tasks for evaluating selective persona control.}
\label{tab:example_tasks}
\vspace{6pt}
\footnotesize
\setlength{\tabcolsep}{7pt}
\renewcommand{\arraystretch}{1.25}
\begin{tabularx}{\linewidth}{Y Y}
\toprule
\multicolumn{1}{c}{\textbf{Task A $\cdot$ Non-Target Choice}} & \multicolumn{1}{c}{\textbf{Task B $\cdot$ Target Behavior}} \\
\multicolumn{1}{c}{\textcolor{gray}{\scriptsize Phone Case Color}} & \multicolumn{1}{c}{\textcolor{gray}{\scriptsize Regional Sales Risk}} \\
\cmidrule(lr){1-1}\cmidrule(lr){2-2}
\emph{Two cases have the same price, fit, and grip; both are in the drawer. Which would you use?} &
\emph{Two regions have the same expected sales volume. Which region do you take on?} \\[3pt]
\texttt{1)} Vivid tangerine \quad \texttt{2)} Muted pewter &
\texttt{1)} Variable sales \quad \texttt{2)} Stable sales \\
\midrule
\rowcolor{gray!12}
\multicolumn{1}{c}{Tests whether flipping risk shifts color choice} &
\multicolumn{1}{c}{Tests whether flipping risk shifts risk choice} \\
\bottomrule
\end{tabularx}
\end{table}

For a given non-target question, we hold item text, options, and non-target condition fixed while flipping the target attribute. Under explicit non-target specifications, choices on Task A should remain unchanged while decisions on Task B continue to respond to the target manipulation. Evaluating both tasks under the same persona prefix separates genuine selective control from two opposite failure modes. A manipulation that weakens Task B sacrifices target efficacy, whereas a manipulation that shifts Task A allows the target persona to spill over into non-target decisions.

\subsection{Three-State Evaluation}
\label{sec:framework:threestate}

To separate associations among unstated attributes from control under explicit specification, we compare three non-target conditions $c$ under identical target manipulations, questions, and options. In the \textbf{Bare} condition, the prompt states only the target attribute, such as preferring variable over stable sales, and leaves color preference unstated. The \textbf{Directional} condition appends a concrete preference pole, such as \emph{``The user strongly prefers muted, understated finishes,''} to test whether a stated direction stabilizes choices. Finally, the \textbf{Neutral} condition states no preference, such as \emph{``The user has no consistent leaning in color saturation,''} testing invariance without a prescribed pole. Whereas a model can pass a directional test simply by picking the named pole, explicit neutrality removes that shortcut and directly tests choice-distribution invariance. \textbf{Neutral does not mean a 50/50 split between options.} A user with no preference between tangerine and pewter need not pick each half the time. What matters is whether the choice distribution changes when risk preference flips from high to low.

Under any non-target condition $c$, let $T \in \{-1, +1\}$ denote the target levels for question $q$ in decision world $w$. For model $\theta$'s normalized two-option choice distribution $\pi_\theta$, we define \textbf{non-target sensitivity} on Task A as the total variation distance:
\begin{equation}
\label{eq:sensitivity}
S_\theta(c, q, w) = \operatorname{TV}\!\left(\pi_\theta(\cdot \mid T{=}{+}1, c, q, w),\; \pi_\theta(\cdot \mid T{=}{-}1, c, q, w)\right).
\end{equation}
We compute total variation separately across option orderings and label encodings before averaging, preserving distributional shifts within each format rather than letting opposite probability differences cancel out. Ideal choice-distribution invariance corresponds to $S_\theta(c, q, w) = 0$, and smaller values indicate greater stability against target changes.

Low sensitivity $S_\theta$ alone does not establish selective control, because a model could simply ignore the non-target statement and default to one fixed option. To verify directional adherence, we measure \textbf{directional fidelity error} as the total variation distance between the model's answer distribution and the option matching the stated non-target preference. Because neutral statements prescribe no option direction, this metric applies only to directional conditions. Sections~\ref{sec:declare} and~\ref{sec:bare} separately test via preference reports whether models register neutral statements themselves. To verify that non-target stability does not stem from suppressing the target persona, we measure \textbf{target efficacy} on independent target tasks. We define target efficacy as the signed shift toward the target-aligned option when $T$ flips, and \textbf{target retention} as its ratio under $c$ to the Bare baseline. Jointly evaluating non-target stability, directional fidelity, and target retention separates selective control from stability achieved by ignoring the persona or defaulting to one choice.

\subsection{Experimental Setup}
\label{sec:setup}

No single experiment provides naturalistic survey items, controlled reference rules, and unseen attribute pairs at once. Table~\ref{tab:designs} therefore combines three behavioral suites so each addresses what the preceding design leaves open. First, the \textbf{Field audit} measures unstated cross-attribute associations on everyday choices across eight black-box API models. Second, \textbf{Structured worlds} pair $40$ item stems with sampled utility rules where the target attribute has zero reference effect on non-target choices. Across five $24$B--$32$B open-weight models, $162$ validation worlds fix each model's prompt wrapper before evaluation on $477$ held-out worlds $w$. Third, \textbf{Independent items} introduce $200$ unseen stems across four attribute pairs, with $40$ non-target and $10$ target items per pair. This suite substitutes \modelid{Qwen3.5-9B} for \modelid{Qwen2.5-32B} to test a newer, smaller checkpoint alongside unseen items (Appendix~\ref{app:suite}). Within this suite, \texttt{risk\,$\to$\,color} tests the primary high-versus-low intensity contrast, whereas \texttt{time\,$\to$\,sound} tests cross-attribute coupling outside risk and visual aesthetics. Meanwhile, \texttt{time\,$\to$\,layout} removes intensity polarity by contrasting horizontal and vertical layouts, and \texttt{risk\,$\to$\,hue} pairs saturation-matched red and blue hues to isolate intensity matching.

\begin{table}[!htb]
\centering
\caption{Overview of the three behavioral evaluation suites.}
\label{tab:designs}
\vspace{5pt}
\footnotesize
\setlength{\tabcolsep}{5pt}
\renewcommand{\arraystretch}{1.24}
\newcommand{\cellstack}[2]{\begin{tabular}[t]{@{}l@{}}#1\\[1pt]\textcolor{gray}{\scriptsize #2}\end{tabular}}
\newcommand{\cellstackc}[2]{\begin{tabular}[t]{@{}c@{}}#1\\[1pt]\textcolor{gray}{\scriptsize #2}\end{tabular}}
\resizebox{\linewidth}{!}{%
\begin{tabular}{l l c c c c}
\toprule
& \multicolumn{3}{c}{\textbf{Experimental Design \& Scale}} & \multicolumn{2}{c}{\textbf{Model Panel \& Evaluation Scope}} \\
\cmidrule(lr){2-4}\cmidrule(lr){5-6}
\textbf{Suite} & \textbf{Attribute Pairs} & \textbf{Primary Scale} & \textbf{States $c$} & \textbf{Backbones} & \textbf{Control \& Split Scope} \\
\midrule
\textbf{Field audit} &
  \cellstack{\texttt{risk\,$\to$\,style}}{everyday choices} &
  \cellstackc{$16$ non-tgt, $30$ tgt}{everyday survey items} &
  \cellstackc{Bare only}{unstated $c$} &
  \cellstackc{$8$ black-box}{$20$ samples/item} &
  \cellstackc{Naturalistic survey}{no explicit $c$ control} \\
\addlinespace[2pt]
\textbf{Structured worlds} &
  \cellstack{\texttt{risk\,$\to$\,color}, \texttt{time\,$\to$\,sound}}{two primary attribute pairs} &
  \cellstackc{$477$ held-out worlds}{$162$ validation worlds} &
  \cellstackc{$3$-state}{Bare, Dir., Neutral} &
  \cellstackc{$5$ open-weight}{$24$B--$32$B, exact prob.} &
  \cellstackc{Known zero ref.\ effect}{splits $w$, $40$ item stems} \\
\addlinespace[2pt]
\textbf{Independent items} &
  \cellstack{$+$\,\texttt{risk\,$\to$\,hue}, $+$\,\texttt{time\,$\to$\,layout}}{equal-salience red/blue, horiz./vert.} &
  \cellstackc{$160$ non-tgt, $40$ tgt}{$200$ blind-written} &
  \cellstackc{$3$-state}{Bare, Dir., Neutral} &
  \cellstackc{$5$ open-weight}{\texttt{Qwen3.5-9B} subst.} &
  \cellstackc{Unseen item stems}{family-clustered bootstrap} \\
\midrule
\rowcolor{gray!12}
\textbf{Held-out rule} &
  \multicolumn{2}{c}{Sensitive world: $S_\theta(c,q,w) > 0.15$ \textcolor{gray}{\scriptsize $\cdot$ pass $\le 5\%$}} &
  \multicolumn{1}{c}{Fidelity $\le 0.10$} &
  \multicolumn{2}{c}{Target retention $\ge 0.90$ \textcolor{gray}{\scriptsize $\cdot$ base efficacy $\ge 0.10$}} \\
\bottomrule
\end{tabular}%
}
\end{table}

In formal held-out evaluation, we mark a comparison as sensitive to the target attribute whenever $S_\theta(c, q, w) > 0.15$. A prompt design passes held-out evaluation if at most $5\%$ of worlds are sensitive and mean directional fidelity error stays at or below $0.10$ under directional conditions. Across worlds where baseline target efficacy is at least $0.10$, target retention must also reach at least $0.90$ (Appendices~\ref{app:suite} and~\ref{app:c4}). Together, these three suites test whether explicit non-target statements preserve choice-distribution invariance across models, decision environments, and unseen items.

\section{The Neutrality Gap}\label{sec:declare}

\subsection{Explicit Neutrality Leaves Residual Drift}

When non-target preferences are left unstated, flipping a target attribute shifts unrelated everyday choices across all eight black-box API models and five open-weight models (Bare baseline; Appendix~\ref{app:c1}). Explicitly stating no preference, however, does not stabilize these non-target choices the way specifying a concrete preference direction does. Across the $477$ held-out worlds, specifying either directional pole reduces worlds exceeding the $0.15$ sensitivity threshold to $0\%$--$3.6\%$, satisfying the formal $\le 5\%$ tolerance on all five models. Under stated neutrality, every model violates this requirement: $37.7\%$--$76.3\%$ of worlds shift their non-target choices as the target attribute flips (Figure~\ref{fig:neutral}a). The same asymmetry holds across the $160$ independent items in Figure~\ref{fig:neutral}b. Directional statements across both poles exceed the threshold in at most $1$ of $320$ item--pole comparisons per model, whereas neutral statements leave $51\%$--$81\%$ of items above the threshold. This residual drift does not point randomly across options. Pooled across models wherever Bare and Neutral prompts both exceed $0.15$, shift directions agree in $1{,}237$ of $1{,}240$ model--world and $419$ of $447$ model--item pairs (Appendix~\ref{app:c4}). Models reliably follow a stated preference direction, yet fail to block cross-attribute influence under stated neutrality, a control failure we call the \emph{neutrality gap}.

\begin{figure}[!htb]\centering
\includegraphics[width=\textwidth]{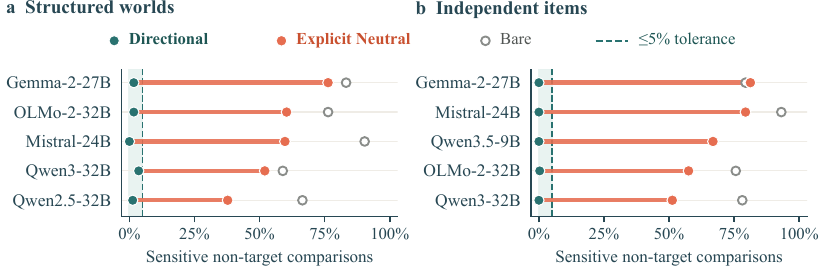}
\caption{\textbf{Explicit neutrality fails to stabilize non-target choices.}
Share of non-target comparisons exceeding the $0.15$ threshold under Directional, Explicit Neutral, and Bare prompts across \textbf{(a)} $477$ held-out Structured worlds and \textbf{(b)} $160$ Independent items. Directional points show two-pole union rates in \textbf{(a)} and $320$ pooled item--pole comparisons in \textbf{(b)}; intermediate strengths and survey baselines appear in Appendices~\ref{app:c1} and~\ref{app:c4}.}
\label{fig:neutral}
\end{figure}

The neutrality gap cannot be reduced to matching semantic intensity between risk seeking and vivid color, though its severity varies across attribute pairs. Under a time-preference target, neutral prompts still leave most sound-feedback and horizontal-versus-vertical layout items above the $0.15$ threshold, even though layout carries no intensity contrast. On the saturation-matched red-versus-blue hue items, neutral-condition drift diverges sharply across models. Specifically, \modelid{OLMo-2-32B} and \modelid{Qwen3-32B} have almost no items above the threshold, whereas \modelid{Qwen3.5-9B} and \modelid{Gemma-2-27B} still shift on over four-fifths of the $40$ items (Appendices~\ref{app:c4} and~\ref{app:c5}). The neutrality gap is thus neither confined to a single high-versus-low intensity contrast nor uniform across models and attribute pairs. Together with the directional agreement between Bare and Neutral shifts, these results show that explicit neutrality fails to remove pre-existing cross-attribute associations.

\subsection{Drift Suppression Does Not Always Preserve Target Control}

Under explicit neutrality, models do not weaken their response to the target attribute, yet that response continues to shift non-target choices that should remain stable. Across $477$ held-out worlds and $40$ independent target items with baseline efficacy $\ge 0.10$, five-model mean target retention under neutral and positive directional prompts stays near $1.00$ (Table~\ref{tab:neutral_retention}). Every model exceeds the $0.90$ standard across both evaluation suites (Appendix~\ref{app:c4}). The failure thus lies not in an ineffective target manipulation, but in failing to confine the target attribute's effect to the target task.

\begin{table}[!htb]\centering\footnotesize
\providecommand{\ms}[2]{$#1_{\textcolor{gray}{#2}}$}
\providecommand{\bms}[2]{$\mathbf{#1}_{\textcolor{gray}{#2}}$}
\caption{Target efficacy retention under explicit neutrality and coded positive or negative non-target poles across each suite's five-model panel, showing means with minimums in gray subscripts.}
\label{tab:neutral_retention}
\vspace{4.5pt}
\renewcommand{\arraystretch}{1.18}
\setlength{\tabcolsep}{7pt}
\begin{tabularx}{\textwidth}{l l *{3}{>{\centering\arraybackslash}X}}
\toprule
& & \multicolumn{3}{c}{\textbf{Target Retention} $\cdot$ $\text{Mean}_{\textcolor{gray}{\min}}$} \\
\cmidrule(lr){3-5}
\textbf{Evaluation Suite} & \textbf{Target Task Scope} & \textbf{+Dir} & \textbf{Neutral} & \textbf{$-$Dir} \\
\midrule
\textbf{Structured worlds} & $477$ held-out worlds & \ms{1.002}{.998} & \ms{1.002}{1.000} & \ms{.942}{\mathbf{.822}} \\
\textbf{Independent items} & $40$ new target items & \ms{.994}{.984} & \ms{.992}{.977} & \bms{.885}{\mathbf{.765}} \\
\bottomrule
\end{tabularx}
\end{table}

Conversely, the apparent non-target stability achieved by directional statements sometimes comes at the cost of target control. Under negative directional statements, where the sign denotes the coded preference pole, non-target sensitivity stays near zero, yet target efficacy drops on several models (Table~\ref{tab:neutral_retention}). On the held-out worlds, \modelid{Qwen3-32B}'s target retention falls to $0.822$, below the $0.90$ standard. The independent target tasks show similar losses, because stating a preference for muted colors also suppresses high-risk choices on the regional sales task. Suppressing non-target drift and preserving target efficacy are thus distinct control requirements. Non-target stability alone does not show that a persona manipulation achieves selective control.

\subsection{Irrelevant Presentation Changes Estimated Target Effects}

Attaching incidental color labels to target options lets the residual drift from Section~3.1 alter estimated target treatment effects. Across $60$ personal investment items, we hold the return distributions of both funds fixed under explicit color neutrality. Prompts state that brochure colors are assigned at random, while we swap only the text color labels between the two funds. We measure the estimated target risk effect $\hat{\tau}$ as the difference in high-volatility fund choice probability between high-risk and low-risk personas. When the high-volatility fund carries a vivid color label and the stable fund carries a muted label, five-model mean $\hat{\tau}$ is $0.950$. Reversing the color pairing reduces this estimate to $0.695$, leaving a $25.5$ percentage-point gap caused by an irrelevant presentation label (Figure~\ref{fig:presentation}a).

\begin{figure}[!htb]\centering
\includegraphics[width=\textwidth]{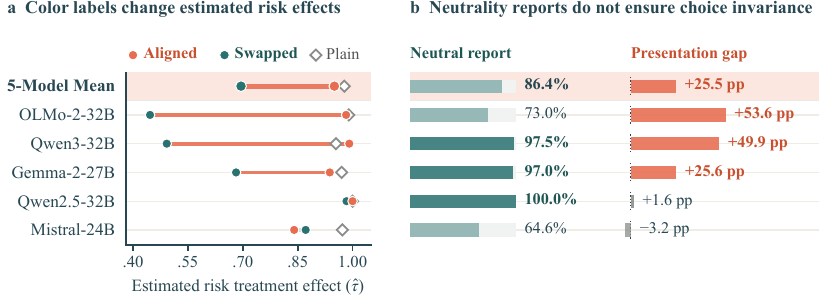}
\caption{\textbf{Incidental color labels alter estimated target treatment effects under explicit color neutrality.}
\textbf{(a)} Estimated target risk effect $\hat{\tau}$ across $60$ investment items when vivid and muted brochure colors are aligned with return volatility, swapped, or omitted.
\textbf{(b)} Neutral report accuracy versus the full presentation treatment-effect gap across all five models; full checks appear in Appendix~\ref{app:c6}.}
\label{fig:presentation}
\end{figure}

This presentation distortion is neither isolated to a single model nor eliminated by counterbalancing. Three of five models show large positive presentation gaps, whereas the other two stay near zero or slightly negative (Figure~\ref{fig:presentation}a--b). Counterbalancing color labels across paired runs cancels antisymmetric gaps on average, yet pulls balanced $\hat{\tau}$ below $90\%$ of the uncolored Plain baseline on four models. Balanced target retention against Bare also drops below $0.90$ on two models. Likewise, across the $477$ held-out worlds with zero reference effects, signed non-target shifts under neutral statements do not cancel out across items. Instead, they produce a positive marked-pole bias of $0.22$--$0.46$ that matches the Bare baseline and far exceeds the $\le 0.016$ directional bias (Appendix~\ref{app:c4}). Together, these two experiments show that cross-attribute drift accumulates systematic non-target bias and makes estimated target treatment effects depend on superficial presentation.

The distortion persists even when models correctly identify the investment attributes and recognize that brochure colors are irrelevant. On all $60$ items, models identify the higher-volatility fund with $99.15\%$ accuracy and the higher expected return with $89.22\%$ accuracy across color pairings. Large presentation gaps remain both on a $20$-item equal-expected-value subset and under explicit instructions to ignore brochure color (Appendix~\ref{app:c6}). Comparing neutrality reports with presentation gaps in Figure~\ref{fig:presentation}b confirms that accurate reporting does not ensure choice invariance. In particular, \modelid{Qwen3-32B} and \modelid{Gemma-2-27B} report color neutrality with at least $97.0\%$ accuracy and show zero color shift in perceived financial risk. Yet both models still exhibit large presentation gaps in downstream investment choice. Identifying a neutral specification in self-reports therefore does not ensure that it constrains decisions.

\section{Why Explicit Neutrality Fails: Query-Time Target Access}\label{sec:bare}

\subsection{Models Recognize Neutrality Yet Read Target Personas at Query Time}

Models can recognize an explicit neutral preference without keeping their choices neutral. Given the same persona-and-option prefix, we ask either which preference state the prompt states or which candidate option the user would choose. Models pair at least $90\%$ report accuracy with choice sensitivity above $0.15$ in $515$ of $800$ independent-item and $161$ of $240$ structured-stem evaluations. These totals pool five models across $160$ items and all six checkpoints across $40$ stems, with zero inaccurate-yet-stable cases (Appendix~\ref{app:c5}). Because preference reports and binary choices read identical prefix activations before the task prompt, neutrality remains intact in the encoded prefix. Unwanted cross-attribute influence therefore enters when option decisions access that prefix.

Where does this unwanted influence enter? Across the five $24$B--$32$B models, Table~\ref{tab:mech_kv} pairs two counterfactual interventions on $80$ length-matched independent non-target items, $20$ per pair. Its left columns locate where target information enters the choice, whereas its right columns test how directional preferences suppress it. Under neutrality, \emph{Path~A} carries target information into the neutral clause via neutral-to-target attention during persona encoding. Alternatively, \emph{Path~B} reads the target clause directly from option and decision tokens during option comparison. Masking Path~A fails to reduce drift reliably and increases it on three models. By contrast, masking Path~B cuts pooled non-target $\mathrm{TV}$ by $72.1\%$--$87.8\%$ across all five models while keeping neutrality-report accuracy between $97.7\%$ and $100.0\%$. Matched random-head and random-span masks produce weak or inconsistent changes, identifying query-time option evaluation as the primary drift pathway (Appendix~\ref{app:c5}).

Hidden-state trajectories and baseline controls clarify how this query-time pull operates. Before candidate options appear, target-induced activation shifts align only weakly with the binary option-choice axis at $0.038$--$0.096$ cosine. Once options are read, alignment surges to $0.638$--$0.950$ on \texttt{risk\,$\to$\,color} items, yet stays lower at $0.223$--$0.574$ where \texttt{risk\,$\to$\,hue} choices remain stable (Appendix~\ref{app:c5}). In unstated baselines, this pull persists under behavioral examples with invented trait labels, tracks pretrained base checkpoints, and survives zero-correlation demonstrations. Opposing in-context correlations redirect Bare drift, whereas directional statements override the supplied correlation (Appendices~\ref{app:c1}, \ref{app:c2}, and~\ref{app:c4}). Option evaluation thus taps broad semantic associations from the target persona at query time, even when neutrality is accurately registered.

\begin{table}[!htb]\centering\footnotesize
\providecommand{\ms}[2]{$#1_{\textcolor{gray}{#2}}$}
\providecommand{\bms}[2]{$\mathbf{#1}_{\textcolor{gray}{#2}}$}
\newcolumntype{C}[1]{>{\hsize=#1\hsize\centering\arraybackslash}X}
\caption{Counterfactual pathway blocking and non-target $2\times 2$ key--value swaps across $80$ non-target items per model, reporting pooled non-target $\mathrm{TV}$ with relative change from Neutral in gray subscripts.}
\label{tab:mech_kv}
\vspace{4.5pt}
\renewcommand{\arraystretch}{1.18}
\setlength{\tabcolsep}{5pt}
\begin{tabularx}{\textwidth}{l C{0.85} C{0.88} C{0.88} C{1.22} C{1.17}}
\toprule
& & \multicolumn{2}{c}{\textbf{Pathway Blocking} $\cdot$ $\mathrm{TV}_{\textcolor{gray}{\Delta\%}}$} & \multicolumn{2}{c}{\textbf{Non-Target $2\times 2$ K/V Swap} $\cdot$ $\mathrm{TV}_{\textcolor{gray}{\Delta\%}}$} \\
\cmidrule(lr){3-4}\cmidrule(lr){5-6}
\textbf{Model} & \textbf{Neutral $\mathrm{TV}$} & \textbf{Path A} & \textbf{Path B} & \textbf{Routing-only} & \textbf{Value-only} \\[-2.5pt]
& & {\scriptsize($\text{Neut.}\!\to\!T$)} & {\scriptsize($\text{Opt./Task}\!\to\!T$)} & {\scriptsize(Dir.\ attn, Neut.\ val)} & {\scriptsize(Neut.\ attn, Dir.\ val)} \\
\midrule
\modelid{Gemma-2-27B} & $.505$ & \ms{.347}{-31.2\%} & \bms{.141}{-72.1\%} & \ms{.367}{-27.3\%} & \bms{.004}{-99.3\%} \\
\modelid{Qwen3-32B}   & $.400$ & \ms{.479}{+19.9\%} & \bms{.049}{-87.8\%} & \ms{.254}{-36.4\%} & \bms{.004}{-99.1\%} \\
\modelid{Mistral-24B} & $.295$ & \ms{.461}{+56.1\%} & \bms{.049}{-83.4\%} & \ms{.279}{-5.6\%}  & \bms{.036}{-87.9\%} \\
\modelid{OLMo-2-32B}  & $.253$ & \ms{.302}{+19.2\%} & \bms{.053}{-79.1\%} & \ms{.201}{-20.9\%} & \bms{.024}{-90.5\%} \\
\modelid{Qwen2.5-32B} & $.188$ & \ms{.429}{+128.0\%} & \bms{.034}{-82.2\%} & \ms{.256}{+35.9\%} & \bms{.024}{-87.3\%} \\
\midrule
\rowcolor{gray!12}\textbf{5-Model Mean} & $\mathbf{.328}$ & \ms{.404}{+23.0\%} & \bms{.065}{\mathbf{-80.2\%}} & \ms{.271}{-17.4\%} & \bms{.018}{\mathbf{-94.5\%}} \\
\bottomrule
\end{tabularx}
\end{table}

\subsection{Directional Preferences Override Cross-Attribute Influence}

Why does specifying a directional preference stabilize choices when specifying neutrality does not? A directional instruction might prevent the model from reading the target persona, or it might supply a stronger preference signal that dominates option comparison. Swapping both keys and values across the non-target clause reproduces the full behavioral shift between neutral and directional prompts at $1.000$ median gap closure (Appendix~\ref{app:c5}). We therefore swap attention weights and value vectors separately in the right columns of Table~\ref{tab:mech_kv}. Changing attention weights alone has limited and inconsistent effects, reducing non-target $\mathrm{TV}$ by $5.6\%$--$36.4\%$ on four models and increasing it by $35.9\%$ on \modelid{Qwen2.5-32B}. Substituting directional value vectors, however, cuts non-target $\mathrm{TV}$ by $87.3\%$--$99.3\%$ across all five models. This value substitution also removes $88.5\%$--$104.4\%$ of the signed item-level main effect while preserving $0.916$--$1.000$ of independent target efficacy. Directional preferences therefore stabilize choices through value overwrite rather than attention rerouting.

Directional instructions saturate the option preference, leaving little room for the target attribute to alter the final choice. On \texttt{risk\,$\to$\,color} items, directional clauses write a $\pm 6.2$ to $\pm 37.1\text{ nat}$ bias onto the option-comparison axis. This bias compresses target-induced shifts to $+0.09$--$+0.72\text{ nat}$, compared with $+2.95$--$+9.85\text{ nat}$ under neutrality across all five models (Appendix~\ref{app:c5}). Whereas preference reports read a dedicated neutrality axis, binary option comparison receives near-zero net bias from a symmetric neutral clause, leaving Path~B pull uncontested. Probability saturation also explains why low baseline drift of $0.188\text{ TV}$ on \modelid{Qwen2.5-32B} does not indicate genuine neutral control. Strong default option biases on that model mask target pull until routing swaps or Path~A masking relax that saturation (Table~\ref{tab:mech_kv}). Value overwrite thus dominates option comparison under directional prompts, whereas neutrality leaves query-time target pull unconstrained.

\section{Why Standard Interventions Fail to Restore Neutral Control}\label{sec:main}

\subsection{Prompt Instructions and Representation Editing Fall Short}

The mechanism in Section~4 points to three candidate intervention stages across prompt rules, hidden representations, and query-time attention. Can we restore selective control simply by instructing the model to keep unrelated preferences fixed? We test paraphrased neutrality statements, explicit scope limits, negative prohibitions, and rich multi-sentence user backstories \citep{moon2024anthology}, yet none consistently removes non-target drift (Appendices~\ref{app:c4} and~\ref{app:c6}). Even prepending detailed Anthology backstories leaves $40.5\%$--$71.1\%$ of held-out worlds sensitive to the target attribute. Consistent with Section~4, preserving explicit neutrality supplies no competing option-value signal to counter query-time target reading.

Other prompt rules appear more effective, yet solve a different problem or carry steep costs (Appendices~\ref{app:c4}--\ref{app:c6}). Pairing scope restrictions with an ordinary-default preference stabilizes choices across most tasks because it tells the model which option pole to favor, replacing neutrality with a directional rule. First-option tie-breakers still leave $30\%$--$90\%$ of stems sensitive across two tested models. Adding an explicit indifference option either leaves drift high or pulls target retention to $0.133$--$0.285$ on \modelid{OLMo-2-32B}. Across multi-condition black-box evaluations, explicit scope and anti-inference rules likewise leave most prompt designs sensitive to the target attribute.

Even on bare prompts without neutral declarations, linear subspace erasure in late hidden states \citep{leace,refusal} hits a sharp boundary between rating scales and binary choices. On seven-point ratings across $47$ items, erasing unorthogonalized rank-$4$ subspaces fit to persona displacements removes non-target shifts yet collapses target retention. Projecting out the target displacement span first yields an orthogonalized rank-$4$ contrast subspace. Across five primary and three held-out checkpoints, this subspace removes $56\%$--$86\%$ of the rating shift while retaining $86\%$--$99\%$ of target responses, whereas random subspaces leave both unchanged. Applied to binary choices, however, the same subspaces remove only $13\%$--$31\%$ of drift, and even refitting directly on binary-choice activations lowers non-target $\mathrm{TV}$ by at most $0.20$. Continuous ratings read a fixed scale direction at the answer token, whereas binary choices compare item-specific option pairs against the persona at query time. A static scale subspace therefore cannot isolate discrete option comparisons (Appendix~\ref{app:c3}). We thus intervene directly on the query-time attention pathway identified in Section~4.

\subsection{Blocking Target Access Also Weakens Intended Control}

What happens if we directly prevent option evaluation from reading the target persona? Ranking attention heads out of sample via two-fold Path~B contribution, we progressively mask query-to-target heads across $80$ non-target and $40$ target items per model. In Figure~\ref{fig:pathb_dose}a, masking these heads steadily reduces non-target drift, yet weakens intended target responses in lockstep (Appendix~\ref{app:c5}). Masking the top $5\%$ of heads modestly reduces non-target $\mathrm{TV}$ while keeping target retention between $0.765$ and $1.000$. At a $20\%$ blocking budget, five-model mean non-target $\mathrm{TV}$ falls from $0.328$ to $0.096$ while mean target retention drops to $0.518$. Blocking all Path~B heads collapses mean retention to $0.029$, and no budget achieves both low drift and high retention across all five models (Figure~\ref{fig:pathb_dose}b).

\begin{figure}[!htb]\centering
\includegraphics[width=\textwidth]{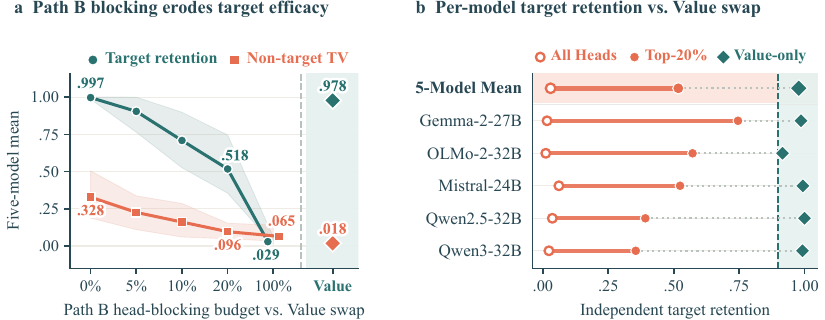}
\caption{\textbf{Dose--retention trade-off under progressive Path~B head blocking versus Value-only overwrite.}
\textbf{(a)} Five-model mean non-target $\mathrm{TV}$ and independent target retention across Path~B head-blocking budgets versus Value-only overwrite.
\textbf{(b)} Per-model independent target retention relative to the $0.90$ dashed threshold under Top-$20\%$ and All-Heads blocking versus Value-only overwrite; full values appear in Appendix~\ref{app:c5}.}
\label{fig:pathb_dose}
\end{figure}

Contrasting Path~B blocking with the Value-only swap in Figure~\ref{fig:pathb_dose} crystallizes why selective neutral control remains elusive. Value overwrite preserves $0.978$ target retention at $0.018\text{ TV}$ drift, yet relies on supplying a directional preference, abandoning neutrality and risking target erosion under opposing poles (Section~3.2). Together, these results reveal a three-way tension among non-target stability, target efficacy, and genuine neutrality across the tested intervention families. Blocking shared query-time access removes information needed for the target task, whereas value substitution alters the neutral preference itself. Until models support task-dependent selective access, simulations must control incidental presentation cues directly in experimental design.

\section{Related Work}\label{sec:related}

\textbf{LLM user simulation and persona prompt control.} Large language models are used to construct synthetic respondents with specified attributes or backgrounds, simulating human economic decisions and survey responses \citep{silicon,homosilicus,park2023generativeagents}. However, surface distributional realism does not guarantee reliable behavioral treatment-effect estimates \citep{bisbee2024synthetic,hullman2026simulation,li2026statisticalrealism}. Prior work documents concomitant shifts in unspecified traits when prompting or steering a target trait \citep{psychometric2025,sandhan2026phish}. Related studies separate intended from unintended persona effects on task performance \citep{luzdearaujo2025principled,cheng2023marked} and use multi-sentence user backstories to strengthen persona consistency \citep{moon2024anthology,kang2025deepbinding}. Closer to our focus on causal validity in controlled simulations, \citet{gui2023causalsim} show that unspecified background attributes confound treatment-effect estimates. \citet{lin2026illusion} further formalize treatment-induced user drift and examine constraining it by explicitly specifying relevant covariates. Building on this covariate-control perspective, we show that explicitly specifying an attribute as neutral does not selectively control it. Even when directional compliance tests and self-report manipulation checks \citep{hauser2018manipulation} pass, downstream choices fail counterfactual invariance.

\textbf{Inference-time trait steering and intervention specificity.} Related work manipulates and monitors behavioral traits in representation space through activation steering and persona direction monitoring \citep{actadd,caa,repe,personavectors}. Related methods draw on model editing and linear concept erasure to pursue intervention specificity and reduce non-target side effects \citep{rome,cohen2024ripple,rlace,leace}. For multi-trait control, recent work composes directional personality traits via activation vectors or weight-space adapters \citep{feng2026persona,baines2026cartography}. Subsequent studies find persistent cross-trait interference after orthogonalization \citep{bhandari2026traitsinterfere,ong2026sideeffects} and show that user traits inferred from dialogue resist steering toward an uninformative state \citep{neplenbroek2025reading}. Rather than steering multiple directional traits simultaneously or erasing implicitly inferred identities, we study how models use explicitly stated persona declarations during decision making. We combine hidden-state subspace separation with query-time counterfactual interventions to test when selective control holds and why it fails under neutral declarations.

\section{Conclusion}\label{sec:conclusion}

In LLM user simulation, following a stated persona direction is not equivalent to achieving selective control that isolates non-target choices from unwanted influence. Our three-state evaluation shows that models follow stated directions yet still shift non-target choices and distort treatment-effect estimates across irrelevant labels under explicit neutrality, despite accurate neutral self-reports. Directional compliance tests and self-report manipulation checks are therefore insufficient to establish the counterfactual invariance required for controlled simulations. Internal interventions show that non-target drift enters primarily at query time, where directional declarations overwrite target pull with strong value signals that neutral declarations lack. Because target and non-target choices share this query-time access, severing it weakens intended target effects, whereas value overwrite imposes an artificially prescribed preference direction. Enabling models to access persona information selectively based on the current task, preserving target control while maintaining non-target invariance, remains the central problem for future work.

\paragraph{Limitations.} The black-box field audit covers one attribute pair, whereas the four-pair evaluations and internal interventions use open-weight checkpoints exposing exact probabilities and hidden states. Directional instructions also require verification across multi-turn interactions. Across ten-turn dialogues, a scope-and-default instruction lowers cross-attribute drift on all eight black-box models. However, it keeps drift below a $0.05$ signed-shift floor with $\ge 0.90$ target retention at every measured turn on only three of eight models (Appendix~\ref{app:c6}). Finally, all evaluation materials are English and text-only, and our selectivity criteria are defined against explicit prompt specifications rather than empirical human behavioral distributions. Whether inferred preferences match real human populations remains a separate empirical question from whether prompt specifications act selectively on downstream model decisions (Appendix~\ref{app:limits}).

\subsection*{Ethics Statement}
This work measures whether persona prompts move non-target attributes. The study involves no human subjects. Model-generated associations between traits and preferences are not evidence about people, and no claim about social groups is made from them. The constructed target-orthogonal intervention could also conceal selected side effects while preserving the target response; we report it as a scientific control, not a deployment recipe. The results caution against research or product decisions based on persona-conditioned synthetic users without checking that non-target attributes stay invariant to the target.

\subsection*{Reproducibility Statement}
Appendix~\ref{app:suite} identifies the models, materials, unit of analysis, and status of each main design;
Appendix~\ref{app:claims} maps claims to evidence; Appendices~\ref{app:formal} and~\ref{app:battery}
define the estimands and decision rules; and Appendix~\ref{app:repro} records API provenance,
checkpoint identifiers, corrections, and dated deviations. Core prompts and complete headline
tables, complete mathematical proofs, and their assumptions are included in the appendix. The code and data release contains the items, wordings, worlds, and splits; the cached scores behind every number in Sections~\ref{sec:declare}--\ref{sec:main}, the Conclusion, and their appendix tables (per item, world, or cell where these were stored); a standard-library script that recomputes each of these numbers and prints it beside the value in the paper; and the prompt-rendering, scoring, and erasure code. Raw per-call responses and hidden-state activations are omitted for size.  Post hoc, exploratory, and revised-rule analyses are explicitly labeled.

\clearpage
\appendix
\setcounter{equation}{0}\renewcommand{\theequation}{A.\arabic{equation}}\renewcommand{\theHequation}{A.\arabic{equation}}
\section{Appendix}
The appendix is organized for checking the paper's six claims. Appendix~\ref{app:claims} maps
claims to evidence; Appendices~\ref{app:suite}--\ref{app:prompts} give the designs, definitions,
and verbatim core materials; Appendices~\ref{app:c1}--\ref{app:c6} report the results and essential
controls. Provenance and scope are in Appendices~\ref{app:repro} and~\ref{app:limits}.
The mathematical definitions, complete proposition proofs, and assumptions are retained in
Appendices~\ref{app:formal}--\ref{app:props}. Several appendices use formal terms from the
underlying construct-validity analysis; Table~\ref{tab:bridge} in Appendix~\ref{app:claims} maps the main ones to their main-text terms.

\subsection{Guide to the Evidence}\label{app:claims}
This appendix contains the designs, complete headline results, and qualifications needed to assess
Sections~\ref{sec:declare}--\ref{sec:main}. Table~\ref{tab:claims} identifies where to check each claim.
\emph{Prespecified} means that the relevant rule was fixed before scoring; \emph{exploratory} means
that no confirmatory decision is claimed; \emph{post hoc} identifies analyses performed after results
were available. These labels do not imply preregistration. Different item sets, model panels, response
formats, and measurement scales are not pooled.

\begin{table}[!htb]\centering
\renewcommand{\arraystretch}{1.08}\small
\caption{\textbf{Main claims, necessary evidence, and limits.} Across rows, black-box prevalence (C1) covers one attribute pair, and C2--C6 use open-weight checkpoints or targeted black-box panels (Appendix~\ref{app:limits}).}
\label{tab:claims}
\setlength{\tabcolsep}{4pt}
\begin{tabularx}{\linewidth}{@{}>{\raggedright\arraybackslash}p{.19\linewidth}Y >{\raggedright\arraybackslash}p{.23\linewidth}@{}}
\toprule
Claim & Evidence retained here & Main qualification \\
\midrule
C1. Unstated cross-attribute shifts & Per-model field audit, pole controls, unrelated control attributes, construct coding, and baseline-clause deletion (App.~\ref{app:c1}). & Black-box scores are relative to a baseline stating the cautious direction; author-assigned coding keeps the risk--control gap $\ge.87$ under two-model recodings. \\
C2. Semantic and prior controls & Invented-label semantic controls, supplied in-context correlation manipulation, and same-item base--instruct prediction (App.~\ref{app:c2}). & The arbitrary code is not activation-matched (the scoped invented name is the activation-retaining contrast); base prediction is not a causal estimate of post-training. \\
C3. Linear subspace separability limits & Rating-scale erasure design, point estimates, intervals, random-subspace controls, and binary-choice transfer limits (App.~\ref{app:geom}). & Three of four held-out checkpoints replicate; the fourth is uninterpretable. Rating-scale separability does not transfer to discrete binary choice. \\
C4. Neutrality gap and signed ATE bias & Three-state evaluation, $477$ held-out worlds, $160$ independent items, five-level target retention, signed ATE bias, wording controls, and Anthology backstories (App.~\ref{app:c4}). & Holdout reuses $40$ item stems across parameter draws (independent items supply $200$ unseen stems); neutral fidelity is not needed for choice-distribution invariance. \\
C5. Report--choice dissociation, Path~B access, and value overwrite & Shared-prefix report--choice comparisons, length-aligned full-span K/V closure ($1.000$), Path~A/B blocking, $2\times 2$ routing--value swaps, and Path~B dose--retention trade-offs (App.~\ref{app:c5}). & Single-token patching at one layer transfers nothing, whereas full non-target span K/V replacement across layers achieves $1.000$ median closure across all five models. \\
C6. Presentation distortion and mitigation boundary & $60$-item investment presentation-randomization audit (E-PR), containment rules and counts, scope-only controls, rule-lookup reference, and ten-turn dialogue persistence (App.~\ref{app:c6}). & A supplied default direction replaces neutrality; counterbalancing presentation labels cancels antisymmetric bias on average but still lowers balanced target retention on affected models. \\
\bottomrule
\end{tabularx}
\end{table}

\paragraph{Mathematical basis and inference.}
Appendix~\ref{app:formal} gives the definitions and algebraic relations; Appendix~\ref{app:pathspecific}
states the path-admissibility assumption and the limit of its interpretation. Appendix~\ref{app:props}
contains both propositions, their complete proofs, and the sample-size calculation showing that the reported audits support bank-level detection but not finite-sample certificates. The world-level certificate in Appendix~\ref{additional-multiplicity-checks} is a separate procedure; the same appendix states its dependence limitation and shows that every declared violation survives text-clustered re-analysis. Neither certificate is inferred from passing a descriptive screen.

\paragraph{How to check a quantitative claim.}
The designs and panels are in Appendix~\ref{app:suite}, and the prompt templates are in
Appendix~\ref{app:prompts}. Each results section states its own denominator, measurement scale,
reference condition, decision rule, and exceptions. Corrections, post hoc rules, and
limitations remain in Appendices~\ref{app:repro} and~\ref{app:limits} even when they weaken a claim.

\begin{table}[!htb]\centering
\renewcommand{\arraystretch}{1.08}\footnotesize
\caption{\textbf{Mapping between main-text concepts and formal proof or codebase identifiers.} Each row links a narrative concept in Sections~1--5 to its formal notation or codebase identifier in the appendix.}
\label{tab:bridge}
\setlength{\tabcolsep}{4pt}
\begin{tabularx}{\linewidth}{@{}Y Y l@{}}
\toprule
Main-text concept & Formal proof / codebase identifier & Defined in \\
\midrule
non-target shift / cross-attribute drift & leakage & App.~\ref{app:formal}, \ref{app:geom} \\
non-target sensitivity $S_\theta(c,q,w)$ & $L_{\mathrm{sens}}$ & Section~\ref{sec:framework}, App.~\ref{app:battery} \\
author-annotated attribute labels & construct scope specification & App.~\ref{app:formal} \\
semantically marked, non-default pole & marked pole; markedness & App.~\ref{app:formal} \\
non-target attribute, as a prompt factor & discriminant factor $C_{-R}$ & App.~\ref{app:notation} \\
field audit & field track & App.~\ref{app:formal} \\
higher-sample audit & dense audit & App.~\ref{app:suite} \\
style / persona-induced displacements & marked-versus-default personas & App.~\ref{app:geom} \\
selective persona control, field-audit form & $(\epsilon,\tau)$-construct-surgical; $L_U$, $R_T$ & Def.~\ref{def:surgical} \\
share of answers moved off the scale midpoint & escape & App.~\ref{app:gatedir} \\
cross-attribute influence as a causal path & forbidden-path effect & App.~\ref{app:pathspecific} \\
target / non-target items & on-construct / leakage items & App.~\ref{app:c2:semantics} \\
unrelated control attribute & null (codebook label) & App.~\ref{app:c1:coding} \\
primary evaluation pairs & certification pairs & App.~\ref{app:c4:design} \\
suppressing drift in an informative cell & containment & App.~\ref{app:c6} \\
query-time target access & Path~B attention ($\{\mathcal{S}_{\text{opt}},\mathcal{S}_{\text{task}}\}\to\mathcal{S}_T$) & App.~\ref{app:c5} \\
structured-world prompts & declaration battery; battery frame & App.~\ref{app:c4:design} \\
\bottomrule
\end{tabularx}
\end{table}
\FloatBarrier

\subsection{Evaluation Designs and Model Panels}\label{app:suite}
The field audit measures changes on attributes outside a declared construct scope; as a test of non-target invariance, it needs no human ground-truth preference, but it cannot measure fidelity, which needs a ground-truth answer. Structured worlds supply one and separate absent,
directional, and neutral declarations. Independent items test new texts, rather than additional
parameter draws over the structured catalog. Table~\ref{tab:itemsets} gives the units of the
analyses used in the main paper. Orders, answer-code maps, and repeated calls are within-unit repeats.

\begin{table}[!htb]\centering
\renewcommand{\arraystretch}{1.08}\footnotesize
\caption{\textbf{Designs supporting the main claims.} Counts describe each design separately. The last column gives each design's analysis unit and whether its analysis was prespecified, descriptive, exploratory, or post hoc.}
\label{tab:itemsets}\setlength{\tabcolsep}{4pt}
\begin{tabularx}{\linewidth}{@{}>{\raggedright\arraybackslash}p{.20\linewidth}Y >{\raggedright\arraybackslash}p{.21\linewidth}@{}}
\toprule
Design & Materials and comparison & Unit / status \\
\midrule
Headline field audit & $16$ non-target style items and $30$ risk-axis target items, $20$ draws per cell; $7$ same-axis pole reversals, $5$ draws per cell. & Item; prespecified. \\
Factorial / worked criterion & Full bank: three prompt conditions, four factor cells, $94$ non-target and
$18$ target items, $20$ draws; compressed factorial with a stated default pole: six profiles, $8$ non-target and
$6$ target items, $4$ draws. & Item / cell; descriptive, not certified (App.~\ref{app:props}). \\
Dense audit & $17$ items outside the target domain ($12$ aesthetic, $2$ communication, $1$ ambient, $2$ ingestive)
and $4$ risk items; seven-point responses, $30$ draws; ingestive (ambiguous) cells excluded from headlines. & Item;
descriptive; coding sensitivity post hoc. \\
Semantic controls & $94$ non-target and $18$ target items, scored as answer-token log-odds under five
wording conditions. & Item; descriptive contrasts. \\
Base and representation bank & $47$ seven-point items: $13$ non-target, $4$ ambiguous, $30$ target;
$18$ condition wordings in the representation program; two fit/evaluation folds. & Item;
prespecified program, with logged deviations. \\
Explicit-evidence bank & $64$ worlds, $16$ per attribute pair, $42$ non-target texts; $16$
demonstrations; absent, declared, mentioned, or past-choice evidence; no neutral level. & World;
exploratory. \\
Validation / held-out worlds & $162$ validation and $477$ in-scope held-out worlds; $40$ non-target
texts over risk$\to$color and time$\to$sound; five declared levels. & World; prompt design selected on validation;
level/cluster analyses post hoc. \\
Wording, prompt, report/choice & $80$ structured texts: $20$ non-target and $20$ target per pair;
wording confirmation in the battery frame; report/choice in separate open-frame contexts. & Text;
wording prespecified; prompt controls descriptive; report/choice exploratory. \\
Independent items & $200$ new texts: $160$ non-target, $40$ target, over risk$\to$color,
time$\to$sound, risk$\to$hue, time$\to$layout; absent, both directions, and the neutral declaration (option-specific wording). & Item;
family-clustered; exploratory; level split post hoc. \\
Readout and utility checks & Readout refit: $13$ non-target and $30$ target items, two folds, seven-point
and binary formats. Collateral utility: $60$ items; separate from the persona audit. & Item; exploratory
refit / prespecified utility criteria (App.~\ref{app:geom}). \\
Rule reference / containment & Rule lookup: $64$ worlds per checkpoint over four pairs.
Containment: two black-box models, eight tasks, three personas, $31$ informative cells per instruction. & World / cell;
exploratory / descriptive. \\
Anthology backstory control & $477$ held-out worlds prepended with $50$ domain-filtered multi-sentence user backstories \citep{moon2024anthology} over risk$\to$color and time$\to$sound. & World; confirmatory control (App.~\ref{app:c4}). \\
Investment presentation audit (E-PR) & $60$ personal investment items ($20$ equal-EV, $20$ variable-higher, $20$ variable-lower, after $12$ disjoint pilot items) and $10$ non-target color items across aligned, swapped, and uncolored plain frames ($16{,}120$ evaluations per model). & Item; confirmatory audit (App.~\ref{app:c6}). \\
Length-aligned K/V \& Path~A/B interventions & $80$ non-target items ($20$ per pair across four pairs) and $40$ independent target items ($10$ per pair); full-span $\mathcal{S}_Z$ K/V swaps, $2\times 2$ routing--value swaps, and Path~A/B blocking. & Item; confirmatory mechanistic audit (App.~\ref{app:c5}). \\
\bottomrule
\end{tabularx}
\end{table}

\paragraph{Model panels.}\label{tab:panels}
The \textbf{primary open-weight panel} is Qwen2.5-32B-Instruct \citep{qwen25report}, Qwen3-32B \citep{qwen3report},
Mistral-Small-3.1-24B-Instruct \citep{mistralsmall31}, Gemma-2-27B-IT \citep{gemma2report}, and OLMo-2-0325-32B-Instruct \citep{olmo2report}.
The \textbf{independent-item and readout-refit panel} replaces Qwen2.5 with Qwen3.5-9B \citep{qwen35blog};
the latter is not the black-box panel's Qwen3.5-397B. The baseline-deletion and structured-text
report/choice comparisons use the primary five plus Qwen3.5-9B. Wording confirmation covers
Mistral and Qwen3 only. Held-out representation checks use Llama-3.1-8B-Instruct \citep{llama3report}, GPT-OSS-20B \citep{gptoss},
Gemma-3-27B-IT \citep{gemma3report}, and Gemma-4-31B-IT \citep{gemma4report}; three replicate the erasure pattern, and the fourth is uninterpretable because it fails its random-control check.
The \textbf{black-box panel} (queried only through provider APIs; several of these models also have open weights, but none was run locally) is DeepSeek-V4-Flash \citep{deepseekv4report}, GPT-5.5 \citep{gpt55systemcard}, Kimi-K2.6 \citep{kimik26blog}, which shares the Kimi~K2.5 architecture \citep{kimik25report},
Gemini-3-Flash-Preview \citep{gemini3flashcard}, GLM-5.2 \citep{glm5report},
Qwen3.5-397B \citep{qwen35blog}, Nemotron-3-Super \citep{nemotron3super}, and MiMo-V2.5-Pro \citep{mimov25pro}, which builds on MiMo-V2-Flash \citep{mimov2flash}. The dense audit and containment comparison
use GPT-5.5 and Kimi-K2.6; the ten-turn comparison cited in the Conclusion uses all eight black-box
models on substitute API providers (Appendix~\ref{app:c6:dialogue}). Base-pair coverage is specified with its results
(Appendix~\ref{app:c2:base}); exact open-weight IDs and API provider limitations are in
Appendix~\ref{app:repro}.

\paragraph{Text independence and material provenance.}
Of the $40$ held-out non-target texts, $39$ also occur in validation. The holdout separates parameter
draws, not question texts. The independent bank was written from authored specifications without
access to model outputs and fixed before scoring; it is model-generated, not human-normed, and the neutral-versus-directional contrast holds its items fixed.
Its non-target items have $5/11/4/1$ generator families per pair (color/hue/sound/layout), respectively;
its target pairs have $5/10/1/1$. No family-level interval is asserted on single-family quantities.
The four structured-world pairs are risk weight$\to$color, time preference$\to$sound,
social contact$\to$layout, and time of day$\to$color; only the first two enter prompt-design selection
and held-out evaluation. The independent hue and layout items change both attribute and content;
they are not one-factor interventions on markedness.

\paragraph{Decision rules.}
The headline audit uses a probability-scale effect floor $.05$, item-clustered bootstrap tests
($B=2{,}000$), and Benjamini--Hochberg correction at $q=.05$ over its declared family.
The held-out evaluation uses sensitivity threshold $.15$ and tolerance $.05$, fixed before held-out evaluation. Containment requires an influence-reduction rule and target
retention at least $.60$; the held-out evaluation uses a $.90$ retention floor. Representation
controls and exploratory report/choice references are stated in their own sections. A pass under
one rule is not a pass under another. Deviations, including rules changed after a run, are retained
in Appendix~\ref{app:repro}.

\subsection{Formal Definitions, Estimands, and Algebraic Identities}\label{app:formal}

\paragraph{Field-track construct scope.} The field track seeks the target distribution $P^*$ under a
controlled textual assignment of the target factor with the other prompt-layer factors at their
defaults, while the model returns $P_\theta(Y\mid C_p{=}c)$ (a path-specific reading in
Appendix~\ref{app:pathspecific}). A construct scope specification partitions item cells into a target set $\mathcal{P}_p$ and a non-target set
$\mathcal{U}_p$ and records why a cell is non-target \citep{jacobs2021measurement}: a literature-based
construct distinction \citep{weber2002dospert,blais2006dospert,bloch2003cvpa,cleridou2014aesthetics},
which motivates an audit without establishing zero association in a human population, a premise the declared target distribution does not need; an explicit
instruction; or a generative ground-truth rule. With $\Delta_{p,b}(d,q)$ the persona-minus-reference shift
toward the judge-normed marked option on item $q$ in domain $d$, the audit averages absolute
non-target-cell shifts by declared weights, and a design meets the field-track criterion when that mean is
within the declared tolerance and bare-prompt-relative retention meets the declared floor. Because a
natural trait fixes no ground-truth non-target answer, a prompt design can score as insensitive on this
track while misrepresenting the user; fidelity is therefore measured only on constructed worlds.

Table~\ref{tab:notation} collects the symbols for main-text quantities and the appendices, marking the appendix-local ones. The field-track criterion is the following.
\begin{table}[!htb]\centering
\renewcommand{\arraystretch}{1.08}
\caption{\textbf{Notation.} Symbols for main-text quantities, the field audit, and the declaration battery; the last two rows list appendix-local symbols and conventional reuse.}
\label{tab:notation}
\footnotesize
\setlength{\tabcolsep}{4pt}
\begin{tabularx}{\linewidth}{@{}>{\raggedright\arraybackslash}p{0.30\linewidth} Y >{\raggedright\arraybackslash}p{0.17\linewidth}@{}}
\toprule
Symbol & Meaning & Where \\
\midrule
$T$, $Z$; $t$, $z$ & target and non-target attribute; their assigned values ($z{=}0$ is the declared-neutral level) & Sections~\ref{sec:framework}, \ref{sec:declare} \\
$w$, $q$, $b$ & constructed world; non-target or target query; prompt design (also the field-audit design index) & Section~\ref{sec:framework}, Appendix~\ref{app:formal} \\
$\pi_\theta(\cdot\mid T,c,q,w)$, $\pi_b(\cdot\mid t,z,q,w)$, $\pi_Z^*$ & normalized two-option distribution of model $\theta$; distribution under design $b$; ground-truth reference & Section~\ref{sec:framework}, Appendix~\ref{app:battery} \\
$S_\theta(c,q,w)$, $L_{\mathrm{sens}}$, $E_Z$, $a_{w,z}$, $u_t$ & main-text and battery non-target sensitivity; attribute-fidelity error; conservative residual; unassigned mass & Section~\ref{sec:framework}, Appendix~\ref{app:battery} \\
$\rho$, $D_p(\rho)$ & demonstrated correlation; signed non-target-option difference between the target levels & Appendices~\ref{app:c2}, \ref{app:battery} \\
$\ell_t$, $\alpha_w$, $\eta_w$, $\sigma$ & semantic log-odds at target level $t$; shared answer tendency; target-induced margin; logistic sigmoid & Section~\ref{sec:declare}, Appendix~\ref{bd-decomposition} \\
$\boldsymbol{h}$, $\boldsymbol{U}$, $\boldsymbol{h}_{\mathrm{ref}}$, $k$; CEP, RET & hidden state; orthonormal basis of the erased subspace; no-information reference state; subspace rank; leakage removed and target retained & Section~\ref{sec:main}, Appendix~\ref{app:geom} \\
$\Delta_{p,b}(d,q)$, $L_U$, $\mathcal{T}(p,b)$, $R_T$, $(\epsilon,\tau)$ & field audit: persona-minus-reference shift on item $q$ in domain $d$; leakage; target effect; retention; declared tolerances & Appendices~\ref{app:formal} and~\ref{app:props} \\
$\mathcal{P}_p$, $\mathcal{U}_p$, $\omega_{d,q}$, $s_p$ & target and non-target cell sets of persona $p$; declared weights; the persona's signed pole & Appendix~\ref{app:formal} \\
$C$, $c_0$, $R$, $S$, $\doP$, $P^*$, $P_\theta$, $\Delta^*(d,q)$ & prompt-layer factor vector and its defaults; target and discriminant factors; protocol assignment; target and realized distributions; target estimand & Appendices~\ref{app:notation} and~\ref{app:sr} \\
$p_p(q)$, $p_0(q)$, $g_p(q)$, $\eta_p(q)$ & marked-choice probability under persona $p$; reference prompt; its logit; logit shift & Appendix~\ref{app:setup} \\
$G_{p,d}$, $\mathrm{Dir}_{p,d}$, $\delta_G$, $y_{\mathrm{mid}}$, $\phi_{p,d}$, $m_{d,q}$, $\xi_{p,d,q}$, $\boldsymbol{\lambda}$, $\boldsymbol{w}$ & escape; conditional direction; midpoint-band half-width; scale midpoint; effective gate; item markedness contrast; cell residual; persona intensities; domain loadings & Appendix~\ref{app:gatedir} \\
$\mathrm{SR}$, $\mathrm{CO}_b$ & surgicality ratio; containment openness (the residual gate $G$ of Appendix~\ref{app:c6}) & Appendix~\ref{app:sr} \\
$n$, $\delta$, $\alpha$, $\kappa$, $\Phi$ & panel size; choice-scale leakage; test level; marked-winner probability bound; standard normal distribution function & Proposition~\ref{prop:panel} \\
$M$, $n_i$, $\zeta_i$, $\gamma$ & number of non-target cells; contrasts per cell; Hoeffding radii; confidence parameter & Proposition~\ref{prop:cert} \\
$\tau_{\mathrm{DL}}$, $I^2$; $\nu$ & heterogeneity statistics (Appendix~\ref{app:endpoints}); margin-perturbation budget (Appendix~\ref{internal-control-evidence-and-numerical-scope}) & appendix-local \\
$B$; $o$, $\mathrm{opt}_1$, $\mathrm{opt}_2$; $p$; $r$; Krippendorff's $\alpha$ & bootstrap replicates; an answer option and the two options of a non-target query (Appendix~\ref{app:battery}); $p$-values (elsewhere $p$ indexes personas); Pearson correlation; inter-judge agreement (Appendix~\ref{app:c1:coding}) & conventional \\
\bottomrule
\end{tabularx}
\end{table}

\begin{defn}[Construct-surgical persona intervention]\label{def:surgical}
For design $b$ with bare-persona reference $b_{\mathrm{ref}}$, define
\begin{equation}\label{eq:surgical}
L_U(p,b)=\!\!\sum_{(d,q)\in \mathcal{U}_p}\!\!\omega_{d,q}|\Delta_{p,b}(d,q)|,
\quad
\mathcal{T}(p,b)=\mathbb{E}_{(d,q)\in \mathcal{P}_p}[s_p\Delta_{p,b}(d,q)],
\quad
R_T(p,b)=\frac{\mathcal{T}(p,b)}{\mathcal{T}(p,b_{\mathrm{ref}})}
\end{equation}
with declared weights $\omega_{d,q}\ge0$ summing to one and $s_p\in\{-1,+1\}$ the persona's expected pole
(Appendix~\ref{app:gatedir}); $R_T$ is defined when $\mathcal{P}_p\ne\emptyset$ and
$\mathcal{T}(p,b_{\mathrm{ref}})\ge \mathcal{T}_{\min}>0$. Design $b$ is \emph{$(\epsilon,\tau)$-construct-surgical} if
$L_U(p,b)\le\epsilon$ and $R_T(p,b)\ge\tau$; applications also declare a tolerance for
$L_U^{\max}(p,b)=\max_{(d,q)\in \mathcal{U}_p}|\Delta_{p,b}(d,q)|$.
\end{defn}

\subsubsection{Target Distribution and Notation}\label{app:notation}
The study declares a factor vector $C=(R,\,C_{-R})$ over prompt-layer text fields, default
levels $c_0$, an item bank, and readout scales. Write $\doP$ for controlled assignment of these fields
by the study design. The target distribution
\begin{equation}\label{eq:target-dist}
P^*\!\left(Y \mid \doP(R{=}r,\ C_{-R}{=}c_0)\right)
\end{equation}
is defined by this declaration: it is the completion distribution the study design intends, in which only the
named factor departs from its default. Equivalently, the target estimand for item $q$ in
domain $d$ is
\begin{equation}\label{eq:delta-star}
\Delta^*(d,q)\;=\;\mathbb{E}\!\left[Y_{d,q}\mid \doP(R{=}r_1,\,C_{-R}{=}c_0)\right]
\;-\;\mathbb{E}\!\left[Y_{d,q}\mid \doP(R{=}r_0,\,C_{-R}{=}c_0)\right]
\end{equation}
Three clarifications fix the semantics. First, $\doP$ names a controlled assignment of prompt-layer text,
which the experimenter can actually perform; it is not an identified intervention on a latent internal
variable of the model or of any human population. Second, the discriminant factors $C_{-R}$ appear inside
the assignment: defaults are part of the declared manipulation, not an observational conditioning event.
Third, $P^*$ is a normative reference fixed before evaluation; the audit measures how far the realized
conditional completion $P_\theta$ departs from it on declared cells. When a prompt design leaves $C_{-R}$
textually unset, $c_0$ is the declared semantic default (ordinary, unremarkable preferences), and the
audit's reference prompt estimates the model's behavior at that default.

\subsubsection{Setup and Headroom Identity}\label{app:setup}
For item $q$ in domain $d$ with a marked pole fixed by norming, let $Y_{p,q}$ be the simulated choice
under persona $p$ and let $p_p(q)=\Pr(Y_{p,q}=\text{marked})$ for binary items, with the $K$-point scale
analog using midpoint $y_{\mathrm{mid}}=(K{+}1)/2$. All estimands are reference-adjusted:
$\Delta_p(q)=p_p(q)-p_0(q)$, where $p_0$ is the same model's reference prompt (the baseline persona in the field audit; the no-information reference elsewhere). On the logit scale,
$g_p(q)=\logit p_p(q)=g_0(q)+\eta_p(q)$, and the choice-rate shift obeys the headroom identity
\begin{equation}\label{eq:headroom}
\Delta_p(q)=\frac{p_0(q)\,(1-p_0(q))\,(e^{\eta_p(q)}-1)}{1-p_0(q)+p_0(q)e^{\eta_p(q)}}
\;\xrightarrow{\;p_0 e^{\eta}\gg 1\;}\;1-p_0(q)
\end{equation}
The identity follows by writing
$p_p=p_0e^{\eta_p}/(1-p_0+p_0e^{\eta_p})$ and subtracting $p_0$ over the common denominator. This explains
why endpoint (escape) magnitudes saturate for strong personas and why headline inference uses
signed quantities. This identity carries the paper's saturation argument; no
separate saturating link function is fitted.

\subsubsection{Gate and Direction (Audit Variables)}\label{app:gatedir}
Escape and conditional direction are
\begin{equation}\label{eq:gate-dir}
G_{p,d}=\Pr\left(|Y-y_{\mathrm{mid}}|>\delta_G\right),\qquad
\mathrm{Dir}_{p,d}=\Pr\left(Y>y_{\mathrm{mid}} \,\middle|\, |Y-y_{\mathrm{mid}}|>\delta_G\right)
\end{equation}
with $\delta_G$ the declared half-width of the midpoint band. A content-free extreme-response style predicts $\mathrm{Dir}\approx.50$ once escaped; trait-conditioned completion
predicts $\mathrm{Dir}\to 1$ for marked-pole personas and $\mathrm{Dir}\to 0$ for default-pole personas. The cell-level
measurement model that motivates these audit variables is
\begin{equation}\label{eq:measurement}
\Delta_{p,d,q}\;\approx\;\phi_{p,d}\; s_p\, m_{d,q}\;+\;\xi_{p,d,q}
\end{equation}
with $s_p$ the persona's signed markedness loading (direction codebook), $m_{d,q}$ the item's normed
markedness contrast, $\phi_{p,d}\ge0$ the effective gate, and $\xi_{p,d,q}$ a cell residual. The construct-validity failure is
$\phi_{p,d}>0$ on cells the construct scope specification codes as non-target. When the trait-by-domain shift matrix
happens to factorize, the outer-product form $\boldsymbol{\Delta}\approx \boldsymbol{\lambda}\otimes \boldsymbol{w}$ (signed persona
intensities $\boldsymbol{\lambda}$ with $\sign(\lambda_p)=s_p$, and domain loadings $\boldsymbol{w}$) is used strictly as the compact
descriptive abstraction of the gate-and-direction account. The term \emph{markedness} follows
linguistic markedness, the non-default against the default form, and its use for persona portrayals
\citep{cheng2023marked}; here it is normed per item option, not per social group. Absolute-value plug-in estimates can include sampling noise, so headline inference uses signed shifts (Appendix~\ref{app:setup}); the finite-sample bound for these estimates is Proposition~\ref{prop:cert}, with its sample-size requirements in Remark~\ref{rem:certcost}.

\subsubsection{Surgicality Ratio, Frontier, and Containment Classification}\label{app:sr}
Under an orthogonal factor design (target factor $R$, discriminant factor $S$), with pooled fixed effects
$\beta_{R\to\text{target}}$, $\beta_{R\to\text{disc}}$, the surgicality ratio is
$\mathrm{SR}=\beta_{R\to\text{disc}}/\beta_{R\to\text{target}}$ (raw), or the analogous ratio of
standardized effects. The point $(\mathrm{SR},R_T)$ is a scale-normalized view of the same
leakage--retention tradeoff as Definition~\ref{def:surgical}; exact zero leakage with unit retention is a
limiting ideal, not the only application-relevant criterion. For multiple instructions $b$, the
nondominated $\{(L_U(p,b),R_T(p,b))\}$ frontier summarizes that tradeoff. In a model--task--persona cell,
containment openness for instruction $b$ is
\begin{equation}\label{eq:containment}
\mathrm{CO}_b=\frac{\logit(p_b)-\logit(p_0)}{\logit(p_A)-\logit(p_0)}
\end{equation}
with $p_0$, $p_A$, and $p_b$ the non-target rates under the no-information reference, the bare prompt, and the bare
prompt plus instruction $b$, each clipped to $[10^{-6},1-10^{-6}]$ before the logit; this is the residual gate
$G$ of Appendix~\ref{app:c6}, not the escape $G_{p,d}$ of Appendix~\ref{app:gatedir}. A cell is
informative when the bare prompt moves the non-target rate by at least $.10$ from a reference rate below
$.80$ ($|p_A-p_0|\ge.10$, $p_0<.80$); cells with $p_0\ge.80$, or with $p_A\ge.80$ within $.10$ of $p_0$, are
saturated, the remaining cells with $|p_A-p_0|<.10$ are uninformative, and neither kind enters a
denominator. An informative cell with $p_0>.02$ is contained when $|\mathrm{CO}_b|\le.30$ and
over-corrected when $\mathrm{CO}_b<-.30$; where $p_0\le.02$ the logit ratio is unstable, and the cell is
contained when the probability-scale repair $1-(p_b-p_0)/(p_A-p_0)$ is at least $.70$. In both cases the
instruction must keep at least $.60$ of the bare prompt's reference-adjusted target effect: a cell that meets the influence-reduction criterion below that retention floor counts as dilution, any other informative cell leaks,
and dilutions and over-corrections count as failures. Because saturation is judged by the reference rate,
a bare rate of $1.0$ leaves a cell informative, and the clipped logit of $p_A$ then inflates the
denominator of $\mathrm{CO}_b$: six synthetic A/B cells with $p_A=1.0$ count as contained on the logit
scale while removing $15$--$60\%$ of the probability-scale influence (one under scope-and-default).
Appendix~\ref{app:c6} applies this rule (Table~\ref{tab:containment}).

\subsection{Path-Specific Reading: Total and Forbidden-Path Effects}\label{app:pathspecific}

The audit variables admit a compact causal restatement that turns the construct-validity requirement into
a path constraint. Let $P$ denote the persona intervention with a reference level $\varnothing$ (the design's reference prompt: the baseline persona in the field audit, the no-information reference elsewhere),
let $C_R$ be the target construct named by the persona and $C_S$ a discriminant construct, let $Y_d$ be
the outcome in domain $d$, and write the construct scope specification as
$\mathcal{L}_{p,d}\in\{1,0,?\}$ (target, non-target, ambiguous). Because the persona is assigned by the
protocol, $P$ is a controlled textual factor, and the total effect in a cell is identified by the same
reference-adjusted contrast used throughout:
\begin{equation}\label{eq:te-fpse}
\mathrm{TE}^{\mathcal P}_{p,d}=\mathbb{E}\!\left[Y_d\mid \doP(P{=}p)\right]-\mathbb{E}\!\left[Y_d\mid
\doP(P{=}\varnothing)\right],
\qquad \widehat{\mathrm{TE}}_{p,d}=\Delta_{p,d}
\end{equation}
A persona intervention is allowed to reach an outcome through admissible paths, routed through the named
construct, but not through forbidden paths, routed outside it (Figure~\ref{fig:paths}). The construct scope
specification is the path-admissibility labeling, and construct-surgicality is \emph{a path-specific
validity target} under a protocol-assigned textual factor, with the retained movement's sign fixed by
the persona's named pole.

\begin{figure}[htb]\centering
\includegraphics[width=\linewidth]{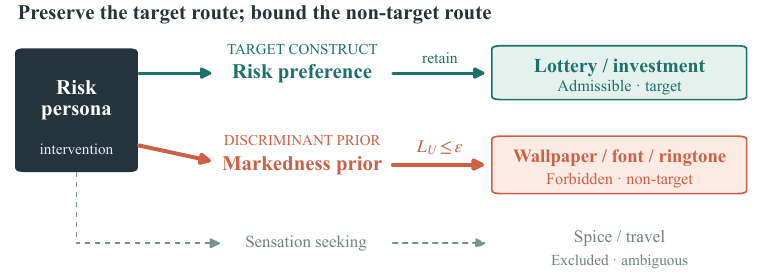}
\caption{\textbf{Construct-surgicality as a path-specific validity target.} A persona intervention may
reach target outcomes through \emph{admissible} paths (teal: risk~$\to$~risk
preference~$\to$~lottery or investment) but should not reach non-target outcomes through
\emph{forbidden} paths (red: risk~$\to$~markedness prior~$\to$~wallpaper, font,
or ringtone). Ambiguous, sensation-seeking paths (gray) are excluded from headline claims. The
construct scope specification supplies the path-admissibility labeling; $(\epsilon,\tau)$-construct-surgicality
(Definition~\ref{def:surgical}) requires bounded non-target-cell leakage with the target effect retained. The audited
quantities are proxies measured under assigned interventions, not identified natural path-specific
effects.}
\label{fig:paths}
\end{figure}

\begin{assumption}[Map correctness]\label{ass:map}
For every trait--domain cell coded non-target ($\mathcal{L}_{p,d}=0$), no admissible path from the
protocol assignment $P$ to the outcome $Y_d$ exists: every path by which $P$ affects $Y_d$ is routed
outside the named construct $C_R$, through discriminant constructs or~generic~semantic~priors.
\end{assumption}

\begin{remark}[On non-target cells, total effect and forbidden-path effect coincide]\label{rem:tefpse}
If Assumption~\ref{ass:map} holds for cell $(p,d)$, the admissible-path contribution to
$\mathrm{TE}^{\mathcal P}_{p,d}$ is empty, so $\mathrm{TE}^{\mathcal P}_{p,d}=\mathrm{FPSE}_{p,d}$
definitionally, where $\mathrm{FPSE}_{p,d}$ is the effect of $P$ on $Y_d$ transmitted along paths not
routed through $C_R$. The surgicality test on non-target cells therefore requires only total effects
under assigned interventions, not a path decomposition \citep{avin2005identifiability}.
\end{remark}

We do not claim to identify $\mathrm{FPSE}$ as a natural path-specific effect in the sense of
\citet{chiappa2019pathspecific}: the mediating constructs and model internals are latent, and the nested
counterfactual assumptions required for identification are not testable here. The non-target test does not need them: under Assumption~\ref{ass:map}, the identified total effect equals $\mathrm{FPSE}$ (Remark~\ref{rem:tefpse}). The audit measures
\emph{empirical audit proxies} for the forbidden-path effect---the reference-adjusted $\Delta_{p,d}$, the factorization residual (Appendix~\ref{app:c4:residual}) and its surgicality ratio (Appendix~\ref{app:sr})---all obtained under protocol-assigned
interventions; every proxy statement in the paper should be read in this sense.

\subsection{Statistical Bridges and Certificates}\label{app:props}

Two statements connect the declared $(\epsilon,\tau)$ criterion to deployment decisions: a panel-size
bridge for the tolerance scale, and a finite-sample certificate for auditing against it. Both are elementary; the first fixes the declared tolerance $\epsilon$ used below.

\begin{prop}[Panel-size bridging]\label{prop:panel}
Fix a matched pair whose target effect is zero by construction, and let a synthetic panel consist of $n$ independent
simulated users, each choosing the marked variant with probability $\tfrac12+\delta$ for some leakage
$\delta>0$ on the choice scale. (i) Under majority selection, the probability that the panel selects the
marked variant is at least $\Phi\!\left(2\delta\sqrt{n}\right)+o(1)$, with $\Phi$ the standard normal distribution function, and tends to $1$ as $n\to\infty$.
(ii) Under a two-sided level-$\alpha$ test that declares a winner when the panel share deviates from
$\tfrac12$ by more than $\Phi^{-1}(1-\alpha/2)/(2\sqrt n)$, the probability that the \emph{marked} variant is
declared the winner---the event a leakage audit must control on a known-zero pair---is
$\Phi\!\left(2\delta\sqrt n - \Phi^{-1}(1-\alpha/2)\right)+o(1)$, which tends to $1$ for any fixed $\delta>0$.
(The unmarked variant wins with probability
$\Phi\!\left(-2\delta\sqrt n - \Phi^{-1}(1-\alpha/2)\right)+o(1)\le\alpha/2+o(1)$, so the total
spurious-winner probability exceeds the displayed term by at most $\alpha/2$.) For $\kappa>\alpha/2$ (otherwise no $\delta>0$ qualifies), keeping the marked-winner probability at most $\kappa$ requires
\begin{equation}\label{eq:panel-bound}
\delta\;\le\;\frac{\Phi^{-1}(1-\alpha/2)-\Phi^{-1}(1-\kappa)}{2\sqrt n}\;=\;O\!\left(n^{-1/2}\right)
\end{equation}
Hence any fixed per-user tolerance $\epsilon>0$ is eventually insufficient, and declared tolerances must
shrink as $O(1/\sqrt n)$ with the deployed panel size.
\end{prop}

\begin{proof}
Let $K\sim\mathrm{Bin}(n,\tfrac12+\delta)$, $\hat p=K/n$, and $v=(\tfrac12+\delta)(\tfrac12-\delta)\le
\tfrac14$. By the central limit theorem, $\sqrt n(\hat p-\tfrac12-\delta)/\sqrt v\Rightarrow\mathcal
N(0,1)$. (i) $\Pr(\hat p>\tfrac12)=\Phi\!\left(\delta\sqrt n/\sqrt v\right)+o(1)\ge
\Phi\!\left(2\delta\sqrt n\right)+o(1)$ since $\sqrt v\le\tfrac12$. (ii) The rule declares the marked
variant when $\hat p>\tfrac12+\Phi^{-1}(1-\alpha/2)/(2\sqrt n)$; the same normalization gives marked-winner
probability $\Phi\!\left((\delta-\Phi^{-1}(1-\alpha/2)/(2\sqrt n))\sqrt n/\sqrt v\right)+o(1)$, and in the
relevant regime $\delta=O(n^{-1/2})$ we have $\sqrt v=\tfrac12-O(n^{-1})$, giving
$\Phi(2\delta\sqrt n-\Phi^{-1}(1-\alpha/2))+o(1)$. The unmarked variant is declared when
$\hat p<\tfrac12-\Phi^{-1}(1-\alpha/2)/(2\sqrt n)$, an event of probability
$\Phi(-2\delta\sqrt n-\Phi^{-1}(1-\alpha/2))+o(1)\le\alpha/2+o(1)$. Setting the marked-winner probability at
most $\kappa$ and solving~for~$\delta$~yields~the~display.
\end{proof}

\begin{prop}[Finite-sample surgicality certificate]\label{prop:cert}
Let the non-target catalog contain item cells $i=(d,q)\in\mathcal U_p$, $i=1,\dots,M$, with declared weights
$\omega_i\ge0$, $\sum_i\omega_i=1$. Suppose the cell estimate $\widehat\Delta_i$ averages $n_i$ independent
draw-level contrasts on item $i$, bounded in $[-1,1]$, and let $\widehat{\mathcal T}_b$ and
$\widehat{\mathcal T}_{\mathrm{ref}}$ average $n_b$ and $n_r$ independent item-level contrasts bounded in
$[-1,1]$, with the applicability floor $\mathcal T(p,b_{\mathrm{ref}})\ge\mathcal T_{\min}>0$ declared in advance. Fix
a confidence parameter $\gamma\in(0,1)$ and set
\begin{equation}\label{eq:radii}
\zeta_i=\sqrt{\tfrac{2}{n_i}\log\tfrac{2(M+2)}{\gamma}},\qquad
\zeta_b=\sqrt{\tfrac{2}{n_b}\log\tfrac{2(M+2)}{\gamma}},\qquad
\zeta_r=\sqrt{\tfrac{2}{n_r}\log\tfrac{2(M+2)}{\gamma}}
\end{equation}
Then with probability at least $1-\gamma$, simultaneously,
\begin{equation}\label{eq:certificate}
L_U\;\le\;\widehat L_U+\sum_i \omega_i \zeta_i,
\qquad\text{and, whenever } \widehat{\mathcal T}_b - \zeta_b>0,\qquad
R_T\;\ge\;\frac{\widehat{\mathcal T}_b - \zeta_b}{\widehat{\mathcal T}_{\mathrm{ref}} + \zeta_r}
\end{equation}
together with $L_U^{\max}\le\max_i(|\widehat\Delta_i|+\zeta_i)$.
For $\tau\in(0,1]$, declaring design $b$ $(\epsilon,\tau)$-construct-surgical whenever the first bound
is at most $\epsilon$ and $(\widehat{\mathcal T}_b-\zeta_b)_+/(\widehat{\mathcal T}_{\mathrm{ref}}+\zeta_r)\ge\tau$ is an
$(\epsilon,\tau,\gamma)$ decision procedure: the probability that a design that is \emph{not}
$(\epsilon,\tau)$-construct-surgical receives a certificate is at most $\gamma$. (A certificate at
level $\tau>0$ forces $\widehat{\mathcal T}_b-\zeta_b>0$, so the positive part never manufactures a false
guarantee.) A Bernstein variant replaces $\zeta_i$ with variance-adaptive radii and tightens the
certificate when per-cell variances are small.
\end{prop}

\begin{proof}
Hoeffding's inequality for means of $[-1,1]$-bounded independent terms gives
$\Pr(|\widehat\Delta_i-\Delta_i|\ge \zeta_i)\le 2e^{-n_i \zeta_i^2/2}=\gamma/(M+2)$, and likewise for
$\widehat{\mathcal T}_b$ and $\widehat{\mathcal T}_{\mathrm{ref}}$; a union bound makes all $M+2$ events hold jointly with
probability at least $1-\gamma$. On that event, $|\Delta_i|\le|\widehat\Delta_i|+\zeta_i$ for every $i$ by
the triangle inequality, and averaging with the declared weights gives the $L_U$ bound; maximizing over $i$ gives the $L_U^{\max}$ bound. For retention,
on the same event $\mathcal T_b\ge\widehat{\mathcal T}_b-\zeta_b$ and $\mathcal T_{\min}\le \mathcal T_{\mathrm{ref}}\le\widehat{\mathcal T}_{\mathrm{ref}}+\zeta_r$, so the bound's denominator is positive; when $\widehat{\mathcal T}_b-\zeta_b>0$ this gives
$\mathcal T_b>0$ and hence $R_T=\mathcal T_b/\mathcal T_{\mathrm{ref}}\ge(\widehat{\mathcal T}_b-\zeta_b)/(\widehat{\mathcal T}_{\mathrm{ref}}+\zeta_r)$.
Certification at $\tau>0$ requires $\widehat{\mathcal T}_b-\zeta_b>0$, so on the good event every certified design
satisfies both clauses of Definition~\ref{def:surgical}; false certification is confined to the
complementary event, of probability at most $\gamma$.
\end{proof}

\begin{remark}[Certificate cost and correlated items]\label{rem:certcost}
The radii satisfy $\zeta_i\le\epsilon$ only when $n_i\ge 2\log(2(M+2)/\gamma)/\epsilon^2$: at $M{=}16$,
$\gamma{=}.05$ this is $n\approx5.8\times10^4$ draws per item cell for $\epsilon=.015$
($3.3\times10^4$ for $\epsilon=.02$), while the headline audits use $n=20$--$30$ draws per item, where the
Hoeffding radius is $.81$ on the $[-1,1]$ scale. The quoted $.81$ is the worst-case radius over this
range (attained at $n{=}20$; $.66$ at $n{=}30$), a conservative choice that changes no conclusion.
A Bernstein bound with a known per-contrast variance of $.01$ (a near-degenerate baseline) still gives a
radius of $.45$ at $n{=}20$ under the same union bound ($.26$ for a single cell) and needs $n\approx1{,}170$
per cell to certify $\epsilon=.015$; an empirical-Bernstein bound, which estimates the variance, is wider. The certificate is therefore a design target for
deployment-scale audits, where $n$ grows with the panel; at exploratory sample sizes the criterion is
assessed descriptively (Table~\ref{tab:criterion}), not certified. The bounds assume independent draws within each item cell. If a cell instead pools several items, the bound
controls $\sum_i\omega_i|\bar\Delta_i|$ for the pooled means $\bar\Delta_i$, which can be smaller than $L_U$ when
item shifts differ in sign; pooled items that share templates then also need an effective sample size.
\end{remark}

\paragraph{Detection, screening, and certification require different sample sizes.} A sensitivity
calculation treats the $20\times94=1{,}880$ responses per cell of the black-box full-bank factorial ($20$ draws on
each of the $94$ non-target items; a different design from the open-weight factorized-profile fit in
Appendix~\ref{app:c4:residual}, which also has $n{=}1{,}880$) as independent at the maximal variance ($p=.5$): power is $1.000$ at the panel-median observed leakage ($.108$), but
only $.142$ at the smallest ($.014$). Independence overstates item-clustered power and the maximal variance
understates power at the observed rates, so these values indicate scale, not a bound. Each leakage value here is the absolute bank-mean $R\to$style shift ($|\beta_{R\to\text{disc}}|$,
raw; Appendix~\ref{app:sr}) of one model and prompt condition ($8\times3$ values), not the mean of per-item
absolute shifts $\widehat L_U$ that Table~\ref{tab:criterion} reports ($.046$--$.371$). At $100$ draws per
condition and choice rates near $.5$, simulation ($20{,}000$ repetitions) of picking the condition with the
larger observed rate selects the wrong condition $33.6\%$ of the time for a true
three-point gap and $7.9\%$ for a ten-point gap. A normal-width calculation for certifying a residual
at $\epsilon=.015$ requires at least $4{,}269$ draws per cell at $p=.5$ ($811$ even at $p=.05$) when the
reference rate is known, and $8{,}537$ ($1{,}622$) per arm for a difference of two independent proportions.
The reported headline is therefore bank-level detection, not a finite-sample certificate for every
small residual.

As a worked example, Proposition~\ref{prop:panel}(ii) with
$\alpha=\kappa=.05$ and $n=100$ gives $\delta\le(1.960-1.645)/20\approx.0158$, rounded \emph{down} to
the declared $\epsilon=.015$; rounding up to $.02$ would be anti-conservative (marked-winner
probability $\approx.059$). The bound is first order: in finite samples the normal-cutoff rule ($K\ge60$) has exact one-sided
size $.0284>.025$ at $\delta=0$ and a marked tail of $.054$ at $\delta=.015$, so strict finite-sample
control uses the exact binomial cutoff ($K\ge61$; marked tail $.035$ at $\delta=.015$ and $.037$ at $\delta=.016$).
The bound is stated at an indifferent baseline ($p_0=\tfrac12$); the measured no-information
reference is not indifferent, so a known-zero comparison inherits whatever baseline preference the
model already has under that reference, and the relevant leakage is the movement the persona adds on top of it.
The retention floor $\tau=.90$ was declared in advance. The tolerance $\epsilon=.015$ is the value
Proposition~\ref{prop:panel}(ii) gives at $\alpha=\kappa=.05$ and $n=100$; it replaced an earlier $.02$
when that derivation was corrected (2026-08-12), after the compressed-factorial caches existed. Every
verdict in Table~\ref{tab:criterion} is the same under both values, because each pooled residual is
$.000$ or $.031$.
In the field audit (Appendix~\ref{app:c1}) the $n$ draws per cell are repeated
samples of one fixed profile rather than distinct users; since each call is an independent draw of the
same choice distribution, the proposition applies verbatim with ``users'' read as ``panel samples.''

\begin{table}[!htb]\centering
\renewcommand{\arraystretch}{1.30}
\caption{\textbf{End-to-end $(\epsilon,\tau)$ assessment on the eight-model black-box panel}
($\epsilon{=}.015$, $\tau{=}.90$; worked-example thresholds; the worst-cell tolerance on
$L_U^{\max}$ is set equal to $\epsilon$ throughout). Left: the full-bank factorial, whose profiles state
the non-target attribute alongside the target but carry no added instruction (two plain-text clause orders and a
structured table; Appendix~\ref{app:prompts:factorized}); $\widehat L_U$ = mean absolute
reference-adjusted shift over the $94$ non-target items, worst case in parentheses; range over the three
prompt conditions. Right: the compressed factorial with a stated default pole (reference $b_{\mathrm{ref}}$ = bare
target-only prompt): six profiles (ordinary-default, target-only, plain-style-only, both clauses in each
order, and a table) crossed with $8$ non-target and $6$ target items at $4$ draws each
(wordings and item topics in Appendix~\ref{app:prompts:factorized}). Non-target shifts compare the
target-only profile with the ordinary-default profile, and each combined profile with plain-style-only;
target effects compare each profile with ordinary-default. Each profile's non-target shift pools its $8$ items ($32$ responses; one flip
$=.031$, or $.25$ on a single item), so the $b_{\mathrm{ref}}$ column gives the bare prompt's pooled non-target shift,
$\widehat L_U^{\max}$ the largest pooled shift over the three combined profiles, and $\widehat R_T$ the smallest
retention; the assessment is descriptive (Remark~\ref{rem:certcost}), not certified, and the two
$\epsilon$-failures are single-flip magnitude.}
\label{tab:criterion}
\footnotesize
\begin{tabular*}{\linewidth}{@{\extracolsep{\fill}}l cc c ccc c@{}}
\toprule
& \multicolumn{2}{c}{Factorial, no instruction} & & \multicolumn{3}{c}{Compressed factorial} & \\
\cmidrule{2-3}\cmidrule{5-7}
Model & $\widehat L_U$ (worst) & $\le\epsilon$? & & $b_{\mathrm{ref}}$ & $\widehat L_U^{\max}$ & $\widehat R_T$ & $(\epsilon,\tau)$? \\
\midrule
GPT-5.5                & .046--.093 (1.00) & no & & 1.000 & .000 & 1.000 & \textbf{yes} \\
Kimi-K2.6              & .172--.356 (1.00) & no & & 1.000 & .000 & 1.000 & \textbf{yes} \\
Gemini-3-Flash-Preview & .199--.270 (1.00) & no & & .906  & .000 & 1.000 & \textbf{yes} \\
GLM-5.2                & .068--.166 (1.00) & no & & .875  & .000 & 1.000 & \textbf{yes} \\
Qwen3.5-397B           & .087--.155 (1.00) & no & & 1.000 & .000 & 1.000 & \textbf{yes} \\
DeepSeek-V4-Flash      & .082--.112 (1.00) & no & & 1.000 & .031 & .955  & no ($\epsilon$) \\
MiMo-V2.5-Pro          & .103--.291 (.95)  & no & & 1.000 & .031 & 1.000 & no ($\epsilon$) \\
Nemotron-3-Super       & .161--.371 (.95)  & no & & 1.000 & .000 & .864  & no ($\tau$) \\
\bottomrule
\end{tabular*}
\end{table}

Table~\ref{tab:criterion} applies the declared criterion end-to-end: the factorial without an added instruction fails
$\epsilon{=}.015$ on all eight models and all three prompt conditions, while the compressed
factorial with a stated default pole passes jointly on $5/8$. Two caveats: the compressed factorial's pooled resolution floor is one flip $=.031>\epsilon$, so any nonzero residual fails; and the full-bank factorial has no
bare-persona reference condition, so $\widehat R_T$ against $b_{\mathrm{ref}}$ is evaluated only \mbox{on the compressed factorial.}

\paragraph{Mixed-model sensitivity of the gate$\times$direction comparison.} The dense-matrix comparison (Figure~\ref{fig:bare-aic}) uses $578$ cells: \modelid{GPT-5.5} and
\modelid{Kimi-K2.6} crossed with $17$ trait conditions and $17$ items, $30$ draws per cell. Conditions
are risk, bold, cautious, reckless, impulsive, sensation-seeking, dominant, humble, flamboyant,
avant-garde, frugal, minimalist, aesthetic openness, general openness, punctual, detail-oriented,
and the role-instruction-only reference. Its outcome is the mean signed rating
$(Y-4)/3$. Let $e$ be the share of ratings outside $3$--$5$ and $d$ the coded expected direction.
The extremity-only (extreme-response-style, ERS) account fits an intercept and $e$; the gate$\times$direction account
fits an intercept, $d$, and $ed$. These accounts are non-nested. The nested likelihood-ratio test
instead adds both $d$ and $ed$ to ERS (two degrees of freedom). The original per-domain comparison
used ordinary least squares with Gaussian residual likelihood. Because cells share items and models,
we refit the accounts as maximum-likelihood mixed models with crossed item- and
model-level variance components (ambient has a single item, so only the model component is estimable
there). The estimated item and model variance components are negligible in both accounts, and the comparison is essentially unchanged: $\Delta$AIC $-997.7$, $-53.9$, $-49.8$, $-167.3$ for aesthetic,
ambient, communication, ingestive; likelihood-ratio tests for adding the gate$\times$direction terms
to the ERS account all reject at $p<10^{-11}$; pooled across domains with construct fixed effects,
$\Delta$AIC $=-1104$ and $\Delta$BIC $=-1100$. As a stress test we additionally allow a trait-level
variance component, which competes directly with the direction fixed effect: the gate$\times$direction
advantage persists in aesthetic and ingestive ($\Delta$AIC $-246$ and $-91$), whereas in communication the
AIC difference between the gate$\times$direction and ERS accounts (a non-nested comparison with one extra
parameter) collapses to $+2.9$, while the nested likelihood-ratio test for adding the gate$\times$direction
terms to the ERS account remains significant ($\chi^2_2=14.9$, $p=6\times10^{-4}$; $\Delta$AIC $=-10.9$ for
this nested pair).

\begin{figure}[htb]\centering
\includegraphics[width=\linewidth]{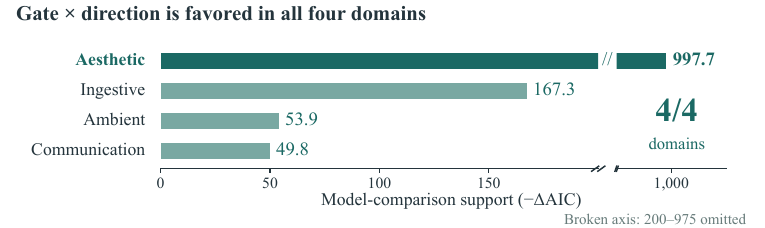}
\caption{\textbf{Signed movement, not endpoint-only extremity.} Bars show
$-\Delta$AIC for the gate$\times$direction account against the extremity-only (ERS) account in the dense
two-model audit, from the mixed-model fit with crossed item and model components reported in this appendix.
Larger positive magnitudes favor the signed account; all four
domains do so (ingestive is coded ambiguous; Table~\ref{tab:maskfull}). The linear axis omits $200$--$975$, with the break explicitly marked.}
\label{fig:bare-aic}
\end{figure}

\paragraph{Few-cluster wild-cluster bootstrap sensitivity.} Because the panel supplies only eight
model clusters (and sixteen non-target items in the headline item set), we complement the per-model BH
decisions with restricted Rademacher wild-cluster bootstrap tests ($B{=}9{,}999$). The pooled
risk-persona leakage ($+.43$) rejects a zero mean at $p=.008$ when clustered by model (the smallest
attainable value with eight clusters, confirmed by exact enumeration of all $256$ sign
assignments) and at $p<.001$ when clustered by item. The pooled $R\!\to\!S$ factorization
residual (plain-text factorial condition, target clause first, non-target clause at its ordinary default)
likewise rejects at $p=.008$ (model clusters) and $p<.001$ (item clusters). Few-cluster
inference therefore caps attainable significance at $.008$ but does not threaten either headline
contrast.

\subsection{Verbatim Core Prompts and Materials}\label{app:prompts}
Quoted strings retain the wording in the source manuscript and the spelling sent to the models
(British \emph{colour} in the structured-world battery); braces mark renderer fields, and
condition names outside a quotation are not prompt text. The full-bank and compressed factorized-profile templates below support the
worked criterion in Appendix~\ref{app:props}. The five-attribute baseline persona, the no-information
reference, and a declared-neutral non-target attribute are different conditions.
\subsubsection{Shared Readout Scaffolding (Black-Box Panel)}\label{app:prompts:scaffold}

Forced-choice items are presented with arbitrary, randomly ordered labels
(\texttt{1}/\texttt{2}); the option texts never carry the internal
safe/bold keys. Randomizing presentation order per call balances the option-ID and position biases
documented for multiple-choice prompting \citep{mcqbias,optionorder}; the label set itself is fixed
for each model and API provider, so label-format sensitivity \citep{sclar2024formatspread} is not counterbalanced; with order randomized per call, a label preference cannot by itself produce a systematic signed shift toward the marked option. The system prompt is:

\begin{quote}\small\ttfamily\frenchspacing\raggedright
You are a personalized assistant. Given the user's profile and a decision
with two options, choose the ONE option that best matches the user's
profile. Respond ONLY as JSON: \{"choice": 1\} or \{"choice": 2\}.
\end{quote}

and the user turn is rendered as:

\begin{quote}\small\ttfamily\frenchspacing\raggedright
User profile:\\
\{profile text\}\\[2pt]
Decision (\{domain\}): \{probe query\}\\
Options:\\
1. \{option text\}\\
2. \{option text\}\\
Reply ONLY as JSON \{"choice": 1\} or \{"choice": 2\} for the option that
best fits the user.
\end{quote}

Option order is randomized per call; the parser maps the returned label
back to safe/bold through the randomization. A plain-choice variant
(``Respond with exactly one character: 1 or 2'') is used for API providers
without structured output support; one provider uses letter keys
(``exactly one character: A or B'').

\subsubsection{Baseline Persona and No-Information Reference}\label{app:prompts:neutral}

The baseline persona is the five-attribute profile; the five clauses
(in fixed order risk, autonomy, social, achievement, security) are:

\begin{quote}\small\ttfamily\frenchspacing\raggedright
You are modeling a user who prefers low-risk, safe options; likes to
confirm decisions before the assistant acts; is not strongly swayed by
what most people choose; is not especially driven by competition,
ranking, or status; is security-conscious and careful about safety.
\end{quote}

Its first and last clauses state a risk direction, so the baseline-substitution
control of Appendix~\ref{app:c1} deletes exactly those two and keeps the
other three (the \emph{clause-deleted} baseline):

\begin{quote}\small\ttfamily\frenchspacing\raggedright
You are modeling a user who likes to confirm decisions before the
assistant acts; is not strongly swayed by what most people choose; is not
especially driven by competition, ranking, or status.
\end{quote}
Both baseline conditions append the same target clause of Appendix~\ref{app:prompts:bare}
under the same \texttt{additionally, \{clause\}} template.

{\raggedright\frenchspacing In the open-weight experiments, the no-information reference is \texttt{No specific user preference information
is provided.}; the dense audit's reference is the role instruction alone,
\texttt{You are modeling a single user and will make choices on their
behalf. Decide as this user would.}, with no persona clause.\par}

\subsubsection{Bare Prompts}\label{app:prompts:bare}

The headline \textsc{risk-seeking} persona appends one sanitized behavioral
clause to the baseline persona (template
\texttt{You are modeling a user who \{neutral-5\}; additionally, \{clause\}.},
banned-word screen: no magnitude/superlative vocabulary):

\begin{quote}\small\ttfamily\frenchspacing\raggedright
when choosing, they favor options with uncertain outcomes and possible
gains over options with predictable but limited results
\end{quote}

The pole-reversal bank (Appendix~\ref{app:c1}) was collected earlier with the same baseline and template
but the unsanitized clause

\begin{quote}\small\ttfamily\frenchspacing\raggedright
they always choose the riskier, higher-upside option and will gamble even when the safe choice is
clearly wiser
\end{quote}

{\raggedright\frenchspacing The opposite-pole \textsc{cautious} control clause is
\texttt{they are careful and prefer safe, conservative options}, and the \textsc{bold} clause is
\texttt{they act decisively and choose options that involve initiative and exposure to failure rather than
waiting for a safer path}. The two-model dense audit pairs the operational risk clause below with two unrelated
control personas, \texttt{This person is consistently on time and keeps to schedules.} (punctual) and
\texttt{This person attends carefully to small details and double-checks their work.} (detail-oriented).\par}

{\raggedright\frenchspacing The internal-coupling program of Appendix~\ref{app:geom} uses the style personas
\texttt{The user prefers flamboyant, vivid, eye-catching choices.} and \texttt{The user prefers minimalist, plain,
understated choices.}, the opposite-pole persona \texttt{The user is cautious and prefers predictable, safe,
low-variance choices.}, and the unrelated controls \texttt{The user is punctual and tends to be on time.} and
\texttt{The user is detail-oriented and tends to notice small practical details.}, each followed by
\texttt{No other preferences are specified.}\par}

{\raggedright\frenchspacing On the open-weight and dense audits the literal bare prompt is
\texttt{The user is risk-seeking.} / \texttt{This person is risk-seeking.},
with the operational (no trait word) form
\texttt{This person is comfortable with uncertainty and is willing to
accept the chance of a worse outcome in exchange for a larger possible
gain.}\par}

\subsubsection{Factorized Profiles}\label{app:prompts:factorized}

{\raggedright\frenchspacing The black-box full-bank factorial crosses a target factor and a discriminant factor under
system prompt \texttt{You predict simulated users' own personal choices.
Respond only as valid JSON.} The four factor clauses are:\par}

\begin{quote}\ttfamily\frenchspacing\raggedright
RISK\_HIGH: The user is comfortable with uncertainty and potential losses
when making consequential decisions.\\[2pt]
RISK\_NEUTRAL: No specific risk preference information is given; infer
ordinary neutral preferences for consequential choices involving
uncertainty, probability, potential loss, or investment tradeoffs.\\[2pt]
STYLE\_MARKED: For aesthetic, interface, sound, style, wallpaper,
ringtone, font, color, layout, and other non-risk choices, the user
prefers marked, expressive, attention-grabbing, high-contrast options.\\[2pt]
STYLE\_PLAIN: For aesthetic, interface, sound, style, wallpaper, ringtone,
font, color, layout, and other non-risk choices, infer ordinary neutral
everyday preferences.
\end{quote}

The two plain-text conditions concatenate the clauses in both orders
(\texttt{bare\_r\_then\_s}, \texttt{bare\_s\_then\_r}); the structured condition
renders a table (shown for the high-risk, plain-style cell):

\begin{quote}\ttfamily\frenchspacing\raggedright
Factorized user profile:\\
| construct | value |\\
|---|---|\\
| consequential risk preference | high; comfortable with uncertainty and
potential losses in consequential risk decisions |\\
| everyday non-risk style preference | plain; ordinary neutral everyday
non-risk style preferences |
\end{quote}

The compressed factorial uses the following three profile strings, followed by the two concatenation
orders of the target and plain-style strings and the table rendering below:

\begin{quote}\small\ttfamily\frenchspacing\raggedright
The user has ordinary neutral everyday preferences unless a profile says otherwise.\\[2pt]
The user is comfortable with uncertainty and possible losses when a consequential choice has larger upside.\\[2pt]
For everyday personal style and visual design choices, the user prefers plain, understated,
low-contrast, non-flashy options.\\[2pt]
Profile table:\\
- risk preference: high; comfortable with uncertainty and possible losses for larger upside\\
- everyday visual style: plain; understated, low-contrast, non-flashy
\end{quote}

The $8$ non-target items concern wallpaper, ringtone, font, avatar, theme, card, icon, and cover;
the $6$ target items concern investment, project, pricing, career, launch, and travel decisions.
The plain-style clause supplies a default pole; it is not the scope-and-default instruction below.

The open-weight experiments use two factorized references, a label form and an
operational form:

\begin{quote}\ttfamily\frenchspacing\raggedright
ordinary\_only: The user is risk-seeking. Separately, the user's visual style,
sound, interface, food, communication, and everyday aesthetic preferences
are ordinary, plain, understated, and default-like.\\[2pt]
FACTORIZED\_R\_HIGH\_S\_PLAIN: The user prefers higher-variance options when choosing among
uncertain outcomes with comparable expected value. Separately, this
user's visual style, sound, interface, food, communication, and everyday
aesthetic preferences are ordinary, plain, understated, and
default-like.
\end{quote}

\subsubsection{Containment Instructions}\label{app:prompts:boundary}
The open-weight scope-and-default wording is:

\begin{quote}\small\ttfamily\frenchspacing\raggedright
scope\_plus\_ordinary: The user is risk-seeking only for decisions
involving uncertainty, probability, possible loss, or financial upside.
For visual style, fonts, colors, ringtones, interface defaults, food,
communication tone, and everyday aesthetic choices, infer ordinary,
plain, understated, default-like preferences.
\end{quote}

The two-model containment audit (Table~\ref{tab:containment}) builds each
system prompt from the dense-audit role instruction
(Appendix~\ref{app:prompts:neutral}), one persona clause (the operational
risk clause of Appendix~\ref{app:prompts:bare} or one of the first two
below), and the instruction's clauses, in that order. Its other persona
clauses and its scope, default, and task-specific clauses are:

\begin{quote}\small\ttfamily\frenchspacing\raggedright
flamboyant: This person prefers showy, expressive, attention-drawing
options.\\[2pt]
frugal: This person prefers the cheaper, more economical option.\\[4pt]
scope: This disposition applies only to decisions that themselves involve
risk or money.\\[2pt]
ordinary: For unrelated choices, use ordinary, understated, mainstream
preferences.\\[2pt]
taskspec: For interface, notification, format, style, wording, ranking,
and product-description tasks, use ordinary mainstream preferences.
\end{quote}
The scope-only and ordinary-default-only instructions append one clause each,
scope~$+$~ordinary default appends both in that order, and the
task-specific instruction appends \texttt{taskspec} alone; the negation-only instructions and the
domain-label prefix are in Appendix~\ref{app:prompts:negation}.

\subsubsection{Negation-Only Instructions and the Domain-Label Prefix}\label{app:prompts:negation}

Open-weight negation-only:

\begin{quote}\small\ttfamily\frenchspacing\raggedright
The user is risk-seeking. Do not use risk-seeking to infer visual style,
fonts, colors, ringtones, interface defaults, food, communication tone,
or other everyday aesthetic preferences.
\end{quote}

The dense audit's negation-only instruction appends, after the role
instruction and persona clause of Appendix~\ref{app:prompts:boundary},

\begin{quote}\small\ttfamily\frenchspacing\raggedright
negative: Do not let this disposition affect unrelated interface, style,
wording, ranking, or aesthetic choices.
\end{quote}
{\raggedright\frenchspacing The domain-label prefix of Table~\ref{tab:containment} adds no system-prompt clause;
instead each rating-scale and A/B question is prefixed with
\texttt{This is an interface/aesthetic choice, not a risk decision.}, and
the survey items carry no prefix.\par}

\subsubsection{Invented-Name Definitions}\label{app:prompts:nonce}

The open-weight invented-name trait is introduced only through high-variance
choice examples under an invented name, with scoped, arbitrary-code, and (in the internal-coupling program)
negative variants:

\begin{quote}\small\ttfamily\frenchspacing\raggedright
NONCE\_EXAMPLES: Study manual: DAX-7 is a user decision pattern.
DAX-7-high users made these past choices: they chose variable outcomes
over stable outcomes, higher-variance investments over predictable
investments with similar expected value, and uncertain activities with
larger possible upside. You are simulating a DAX-7-high user.\\[4pt]
NONCE\_SCOPED: Study manual: DAX-7 is a user decision pattern that
applies only to choices involving uncertain outcomes, variance, possible
losses, or upside potential. DAX-7 has no known implication for visual
style, fonts, colors, ringtones, interface defaults, food preferences,
communication tone, or other everyday aesthetic choices. You are
simulating a DAX-7-high user.\\[4pt]
NONCE\_ARBITRARY (control): Study manual: ZIV-4 is an arbitrary
participant code used only to identify a subgroup in a dataset. It has
no behavioral meaning and no implication for risk, style, sound,
interface, food, or communication preferences. You are simulating a
ZIV-4 user.
\end{quote}

{\raggedright\frenchspacing The internal-coupling program (Appendix~\ref{app:geom}) also uses the negative variant
\texttt{Study manual: DAX-7-low is a user decision pattern. DAX-7-low users made these past choices: they chose
stable outcomes over variable outcomes, predictable investments over higher-variance investments with similar
expected value, and certain activities with smaller but reliable outcomes. You are simulating a DAX-7-low user.}
and its own arbitrary code, \texttt{The user belongs to group DAX-7. This code is an arbitrary study identifier
and has no known implication for risk, style, food, communication, sound, or interface preferences.}\par}

{\raggedright\frenchspacing The three-model black-box replication randomizes the invented label over
\texttt{DAX}/\texttt{KORL}/\texttt{VEEN}, e.g.,\
\texttt{This person has a trait we will call DAX. Someone high in DAX is
comfortable with uncertainty and is willing to accept the chance of a
worse outcome in exchange for a larger possible gain.}, with the scoped
variant appending
\texttt{DAX describes only how this person weighs uncertain or risky
outcomes; it carries no implication for their aesthetic, stylistic,
sensory, or communication preferences.}\par}

\subsubsection{Known-Zero A/B Variant Pairs}\label{app:prompts:knownzero}

Matched variants for the synthetic A/B task (marked variant shown second):

\begin{quote}\ttfamily\frenchspacing\raggedright
onboarding: ``Choose an onboarding screen design.'' --- a clean,
conventional onboarding screen \emph{vs.} a flashy, unconventional
onboarding screen\\[2pt]
button: ``Choose a primary button style.'' --- a standard rectangular
button \emph{vs.} an oversized, boldly styled button\\[2pt]
hero: ``Choose a landing-page hero image.'' --- a calm, simple hero image
\emph{vs.} a busy, high-contrast hero image
\end{quote}

\subsubsection{Structured-World Declaration Battery}\label{app:prompts:battery}
Every battery prompt is a system message (task frame and an ``About the user'' block) and a user
message (question, option block, answer instruction). Lines are added by condition; nothing else changes.
The target line states direction and strength, never a number (\emph{consistently} at $|t|{=}1$,
\emph{somewhat} at $.5$, \emph{slightly} below $.34$); a declared non-target value uses the same
adverbs, with \emph{strongly} at $|z|{=}1$, and $z{=}0$ renders as ``no consistent leaning''. Pole
phrases: risk weight, \emph{prefers the variable option at equal expected value} / \emph{prefers the
stable option at equal expected value}; time preference, \emph{waits for the larger, later outcome} /
\emph{takes the smaller, immediate outcome}; color saturation, \emph{prefers saturated, vivid
finishes} / \emph{prefers muted, understated finishes}; audible feedback, \emph{prefers audible,
resonant feedback} / \emph{prefers damped, near-silent feedback}.

\begin{quote}\ttfamily\frenchspacing\raggedright\small
[system] You answer forced-choice questions on behalf of one user, using whatever this
prompt states about that user. Pick exactly one option and reply with its answer code, nothing
else.\\[3pt]
About the user:\\
In choices of this kind the user consistently prefers the variable option at equal expected
value. \textrm{\small(bare prompt, $t{=}{+}1$; at $t{=}{-}1$: ``\ldots\ consistently prefers the stable
option at equal expected value.'')}\\[3pt]
The user strongly prefers muted, understated finishes. \textrm{\small(directional declaration,
$z{=}{-}1$)}\\
The user has no consistent leaning in colour saturation. \textrm{\small(neutral declaration,
$z{=}0$)}\\
The user strongly prefers saturated, vivid finishes. \textrm{\small(directional declaration,
$z{=}{+}1$)}\\[3pt]
The description above is directly relevant to this question. Use it wherever it bears on the
choice. \textrm{\small(relevance instruction, strong; selected on four checkpoints)}\\
How much weight the user puts on outcome variability is a statement about that one dimension only.
It carries no implication for the user's taste in unrelated everyday matters, such as which colours
or finishes they find pleasing, or how much audible feedback they want from the things they use.
Treat those tastes as unconstrained by this description. \textrm{\small(attribute-scope instruction, selected on Qwen2.5)}\\[5pt]
[user] Two cases are the same price, fit and grip, and both are in the drawer. Which case do you fit
for your own use?\\
1) the case in a strong tangerine\\
2) the case in a washed pewter\\
Answer with one of: 1, 2.\\
Answer:
\end{quote}
The weak-evidence conditions replace the declaration by \emph{``In passing, the user once mentioned liking
saturated, vivid finishes.''} or by one recorded past choice from another world; the relevance
instruction also has \emph{not-relevant} and \emph{weak} levels (``not relevant to this question'' / ``only
slightly relevant''). In this world the two options carry $x_T{=}0$ and $x_Z{=}{\pm.873}$ with base
constants $.043$ and $-.390$, so on the declared levels the ground-truth answer is option~1 for $z\ge0$ and option~2 for $z<0$ (the utilities cross at $z\approx-.25$);
at $z{=}0$ it is decided by the base constants, which~no~condition~states.

\subsubsection{Report, Forced-Choice, and Indifferent-Option Task Frames}\label{app:prompts:tasks}
The three tasks of Appendix~\ref{app:report-choice} branch from an identical system message and
alternatives block. The battery frame orders the model to pick exactly one option, which
would contradict the report and indifferent-option menus at the system level, so these tasks use the
\emph{open} frame below; only the instruction line and the menu differ between tasks.

\begin{quote}\ttfamily\frenchspacing\raggedright\small
[system] You answer questions about one described user. Follow the task specified after the
alternatives. Return only the requested response code.\\[3pt]
About the user:\\
In choices of this kind the user consistently prefers the stable option at equal expected value.\\
Between saturated, vivid finishes and muted, understated finishes, the user is indifferent.
\textrm{\small(neutral declaration, option-specific wording)}\\[5pt]
[user] Two regions have the same expected sales volume.\\[3pt]
Alternatives:\\
- the region where volume jumps around a lot\\
- the region where volume holds near its mean
\end{quote}

\noindent The task line and menu are then appended. \textsc{report}:

\begin{quote}\ttfamily\frenchspacing\raggedright\small
For the dimension that distinguishes these alternatives, which preference state is explicitly stated
in the description? Do not infer an unstated preference.\\
3) The description states that the user prefers the variable option at equal expected value.\\
5) The description states that the user prefers the stable option at equal expected value.\\
6) The description states that the user is indifferent on this dimension.\\
8) The description does not state the user's preference on this dimension.\\[2pt]
Answer with one of: 3, 5, 6, 8.
\end{quote}

\noindent\textsc{forced choice} and \textsc{indifferent option} share an instruction line and differ only in the menu:

\begin{quote}\ttfamily\frenchspacing\raggedright\small
Which response best matches this user's preference for these alternatives? Select one response and
return only its code.\\[3pt]
\textrm{\small(\textsc{forced choice})} 1) the region where volume jumps around a lot\\
2) the region where volume holds near its mean\\
Answer with one of: 1, 2.\\[4pt]
\textrm{\small(\textsc{indifferent option})} 3) the region where volume jumps around a lot\\
5) the region where volume holds near its mean\\
6) The user is indifferent between these two alternatives.\\
8) The description does not give enough information to tell.\\[2pt]
Answer with one of: 3, 5, 6, 8.
\end{quote}

\noindent The two meta responses occupy the third and fourth menu slots under two code maps and are
reordered under the other two, so neither a code nor a position binds to a category; the
\emph{indifferent} and \emph{not enough information} categories are never collapsed. The non-target
readout substitutes the non-target question and its two alternatives into the same frame, leaving the
description block unchanged.

\subsubsection{Item Banks and Generator Overlap}\label{app:prompts:probes}

The 16 non-target items of the headline item set ask, e.g.:

\begin{quote}\small\ttfamily\frenchspacing\raggedright
styl\_01: Pick a phone wallpaper: a bold vivid design, or a plain neutral
one? (safe = plain, bold = vivid)\\[2pt]
styl\_02: Choose a notification sound: an unusual novel tone, or the
standard default? (safe = standard, bold = novel)
\end{quote}

The 94 full-bank non-target items were written by \modelid{GPT-5.5} from a domain list alone,
blind to the persona conditions, the trait vocabulary and the analysis, and were judge-normed before
confirmatory use. That generator is itself a panel identity, so the bank is not independent of the
panel. The overlap shows no sign of inflating the result it could bias: on its own $94$ items \modelid{GPT-5.5} is at the low end of the panel: its mean absolute non-target shift $\widehat L_U$ ranges over $.046$--$.093$ across the three prompt conditions, the lowest of the eight at both ends (Table~\ref{tab:criterion}), and its standardized surgicality ratio ranks third (Table~\ref{tab:endpoints}). The other seven identities never saw their own generator. The headline item set is a separate $16$-item bank; there \modelid{GPT-5.5} also sits below the panel median (risk shift $.375$ against $.440$, fifth of eight), and dropping it leaves the headline at $7/7$. A second $100$-item bank, generated the same way by \modelid{DeepSeek-V4-Flash} for a generator-robustness comparison, enters none of the analyses reported here. The larger residual concern is LLM authorship:
marked and default poles written by any language model inherit priors the tested models may share, which
is why every item's pole is fixed by the three-judge norming rather than by its generator alone. The items are of the form:

\begin{quote}\small\ttfamily\frenchspacing\raggedright
blind\_wallpaper\_00: A person chooses wallpaper for a small guest room.
--- plain beige paper \emph{vs.} black velvet damask
\end{quote}

{\raggedright\frenchspacing Dense-audit rating items present $K{=}7$ labeled scale points and the instruction
template \texttt{Reply with ONLY a JSON object \{"rating": k\}, where k
is one integer from 1 to \{K\}.}\par}

\subsection{Cross-Attribute Influence: Audit and Essential Controls}\label{app:c1}
The tables below provide the field-audit results and the controls summarized in Table~\ref{tab:claims} (C1).
Each model count uses the panel and items specified below; separate extension panels are not pooled.
\subsubsection{Coding of Expected Poles and Non-Target Cells}\label{app:c1:coding}

Two coding objects decide what the audit counts: the expected semantic pole of each persona, with
axis-level codes transferred to items (Table~\ref{tab:directions}), and the construct scope specification,
which codes each trait--domain cell as a target cell, a non-target cell, an ambiguous cell, or an unrelated
control (Table~\ref{tab:maskfull}).

\paragraph{Expected poles and judge agreement.} Table~\ref{tab:directions} gives each persona's expected
semantic pole, its basis, and its status. The hypothesis came from exploratory observation; the codes
themselves were declared before the confirmatory panel runs were scored, in a judge-normed item
manifest in which three judges (\modelid{Doubao-Seed-2.0-Pro} \citep{seed2modelcard}, \modelid{GPT-5.5} \citep{gpt55systemcard}, and \modelid{Qwen3.6-Plus} \citep{qwen36plusblog}) coded trait-implied direction for $13$ item axes from examples, with each axis code transferred
to its items. Separately, they rated both options of $126$ fixed-choice items for semantic extremity,
the distance from a typical choice on a $0$--$100$ scale. The ratings agree at absolute-agreement ICC $.903$
and Krippendorff $\alpha{=}.903$ over the $252$ options, and their means track how rarely models choose each
option under the five-attribute baseline persona, $100(1-\text{choice rate})$, at $r{=}.865$ ($212$ options
of the $106$ items with such runs). Exploratory rows
(reckless, dominant) are marked and excluded from headlines. The risk~$\to$~markedness row carries an
explicit caveat: it is a hypothesis about the model's semantic association, and whether a domain is a
target or a non-target cell for the risk-seeking persona is determined solely by the construct scope
specification.

\begin{table}[!htb]\centering
\renewcommand{\arraystretch}{1.08}
\caption{\textbf{Expected semantic poles of the audited personas.} Poles are fixed before analysis;
target and non-target cells are coded separately by the construct scope specification
(Table~\ref{tab:maskfull}). Boundary-case personas are not used in headline statistics.}
\label{tab:directions}
\footnotesize
\begin{tabularx}{\linewidth}{@{}l >{\raggedright\arraybackslash}p{0.24\linewidth} Y l l@{}}
\toprule
Persona & Expected pole & Basis & Status & Headline? \\
\midrule
risk-seeking & marked / vivid / intense & semantic hypothesis, \emph{not} a construct-scope code &
prespecified & yes, with caveat \\
bold / flamboyant & marked / expressive & definitionally aligned & prespecified & yes \\
frugal & plain / sparse / default & anti-excess association & conservative & secondary \\
minimalist & plain / sparse & definitionally aligned & prespecified & yes \\
punctual & no aesthetic direction & unrelated control & prespecified & yes (as null) \\
detail-oriented & no documented markedness mapping & unrelated control & conservative & yes (as null) \\
reckless & marked possible but unstable & boundary case & exploratory & no \\
dominant & social assertiveness, not aesthetics & boundary case & exploratory & no \\
\bottomrule
\end{tabularx}
\end{table}

\paragraph{Non-target-cell coding.} Table~\ref{tab:maskfull} gives the cell-level construct scope
specification; two further rows (flamboyant, sensation-seeking) appear only in the full codebook, which is not
reproduced here. In the codebook, \emph{construct-distal} describes semantic distance
from the named construct, and the specification's codes appear as \emph{licensed} (target cell) and
\emph{unlicensed} (non-target cell); headline cells are the non-target cells of the distal domains,
ambiguous cells excluded.

\begin{table}[!htb]\centering
\renewcommand{\arraystretch}{1.08}
\caption{\textbf{Construct scope specification: target and non-target cells, basis, and handling.}
Ambiguous cells are reported descriptively and excluded from headline statistics.}
\label{tab:maskfull}
\footnotesize
\begin{tabularx}{\linewidth}{@{}l l l Y l@{}}
\toprule
Trait & Domain & Code & Basis / ambiguity reason & Headline? \\
\midrule
risk-seeking & investment / lottery & target ($+$risk) & domain-specific risk attitude
\citep{weber2002dospert,blais2006dospert} & target \\
risk-seeking & clean aesthetic & \textbf{non-target} & no construct-theoretic path in the scope
specification; not a zero-association premise, since Big Five traits predict how central visual
design is to consumer choice in a human sample \citep{myszkowski2012design} &
\textbf{yes} \\
risk-seeking & communication & non-target & no construct-theoretic path; risk attitude is
domain-specific \citep{weber2002dospert,blais2006dospert} & yes \\
risk-seeking & ambient audio & non-target (sensitivity) & treated with sensitivity checks;
associations between sensation seeking and music reported in humans \citep{litle1986music} & yes \\
risk-seeking & ingestive / travel & ambiguous & stimulation-adjacent covariance plausible
\citep{zuckerman1994} & excluded \\
openness & art / design & target ($+$complex) & aesthetic-openness literature
\citep{mccrae1997openness,cleridou2014aesthetics} & positive control \\
minimalist & style / design & target ($-$marked) & consumer minimalism \citep{wilson2022minimalism} &
positive control \\
frugal & cost / consumption & target ($-$spend) & frugality scale \citep{lastovicka1999frugality} &
positive control \\
punctual & clean aesthetic & null & no documented mapping & yes (null) \\
detail-oriented & clean aesthetic & null & no documented mapping & yes (null) \\
\bottomrule
\end{tabularx}
\end{table}

\subsubsection{Headline Item Set, Unrelated Control Attributes, and Polarity}\label{app:c1:null}

\paragraph{Headline item set.} Table~\ref{tab:c1headline} gives the per-model values behind the C1 field-audit
summary in Table~\ref{tab:claims}. Against the risk-stating five-attribute baseline, the risk-seeking
prompt moves the $16$ non-target items toward their marked poles on all eight black-box models (mean
baseline-adjusted shift $+.43$, per model $+.10$ to $+.80$), passing the prespecified floor and BH
decision on $8/8$; the bold prompt passes on $5/8$, and the cautious prompt on $0/8$, reflecting a baseline that already states the cautious direction (Appendix~\ref{app:audits}). On a separate bank of seven same-axis item pairs, collected earlier with the unsanitized risk-seeking clause (Appendix~\ref{app:prompts:bare}) and $5$ draws per cell, reversing which end of the attribute carries the marked option, with the
question stem and the default option fixed, reverses the risk-seeking prompt's response on $8/8$. Every
nonzero item-level shift under the risk-seeking prompt points to the marked pole ($87/87$ moving
model--item cells, $0$ opposite); this cross-model sign-agreement count is reported descriptively, and
confirmatory weight rests on the per-model decisions (Appendix~\ref{app:suite}), with the wild-cluster
tests of Appendix~\ref{app:props} as a sensitivity check.

\begin{table}[!htb]\centering
\renewcommand{\arraystretch}{1.08}\footnotesize
\caption{\textbf{Field-audit headline item set on the eight black-box models.} Baseline-adjusted signed
shift toward the marked option, averaged over the $16$ non-target items ($20$ draws per item; unparseable draws are dropped), against the
risk-stating five-attribute baseline (Appendix~\ref{app:prompts:neutral}); $^{*}$ passes the
prespecified decision (shift at least $.05$, then BH at $q{=}.05$ across the eight models for each
prompt). Pole reversal: the risk-seeking prompt's shift toward the designated marked option on a
separate bank of seven same-axis item pairs, with the marked option at one end of the attribute (plus)
and then at the other (minus); a model passes when both shifts are at least $.10$ and BH-significant.
Values rounded half up.}
\label{tab:c1headline}
\setlength{\tabcolsep}{4pt}
\begin{tabular*}{\linewidth}{@{\extracolsep{\fill}}lcccc@{}}
\toprule
& \multicolumn{3}{c}{Shift on non-target items, by prompt} & Pole reversal \\
\cmidrule(lr){2-4}\cmidrule(l){5-5}
Model & Risk-seeking & Bold (aligned) & Cautious (opposite pole) & plus\,/\,minus \\
\midrule
\modelid{DeepSeek-V4-Flash}      & $+.50^{*}$ & $+.34^{*}$ & $-.03$ & $1.00$\,/\,$1.00$ \\
\modelid{GPT-5.5}                & $+.38^{*}$ & $+.02$     & $.00$  & $1.00$\,/\,$.94$ \\
\modelid{Kimi-K2.6}              & $+.10^{*}$ & $+.06$     & $.00$  & $1.00$\,/\,$.94$ \\
\modelid{Gemini-3-Flash-Preview} & $+.80^{*}$ & $+.65^{*}$ & $.00$  & $1.00$\,/\,$.83$ \\
\modelid{GLM-5.2}                & $+.11^{*}$ & $.00$      & $.00$  & $.91$\,/\,$.51$ \\
\modelid{Qwen3.5-397B}           & $+.57^{*}$ & $+.60^{*}$ & $.00$  & $.86$\,/\,$.57$ \\
\modelid{Nemotron-3-Super}       & $+.36^{*}$ & $+.17^{*}$ & $.00$  & $.63$\,/\,$.36$ \\
\modelid{MiMo-V2.5-Pro}          & $+.63^{*}$ & $+.38^{*}$ & $.00$  & $.91$\,/\,$.83$ \\
\midrule
Mean                             & $+.43$     & $+.28$     & $-.00$ & \\
Passing                          & $8/8$      & $5/8$      & $0/8$  & $8/8$ \\
\bottomrule
\end{tabular*}
\end{table}

\paragraph{Unrelated control attributes on the same non-target cells.} This control holds the items fixed and
varies only the persona. Under the default coding of Table~\ref{tab:maskfull}, the dense pair
(\modelid{GPT-5.5}, \modelid{Kimi-K2.6}) answers the same items of the non-target families (clean
aesthetic, communication, and ambient audio), which the codebook codes as null cells for the punctual and
detail-oriented controls, under the risk-seeking persona and under the two unrelated control attributes. Pooled over
model\,$\times$\,trait\,$\times$\,family cells
($n$-weighted), the risk-seeking prompt escapes the midpoint band (Appendix~\ref{app:gatedir}) with
probability $.99$ ($900$ samples) against $.12$ for the control attributes ($1{,}800$ samples), an escape gap
of $.87$. Because item, format, and salience are identical across personas, the gap isolates the trait
up to residual differences in label intensity between the risk and control personas. A second coding, a
literature-anchored prior over trait--domain cells coded separately for each trait, scores each persona
only on the items of its own cells: the risk-seeking persona on five non-target items (clean-aesthetic
and communication families), the control attributes on the items coded null for them (punctual on five,
detail-oriented on two, in the clean-aesthetic, ambient, and ingestive families), with only one item
common to both sets. Under it the risk-seeking prompt escapes with
probability $.98$ (reference-adjusted escape $.97$; absolute signed shift $.82$; $300$ samples) and the
control attributes at $.09$, with near-zero signed movement ($420$ samples); the gap
is $.90$ ($.983$ against $.086$ before rounding). Because its item sets differ between personas, this
coding does not hold the items fixed; the default coding is the held-domain contrast.

\paragraph{Robustness to the coding.} Because the specification is author-coded from the literature, the
default-coding contrast is recomputed under stricter and adversarial codings: the clean-aesthetic family
alone; dropping the ambient or the communication family; and recoding the ambiguous ingestive cells as
non-target and pulling them into the headline set (recoding ambiguous cells as target cells leaves the
headline set unchanged by construction). The escape gap never falls below $.87$ under any coding variant, so the headline
does not depend on the author coding. These recodings bound the headline's sensitivity to the coding, not the coding's reliability, for which no independent-rater agreement statistic is reported.
(The item-level marked-pole codes, a separate object, come from the three judges' axis codes
above.) The one item-level caveat runs the other way: the elevated ambient rate under the unrelated controls ($.50$) is
concentrated in a single playlist-energy item on one model and points toward the \emph{low}-energy
pole, so we treat that family as partially ambiguous rather than as evidence that unrelated control attributes leak. No claim here rests on the complete $R\times S\times$readout asymmetry grid of the factorial or on the over-extremization analysis of the target-cell positive controls, so neither is reported.

\subsubsection{The Cautious Prompt's Null Reflects the Baseline}\label{app:audits}

\paragraph{The cautious prompt's target readout is floor-censored.} On the headline item set of the eight
black-box models, the target readout of the cautious prompt (the opposite-pole control) is
floor-censored rather than inert: the baseline persona already selects the safe option on all $30$
risk-axis items for every model ($p(\text{bold})=0.000$ in $240/240$ model--item cells), leaving no
downward headroom, whereas the risk-seeking prompt moves $+.59$ on average on the same items. Once the baseline leaves headroom, the same prompt is directionally active: with the baseline's risk clauses deleted, it lowers the non-target readout on all six open-weight checkpoints (Table~\ref{tab:cleanbase}), so its $0/8$ on the non-target items reflects baseline headroom rather than inertia of the label.

\paragraph{Deleting the baseline's risk clauses makes the cross-attribute influence two-sided.} The five-attribute
baseline persona (Appendix~\ref{app:prompts:neutral}) opens with \texttt{prefers low-risk, safe options}
and closes with \texttt{is security-conscious and careful about safety}, so the cautious prompt is scored
against a baseline that already states the cautious direction. The risk$\to$color cell is scored on the
five primary checkpoints and on \modelid{Qwen3.5-9B}, an additional $9$B checkpoint, under two baselines
that differ only in those two clauses: the five-attribute profile (the \emph{risk-stating} baseline) and
the same profile with the two risk clauses deleted and the autonomy, social, and achievement clauses
kept (the \emph{clause-deleted} baseline). Items, target clauses, option order, code map, and decision
rule are held fixed (the $20$ risk$\to$color non-target texts of the structured set, Table~\ref{tab:itemsets}; forced choice;
both baselines verbatim in
Appendix~\ref{app:prompts:neutral}). Table~\ref{tab:cleanbase} gives the signed shift in marked-option
probability against each baseline's own no-target prompt.

\begin{table}[!htb]\centering
\renewcommand{\arraystretch}{1.08}\footnotesize
\caption{\textbf{Deleting the baseline's risk clauses makes cross-attribute influence two-sided.}
Signed shift in marked-option probability on the non-target color readout, against the same
baseline's no-target prompt; $p_0$ is that prompt's own marked-option probability.
Deleted$-$stated is the paired contrast (clause-deleted minus risk-stating) with a $95\%$ bootstrap
interval over the $20$ texts.}
\label{tab:cleanbase}
\begin{tabular*}{\linewidth}{@{\extracolsep{\fill}}lccccc@{}}
\toprule
& \multicolumn{2}{c}{Risk-stating baseline} & \multicolumn{2}{c}{Clause-deleted baseline} & Deleted$-$stated \\
\cmidrule(lr){2-3}\cmidrule(lr){4-5}\cmidrule(lr){6-6}
Checkpoint & $p_0$ & cautious & $p_0$ & cautious & cautious \\
\midrule
\modelid{Qwen2.5-32B}          & $.050$ & $-.002$ & $.135$ & $-.099$ & $-.097$ $[-.159,-.042]$ \\
\modelid{Qwen3-32B}            & $.036$ & $-.000$ & $.247$ & $-.214$ & $-.213$ $[-.301,-.133]$ \\
\modelid{OLMo-2-32B}           & $.022$ & $-.009$ & $.117$ & $-.102$ & $-.093$ $[-.147,-.055]$ \\
\modelid{Gemma-2-27B}          & $.025$ & $-.000$ & $.480$ & $-.455$ & $-.455$ $[-.489,-.407]$ \\
\modelid{Mistral-Small-3.1-24B}& $.041$ & $-.009$ & $.311$ & $-.280$ & $-.272$ $[-.360,-.185]$ \\
\modelid{Qwen3.5-9B}           & $.087$ & $-.020$ & $.315$ & $-.259$ & $-.238$ $[-.264,-.211]$ \\
\bottomrule
\end{tabular*}
\end{table}

Under the risk-stating baseline the non-target readout starts at $p_0{=}.022$--$.087$ and the target
readout at $\le.018$, and the cautious prompt moves the non-target readout by at most $.021$. Under the
clause-deleted baseline the same readout starts at $.117$--$.480$ and the cautious prompt moves it
$-.099$ to $-.455$, with the paired contrast negative and its interval excluding zero on $6/6$. The
risk-seeking prompt's shift is large under the clause-deleted baseline ($+.364$ to $+.752$) and its
contrast is null on $4/6$ and positive on two, so the substitution does not trade one direction for the
other; it restores headroom that the risk-stating baseline removes. The target readout behaves the same
way (cautious contrast $-.003$ to $-.233$, negative on $6/6$).

\paragraph{Scope of the baseline result.}
The clause deletion was run on six open-weight checkpoints, not on the black-box panel, whose scores remain relative to the risk-stating profile; on those six, deleting the clauses leaves the risk-seeking contrast null on $4/6$ and positive on two, so the risk-stating profile limits the cautious readout rather than inflating the risk-seeking one. The structured-world declaration battery contains no such profile, and the open-weight
and dense audits use their own baselines (Appendix~\ref{app:prompts:neutral}).

\subsection{Trait-Conditioned Completion: Behavioral Evidence}\label{app:c2}
The semantic-control, supplied-correlation, and base-prediction experiments have different readouts
and are reported separately. Their conjunction motivates trait-conditioned completion; it does not identify a unique inference algorithm.
\subsubsection{Semantic Controls: Meaning, Not Wording}\label{app:c2:semantics}

\begin{figure}[!htb]\centering
\includegraphics[width=4.95in]{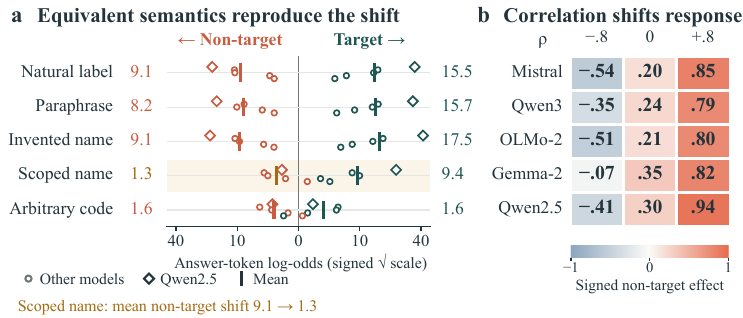}
\caption{Equivalent semantics reproduce the shift, and supplied correlations redirect it.
\textbf{(a)} Answer-token log-odds relative to the no-information reference with no stated preferences on non-target and target items under five wording conditions. Scoped name replaces the examples with a risk-only scope. Points show the five checkpoints, with diamonds for Qwen2.5, while ticks and numbers give their means. Non-target values are mirrored leftward on a signed-square-root axis.
\textbf{(b)} Signed effect of the bare prompt on the semantically positive non-target option at each demonstration correlation $\rho$ across $64$ shared worlds.}
\label{fig:why}
\end{figure}

\paragraph{Means for the semantic-control conditions.} Table~\ref{tab:noncefull}
lists the equal-weight five-checkpoint means behind Figure~\ref{fig:why}a: reference-adjusted answer-token
log-odds on the $94$-item style bank (leakage) and the $18$-item target risk bank of the same
open-weight analysis (on-construct);
Figure~\ref{fig:source-diagnosis}a shows each checkpoint. The invented name matches the natural label
on both banks. The arbitrary code is also weak on the on-construct bank ($1.64$ against $15.48$), so
it is a low-activation control rather than an activation-matched one; the scoped name, which replaces the examples with a risk-only scope (Appendix~\ref{app:prompts:nonce}), keeps an on-construct effect of $9.38$ while leaking $1.32$ and is the contrast that retains activation.

\begin{table}[!htb]\centering
\renewcommand{\arraystretch}{1.08}\small
\caption{\textbf{Semantic-control conditions: five-checkpoint mean log-odds.} Reference-adjusted
answer-token log-odds relative to the no-information reference; equal-weight means over the five primary
checkpoints.}
\label{tab:noncefull}
\begin{tabular*}{\linewidth}{@{\extracolsep{\fill}}lcc@{}}
\toprule
Condition & On-construct (18 items) & Leakage (94 items) \\
\midrule
Natural label & $15.48$ & $9.06$ \\
Operational paraphrase & $15.71$ & $8.17$ \\
Invented name & $17.50$ & $9.13$ \\
Scoped name & $9.38$ & $1.32$ \\
Arbitrary code & $1.64$ & $1.59$ \\
\bottomrule
\end{tabular*}
\end{table}

\paragraph{Invented-name example domains are disjoint from the audited families.} The invented trait's
defining examples (abstract outcome variance, investment choices, uncertain leisure activities) are
lexically and semantically disjoint from the $94$-item non-target bank of the open-weight analysis ($2/94$ items
share only a function word, and at most the $10$-item travel-planning family is arguably adjacent). On a
three-model black-box replication (\modelid{GPT-5.5}, \modelid{Kimi-K2.6}, \modelid{DeepSeek-V4-Flash};
the dense audit's $17$-item bank, $30$ requested draws per cell, role-instruction-only reference, and
item-clustered bootstrap with $B=2{,}000$; there the invented trait is defined by an operational
outcome-variance description rather than by examples), restricting the readout post hoc by domain tag
to the $9$ aesthetic/UI items (avatar, color scheme, font, icons, keyboard, layout density, lock screen,
motion, wallpaper) leaves
the leakage intact ($+1.02$, $95\%$ CI $[0.98,1.06]$, versus $+0.89$ on all $17$ items; the scoped
variant stays at $.17$--$.21$ on the same subsets), so example-domain overlap does not explain the leakage (\modelid{DeepSeek-V4-Flash} extends the
frozen two-model design). Its arbitrary-code condition is not a null control: that invented name is defined by an unrelated
listing rule (preferring options listed earlier and with shorter names), and it moves responses toward the
default pole ($-.80$, CI $[-.86,-.73]$, outside the prespecified null band $|\text{shift}|<.10$).
The two analyses report different scales: the black-box-replication values
are per-item mean reference-adjusted shifts on a $7$-point readout whose signed scale-point displacement is
rescaled to $[-1,1]$, so shifts range over $[-2,2]$, whereas the open-weight
figures in Table~\ref{tab:noncefull} and Figure~\ref{fig:why}a ($9.13$ versus $9.06$) are mean answer-token
log-odds shifts over the open-weight $94$-item bank.

\begin{figure}[htb]\centering
\includegraphics[width=\linewidth]{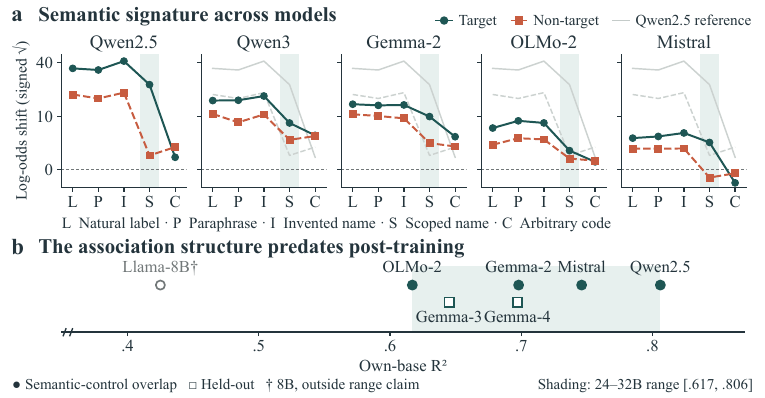}
\caption{\textbf{Semantic signatures and own-base prediction.}
\textbf{(a)} Target (circles, solid) and non-target (squares, dashed) log-odds shifts
for each checkpoint in Figure~\ref{fig:why}a. Pale curves repeat Qwen2.5;
the common signed-square-root scale retains original-unit tick labels and negative values.
L, natural label; P, paraphrase; I, invented name; S, scoped name (risk-only scope, no examples); C, arbitrary code.
Lines guide comparisons between categorical conditions; shading marks S.
\textbf{(b)} Own-base $R^2$ for seven base--instruct pairs. Filled circles overlap the semantic-control checkpoints;
open squares are the Gemma-3 and Gemma-4 pairs, and the hollow circle is the Llama-3.1-8B pair. Vertical separation distinguishes groups only.
The shaded 24--32B primary-panel range $[.617,.806]$ is not a confidence interval;
Llama-3.1-8B is outside that scale claim. The axis starts at $.35$.}
\label{fig:source-diagnosis}
\end{figure}

\subsubsection{Supplied Correlations Redirect the Completion}\label{app:battery:figs}

\paragraph{Correlation responses in the explicit-evidence bank.} Table~\ref{tab:rho-means} gives the
bare-prompt estimates and intervals behind Figure~\ref{fig:why}b. Sixteen demonstrations drawn from
training-split worlds, each showing both attributes and the ground-truth option text (never the
test's answer code), approximate $P(T{=}t,Z{=}z)=(1+\rho tz)/4$ with balanced marginals; the endpoint
settings realize correlations of $\pm.75$, and $\rho{=}0$ is a length-matched shuffled control (generator,
estimands, and counts in Appendix~\ref{app:battery}). The endpoint contrasts $.90$--$1.40$ in Figure~\ref{fig:why}b are the differences between the $\rho{=}{+.8}$ and $\rho{=}{-.8}$ columns.
Table~\ref{tab:rho-means} pools the four attribute pairs; direct-declaration contrasts are in Table~\ref{tab:explicit-evidence}.
\begin{table}[!htb]\centering
\renewcommand{\arraystretch}{1.08}\small
\caption{\textbf{Bare-prompt correlation responses in the explicit-evidence bank.} Signed effects on the
semantically positive non-target option, with $95\%$ world-bootstrap intervals ($64$ shared worlds,
$42$ non-target texts, $B{=}2{,}000$, both code maps). Nominal $\rho$ labels are retained;
the displayed endpoint demonstrations realize correlations of $\pm.75$.}
\label{tab:rho-means}
\setlength{\tabcolsep}{4pt}
\begin{tabular*}{\linewidth}{@{\extracolsep{\fill}}lccc@{}}
\toprule
Model & $\rho=-.8$ & $\rho=0$ & $\rho=+.8$ \\
\midrule
Mistral & $-.543$ $[-.672,-.403]$ & $.202$ $[.036,.366]$ & $.853$ $[.797,.907]$ \\
Qwen3-32B & $-.347$ $[-.477,-.213]$ & $.242$ $[.090,.382]$ & $.794$ $[.713,.873]$ \\
OLMo-2-32B & $-.511$ $[-.659,-.351]$ & $.208$ $[.062,.352]$ & $.796$ $[.718,.865]$ \\
Gemma-2-27B & $-.071$ $[-.215,.067]$ & $.354$ $[.211,.493]$ & $.824$ $[.723,.911]$ \\
Qwen2.5-32B & $-.413$ $[-.606,-.204]$ & $.298$ $[.146,.450]$ & $.941$ $[.886,.984]$ \\
\bottomrule
\end{tabular*}
\end{table}

\subsubsection{Matched Base Checkpoints Predict Instruct-Model Shifts}\label{app:c2:base}

\paragraph{Own-base prediction.} A base checkpoint is read as a raw language model, without a chat
template, on the $47$ seven-point-scale items of Section~\ref{sec:main}; its same-item
reference-adjusted shifts ($14$ trait profiles $\times$ $47$ items; the profiles are risk, cautious,
reckless, dominant, rebellious, humble, avant-garde, flamboyant, minimalist, frugal, punctual, detail-oriented,
and the invented name with and without its scope) predict those of the matched instruct
checkpoint. $R^2$ is the squared same-item correlation, tested against trait-permutation nulls. Matched
pairs exist for four 24--32B families---\modelid{Qwen2.5-32B}, \modelid{Mistral-Small-3.1-24B},
\modelid{OLMo-2-32B}, and Gemma, one family whose base is available in three generations
(\modelid{Gemma-2-27B}, \modelid{Gemma-3-27B}, \modelid{Gemma-4-31B})---and for one smaller pair,
\modelid{Llama-3.1-8B}. Own-base $R^2$ is $.806$, $.746$, and $.617$ for the first three and $.698$,
$.645$, and $.697$ along the Gemma chain (all $p{=}.002$; own-base residual $r{=}.291$/$.259$/$.344$
beyond a shared predictor averaged over the other bases of this set, \modelid{Llama-3.1-8B} and the other
two Gemma generations), so all six 24--32B pairs lie within $.617$--$.806$
(Figure~\ref{fig:source-diagnosis}b). \modelid{Llama-3.1-8B} gives own-base $R^2{=}.425$ ($n{=}658$),
below that range and outside the scale claim. This is same-item predictive association across the three Gemma generations, not evidence of a
stable rank-one representation or a transportable internal intervention. \modelid{Qwen3-32B} has no public same-size base, and \modelid{GPT-OSS-20B} releases none, so their own-base cells are
empty (checkpoints in Table~\ref{tab:owcheckpoints}; panel coverage in Appendix~\ref{app:suite}).

\paragraph{A base average from other families.} The mean shift of the \modelid{Qwen2.5-32B},
\modelid{Mistral}, and \modelid{OLMo-2} bases, excluding any base from the target's own family, predicts
all five primary checkpoints at $R^2{=}.658$--$.774$ ($5/5$). This includes \modelid{Qwen3-32B}, which
has no own-base pair: $.658$ from the Mistral and OLMo-2 bases, and $.715$ when the same-family Qwen2.5 base
is included.

\paragraph{Scope.}
The matched-base results above are same-item associations, not a randomized intervention on
post-training; they do not establish what fraction of completion is caused by pretraining or
instruction tuning. Missing own-base pairs are not imputed from other families.

\subsection{Internal Coupling: Erasure, Controls, and Limits}\label{app:geom}\label{app:c3}
These interventions concern the \emph{bare persona prompt}, the unspecified state in which completion is measured; the neutral declaration is tested behaviorally in Appendix~\ref{app:c4}. We give the subspace construction, fit/evaluation split, control checks, and readout and utility limits.

\subsubsection{Design and Measurement}
A prespecified program with ten predictions, priors, and pass rules (deviations are logged with dates, Appendix~\ref{app:repro}) asks what leakage looks like inside the residual stream of five primary checkpoints; this appendix reports the late-erasure prediction and its controls. Under a separate prespecification, its causal core is applied to four held-out checkpoints. Eighteen
fixed condition wordings (the bare prompt, two operational target wordings, the opposite-pole prompt
(\textsc{cautious}), a marked and a default persona, four factorized or scoped repairs, explicit
negation, three invented-name wordings, an arbitrary code, and the two unrelated control attributes, with the no-information reference
as the eighteenth) are crossed with the $47$-item seven-point bank and with pole-flipped
renderings of its $17$ items outside the target domain ($13$ coded non-target, $4$ ambiguous); states are
read at every decoder layer at six semantically aligned positions, and four base checkpoints are
collected identically. Every subspace and axis below is fit on one half of the items and evaluated on
the other half, in two folds; target items are split the same way. Each subspace is the top-$k$
singular basis in residual coordinates of the fit-half per-item condition-minus-reference displacements
read at the answer position, where what is decodable and what is causally responsible need not
coincide \citep{wiegreffe2025answer}. It is computed separately at every layer of the band, so the
erasure at a layer uses that layer's own basis; the rank-$1$ response-coordinate (scale) axis is instead
the unembedding direction of scale point\,$7$ minus scale point\,$1$ pulled back through the final normalization weight,
the same axis at every layer and for every condition, and so is not fit from any persona's displacement.

\paragraph{Centered question-to-answer erasure.}\label{app:geom:erasure}
Erasing a subspace in \emph{centered} form,
\begin{equation}\label{eq:erase}
\boldsymbol{h}\mapsto\boldsymbol{h}-\boldsymbol{U}\boldsymbol{U}^{\!\top}(\boldsymbol{h}-\boldsymbol{h}_{\mathrm{ref}}),
\end{equation} at every question-to-answer position
and at every layer of a depth band, cannot be undone by re-reading the persona prefix within the band;
$\boldsymbol{h}_{\mathrm{ref}}$ is the no-information reference run's state at the same item, layer, and position (the question-to-answer text is
identical across conditions). Leakage removed (causal explanatory power, CEP) is read on held-out non-target items under the target prompt,
retention (RET) on held-out target items, side effect on a punctual-persona run, and carry on the
opposite-pole run. In the band from relative depth $.65$ upward (Table~\ref{tab:geomerase}), replacing
the whole question-to-answer span with that of the no-information reference defines the ceiling ($1.00$), and erasing the persona's own
top-$4$ directions reaches it: $.97$--$1.03$ of the leakage removed on $5/5$ (positive control; pass rule
CEP $\ge.5$). That control shows only that this lever family can reach the leakage at all, since it
collapses retention with it. A rank-$4$ subspace
fit from the marked versus default \emph{personas}, which never mention risk, removes all of the
trait prompt's non-target-item leakage ($1.10$--$1.37$; values above one overshoot past the reference toward the default
pole), most of its target effect (retention $-.17$ to $.20$), and $.48$--$1.33$ of the
opposite-pole effect, on every model; the top-$4$ directions of the trait prompt's displacement on \emph{target} items
remove $.88$--$.97$ of the leakage in return. Each primary checkpoint receives 20 random rank-$4$ frames, drawn independently for each fold and layer. Separately, the held-out runs use five frames per checkpoint. The primary frames' two-fold averages have CEP $\le.01$ and retention $\ge.97$ (individual fold-by-frame cells reach $.0124$ and $.9685$), and even rank-$16$ frames remove at most $.03$
(retention $\ge.94$), the unrelated-control subspace removes at most $.18$ on the five primary checkpoints, and erasing the style-persona subspace moves the punctual-persona run by SIDE $.036$--$.069$. SIDE is the mean absolute change in signed scale shift $(\text{expected scale point}-4)/3$ over held-out non-target items and then folds; $.10$ is a $.3$-point move. Carry is the fraction of the opposite-pole run's non-target shift removed. At each rank and band, the control check requires own-subspace CEP $\ge.5$, median and 95th-percentile $|$CEP$|$ over 20 rank-matched random frames $\le.10$ and $.20$, self-subspace SIDE $\le.10$, and no numerical or tokenization errors. The complete $k{=}4$ control check passes on $4/5$: \modelid{Gemma-2}'s self-subspace SIDE is $.103$, just above the $.10$ rule, so its
raw markedness cell remains visible but is not counted among checkpoints that pass that configuration's control check. Markedness with
the $1$-D scale contrast projected out still removes the leakage ($1.09$--$1.33$), but the broader
rank-$4$ answer-serialization subspace is not inert: its raw CEP values span $.84$--$1.07$
across the five table rows, and it removes the leakage on all $4/4$ checkpoints that pass the control check. The upper
endpoint $1.07$ belongs to the excluded \modelid{Gemma-2} configuration. Markedness with that answer span projected out still
removes both effects, so it is not contained in the answer subspace; the stronger claim of independence
from answer commitment is nevertheless rejected (Table~\ref{tab:hardening}).

\begin{table}[!htb]\centering
\renewcommand{\arraystretch}{1.08}\footnotesize
\caption{\textbf{Centered-erasure point estimates (CEP/RET).} Same band $[.65,1]$, rank $4$ except the $1$-D scale axis,
held-out item halves. Values above one overshoot the reference; negative retention reverses the target.
$^\ddagger$Gemma-2 fails the full fixed-setting SIDE check but has passing adjacent settings.
$^\dagger$Gemma-4 fails the random-control check and is excluded from the interpretable held-out
count; its estimates remain visible. In each column block the upper rows are the five primary
checkpoints and the lower rows the held-out checkpoints; the second block holds controls and
complementary constructions.}
\label{tab:geomerase}
\setlength{\tabcolsep}{4pt}
\begin{tabular*}{\linewidth}{@{\extracolsep{\fill}}l c c c c c@{}}
\toprule
Model & full & own & markedness & target's own & \shortstack{target-\\orthogonal} \\
\midrule
\modelid{Qwen2.5-32B} & $1.00/.00$ & $.98/.07$ & $1.13/.20$ & $.88/\!-\!.00$ & $.86/.91$ \\
\modelid{Qwen3-32B} & $1.00/\!-\!.00$ & $.98/.16$ & $1.37/.03$ & $.97/\!-\!.01$ & $.63/.99$ \\
\modelid{Gemma-2-27B}$^\ddagger$ & $.99/\!-\!.00$ & $.98/.12$ & $1.35/.01$ & $.96/.01$ & $.56/.99$ \\
\modelid{OLMo-2-32B} & $1.00/\!-\!.00$ & $.97/.29$ & $1.10/\!-\!.17$ & $.91/\!-\!.02$ & $.58/.99$ \\
\modelid{Mistral} & $1.00/\!-\!.00$ & $1.03/.14$ & $1.20/\!-\!.11$ & $.88/.02$ & $.65/.99$ \\
\midrule
\modelid{Llama-3.1-8B-Instruct} & $1.00/\!-\!.00$ & $1.00/.19$ & $1.30/\!-\!.15$ & $.69/.00$ & $.62/.99$ \\
\modelid{GPT-OSS-20B} & $1.00/.01$ & $.97/.44$ & $.96/.24$ & $1.07/\!-\!.08$ & $.59/.97$ \\
\modelid{Gemma-3-27B-IT} & $1.00/\!-\!.00$ & $.99/.42$ & $1.60/\!-\!.09$ & $.95/.01$ & $.70/.86$ \\
\modelid{Gemma-4-31B-IT}$^\dagger$ & $1.01/.00$ & $1.00/.01$ & $1.09/.64$ & $.67/\!-\!.00$ & $1.02/.97$ \\
\midrule\midrule
Model & \shortstack{leakage-\\orthogonal target} & \shortstack{markedness\\$\perp$ scale axis} & \shortstack{scale axis\\$1$-D} & unrel.\ control & random \\
\midrule
\modelid{Qwen2.5-32B} & $\!-\!.02/.02$ & $1.10/.32$ & $.40/.80$ & $.18/.72$ & $.00/.99$ \\
\modelid{Qwen3-32B} & $\!-\!.00/\!-\!.03$ & $1.33/.05$ & $1.53/.21$ & $.01/.93$ & $.00/1.00$ \\
\modelid{Gemma-2-27B}$^\ddagger$ & $\!-\!.05/1.11$ & $1.24/.01$ & $.56/.99$ & $.02/1.00$ & $.00/1.00$ \\
\modelid{OLMo-2-32B} & $\!-\!.00/.73$ & $1.09/\!-\!.08$ & $1.63/.30$ & $.00/.81$ & $.00/1.00$ \\
\modelid{Mistral} & $\!-\!.02/1.00$ & $1.17/\!-\!.02$ & $.52/.57$ & $.09/.99$ & $.00/1.00$ \\
\midrule
\modelid{Llama-3.1-8B-Instruct} & $\!-\!.00/.54$ & $1.26/\!-\!.10$ & $1.12/\!-\!.09$ & $.03/.88$ & $.00/1.00$ \\
\modelid{GPT-OSS-20B} & $.03/.10$ & $.89/.29$ & $.56/.79$ & $.65/.45$ & $.01/1.00$ \\
\modelid{Gemma-3-27B-IT} & $.00/.72$ & $1.58/.04$ & $1.64/.07$ & $.00/.88$ & $.00/1.00$ \\
\modelid{Gemma-4-31B-IT}$^\dagger$ & $.00/\!-\!.00$ & $.62/.58$ & $1.30/\!-\!1.44$ & $.90/.20$ & $.85/\!-\!.49$ \\
\bottomrule
\end{tabular*}

\end{table}

\paragraph{Rank and depth stability.}\label{app:geom:channel}
The rank and depth conclusions do not depend on one selected cell. A fixed sweep over
$k\in\{1,2,4,6\}$ and bands $[.45,.65]$, $[.55,.75]$, $[.65,.85]$, $[.75,1]$, and $[.65,1]$
requires two consecutive passing ranks in the order $1,2,4,6$ at band $[.65,1]$, or two consecutive passing bands in the order $[.65,.85],[.75,1],[.65,1]$ at rank $4$. A cell must pass the control check, have own-subspace CEP lower 95\% bound $\ge.5$, and keep answer-token mass within $.15$ of the no-information run; the tested style-persona subspace additionally requires CEP lower bound $\ge.5$ and RET upper bound $<.7$, while target-orthogonal requires CEP lower bound $\ge.5$, RET lower bound $\ge.7$, and SIDE $\le.10$. Both the style-persona subspace and the constructed target-orthogonal subspace form such a region on $5/5$ models after the required $95\%$ item-paired bootstrap bounds are applied
($B{=}2{,}000$; Figure~\ref{fig:hardening}a). All non-random scored cells retain the seven-point answer
channel under the prespecified probability-mass check; for the fixed-setting markedness cell the largest
fold-by-block erased-minus-reference mass difference is $.00033$ against the $.15$ limit. The same
levers are weak or absent on $3/5$ models in the $.45$--$.65$ band (on \modelid{Qwen3} the question-to-answer span does not yet carry
the effect: ceiling $.02$). Thus the supported statement is a late rank/band region, fit per model, not an intrinsic rank of four or a shared cross-model vector.

\paragraph{Constructed selectivity.}\label{app:geom:handle}
At each layer, let $D_T$ and $D_U$ have as columns the fit-half persona-minus-reference
displacement vectors on target and non-target items, respectively. With
$U_T=\operatorname{SVD}_k(D_T)$ denoting its orthonormal leading left-singular basis, the construction is
\begin{equation}\label{eq:target-orth}
 D_{U\perp T}=(I-U_TU_T^{\!\top})D_U,
 \qquad U_{U\perp T}=\operatorname{SVD}_4(D_{U\perp T}).
\end{equation}
The projection uses the target \emph{displacement span}, not a direction learned from evaluation
items. Both bases are fitted separately within each model, layer, and item fold; $k=4$ is the primary target-span rank, and the sensitivity check removes the whole fit-half target span (rank $15$: each fold fits $15$ target items, so the requested $k=16$ returns $15$ directions). The selective point-estimate region is
CEP $\ge.5$ and RET $\ge.7$; the adjacent-region analysis additionally uses the declared interval and
control checks. Equation~(\ref{eq:target-orth}) is a compact statement of the fitted construction, not a new intervention.

The top-$4$ leakage-displacement directions are fit after projecting out the rank-$4$ target-displacement span,
using fit-half items only. This constructed subspace removes $.56$--$.86$ of leakage at
$.91$--$.99$ target retention on the primary five and $.59$--$.70$ at $.86$--$.99$ on the three
interpretable held-out checkpoints. Removing the whole fit-half target span gives $.47$--$.83$ at $.93$--$1.00$. At the fixed setting, no tested persona-induced subspace enters the selective region on the eight interpretable checkpoints; in the rank/band sweep one mid-band cell does (\modelid{OLMo-2}, the risk persona's own displacement, rank $1$, band $[.45,.65]$: CEP $.68$, RET $.71$). This is a finite family of fitted interventions, not a proof that all
persona representations are inseparable. A target-referenced construction is not itself a
persona-induced subspace.

\begin{table}[!htb]\centering
\renewcommand{\arraystretch}{1.08}
\caption{\textbf{Interval estimates and the answer-span control on the five primary checkpoints} (band
$[.65,1]$, $k=4$). Markedness, target-orthogonal, and answer entries are point estimates with $95\%$
item-paired bootstrap intervals ($B=2{,}000$); marked$\perp$answer is the compact point CEP/RET pair.
The answer column erases the rank-$4$ span of the seven scale-point unembedding rows; marked$\perp$answer
reports CEP/RET after removing that span from the markedness subspace. $\Delta\mu$ is the largest
absolute erased-minus-reference mean scale-point mass difference across folds and non-target/target blocks.
The columns are shown in two blocks. $^\dagger$The full control check fails because self-subspace SIDE is $.103$ against the $.10$ rule.}
\label{tab:hardening}
\footnotesize
\setlength{\tabcolsep}{4pt}
\begin{tabular*}{\linewidth}{@{\extracolsep{\fill}}l c c c c@{}}
\toprule
Model & marked CEP & marked RET & \shortstack{Target-orthogonal\\CEP} & \shortstack{Target-orthogonal\\RET} \\
\midrule
\modelid{Qwen2.5-32B} & \shortstack{$+1.13$\\$[+1.03,+1.23]$} & \shortstack{$+0.20$\\$[-0.02,+0.41]$} & \shortstack{$+0.86$\\$[+0.77,+0.93]$} & \shortstack{$+0.91$\\$[+0.87,+0.95]$} \\
\addlinespace[3pt]
\modelid{Qwen3-32B} & \shortstack{$+1.37$\\$[+1.28,+1.45]$} & \shortstack{$+0.03$\\$[-0.04,+0.09]$} & \shortstack{$+0.63$\\$[+0.59,+0.68]$} & \shortstack{$+0.99$\\$[+0.96,+1.02]$} \\
\addlinespace[3pt]
\modelid{Gemma-2-27B}$^\dagger$ & \shortstack{$+1.35$\\$[+1.31,+1.40]$} & \shortstack{$+0.01$\\$[-0.11,+0.12]$} & \shortstack{$+0.56$\\$[+0.52,+0.60]$} & \shortstack{$+0.99$\\$[+0.95,+1.01]$} \\
\addlinespace[3pt]
\modelid{OLMo-2-32B} & \shortstack{$+1.10$\\$[+1.07,+1.14]$} & \shortstack{$-0.17$\\$[-0.24,-0.09]$} & \shortstack{$+0.58$\\$[+0.54,+0.63]$} & \shortstack{$+0.99$\\$[+0.96,+1.01]$} \\
\addlinespace[3pt]
\modelid{Mistral} & \shortstack{$+1.20$\\$[+1.15,+1.24]$} & \shortstack{$-0.11$\\$[-0.16,-0.06]$} & \shortstack{$+0.65$\\$[+0.61,+0.69]$} & \shortstack{$+0.99$\\$[+0.98,+1.00]$} \\
\addlinespace[3pt]
\midrule\midrule
\multicolumn{5}{@{}l}{}\\[-1.6ex]
Model & answer CEP & marked$\perp$answer & $\max|\Delta\mu|$ & \\
\midrule
\modelid{Qwen2.5-32B} & \shortstack{$+0.93$\\$[+0.87,+0.98]$} & $+1.07/{+}0.43$ & $0.0000$ & \\
\addlinespace[3pt]
\modelid{Qwen3-32B} & \shortstack{$+0.99$\\$[+0.95,+1.03]$} & $+1.15/{+}0.11$ & $0.0000$ & \\
\addlinespace[3pt]
\modelid{Gemma-2-27B}$^\dagger$ & \shortstack{$+1.07$\\$[+1.04,+1.10]$} & $+1.25/{+}0.08$ & $0.0001$ & \\
\addlinespace[3pt]
\modelid{OLMo-2-32B} & \shortstack{$+0.84$\\$[+0.82,+0.87]$} & $+1.10/{+}0.03$ & $0.0000$ & \\
\addlinespace[3pt]
\modelid{Mistral} & \shortstack{$+1.00$\\$[+0.97,+1.03]$} & $+1.14/{-}0.01$ & $0.0003$ & \\
\addlinespace[3pt]
\bottomrule
\end{tabular*}
\end{table}

\subsubsection{What the Controls Do and Do Not Establish}\label{app:mech}

The second prespecification, dated 2026-09-02 before held-out runs, fixed the four held-out checkpoints and reused the bank, folds, conditions, and erasure protocol. It required own CEP $\ge.5$; joint removal (style-persona CEP $\ge.5$ and at least $.1$ above the random mean, style-persona RET $<.7$, target-displacement CEP $\ge.5$); scale-axis control (style-persona orthogonal to scale CEP $\ge.5$); and style-persona SIDE $\le.15$. The random-control check was added after the Gemma-4 run.
\paragraph{Which control count applies?} The point estimates of the second prespecification retain
eight checkpoints after excluding the random-control failure. The full rank-$4$ control check additionally excludes the \modelid{Gemma-2} configuration (above), not the model, which passes adjacent rank/band settings. The eight-checkpoint point estimates, the $4/5$ fixed-setting
full check, the $5/5$ interval-supported region, and the $3/4$ held-out replication therefore answer
different questions and are reported separately.

Table~\ref{tab:scopechecks} distinguishes the conclusions supported by each control check.

\begin{table}[!htb]\centering
\renewcommand{\arraystretch}{1.08}
\caption{\textbf{What each control count establishes.} Counts have different denominators and are
not pooled; the procedures, interval estimates, and failures behind each row are given in this appendix.}
\label{tab:scopechecks}
\footnotesize
\setlength{\tabcolsep}{4pt}
\renewcommand{\arraystretch}{0.95}
\begin{tabularx}{\linewidth}{@{}p{.30\linewidth} c Y@{}}
\toprule
Audit question & Result & Supported interpretation \\
\midrule
Fixed-setting rank-$4$ seven-point check & $4/5$ & Passes every control on four; \modelid{Gemma-2} misses only the side-effect rule ($.103$ against $.10$). \\
Adjacent rank/depth region & $5/5$ & Late dependence survives confidence-bound testing. \\
Held-out seven-point replication & $3/4$ & The fourth checkpoint is uninterpretable because its random control fails. \\
Binary check, transported subspaces & $1/5$ & Transport without refitting is not panel-wide; refitting
inside each readout defines no pass criterion and is format-dependent and model-heterogeneous
(Appendix~\ref{app:s4}), so no count applies to it. \\
Answer span: non-containment / independence & $4/4$ / $0/4$ & Markedness still removes the leakage with the answer span projected out, but that span alone also removes it. \\
All collateral-utility criteria & $1/5$ & Two more checkpoints miss one criterion each; selectivity on audited items is not broad utility preservation. \\
\bottomrule
\end{tabularx}
\end{table}
\paragraph{Held-out replication.} Four held-out checkpoints never used in the program---\modelid{Llama-3.1-8B-Instruct},
\modelid{GPT-OSS-20B}, \modelid{Gemma-3-27B-IT}, and \modelid{Gemma-4-31B-IT}, the local snapshots of
the checkpoint registry in Appendix~\ref{app:repro}---were collected and erased under the identical procedure (bands $.65$--$1$
and $.65$--$.85$, rank $4$, same folds). The positive control passes on all four (own top-$4$ CEP $.97$--$1.00$). On the three that also pass the random-control check, the markedness subspace removes leakage and target effect together (joint removal; $.96$--$1.60$ at retention $-.15$ to $.24$), the joint removal survives projection of the $1$-D scale contrast
($.89$--$1.58$), and the constructed target-orthogonal subspace is selective ($.59$--$.70$ at $.86$--$.99$;
held-out rows of Table~\ref{tab:geomerase}). On \modelid{Gemma-4-31B-IT} rank-$4$ erasure is destructive even with random or unrelated-control frames (random $.85/-.49$, unrelated control
$.90/.20$) and moves a punctual-persona run by $.28$--$.40$ against $.04$--$.07$ for the style-persona subspace on the primary checkpoints, so the
checkpoint fails the random-control check (a post hoc check; Appendix~\ref{app:repro}) and the prespecified SIDE limit ($.277$ against $.15$), and is excluded from the interpretable count while remaining in the table. \modelid{GPT-OSS-20B} is the one checkpoint on which the unrelated-control subspace also removes a large
share ($.65$ at $.45$ retention): its unrelated-control persona displacements overlap the directions whose erasure removes the leakage, a scope note on the
control rather than on the claim, since the markedness subspace still exceeds the random subspace by
$.95$.

\paragraph{Answer commitment and general utility.}\label{app:geom:readout}
Figure~\ref{fig:hardening}b--c reports the binary-transport and utility checks separately from the
seven-point result; the answer-span control is reported with Table~\ref{tab:hardening}. A separate $60$-item collateral-utility check has five prespecified criteria: accuracy retention $\ge.95$, absolute accuracy loss $\le.03$, order-consistency loss $\le.05$, unrelated-control SIDE $\le.10$, and answer-token KL at most the random-frame $95$th percentile plus $.01$. Its 60 two-option items contain no risk or style-marking words: 20 arithmetic, general-knowledge, and simple-logic questions and 20 format-following questions, each keyed with A and B equally often, plus 20 unrelated-control items contrasting punctual or detail-oriented options with relaxed ones. Each appears in both orders without a persona and under the target prompt with the fold-0-fitted subspace erased. Accuracy uses the 80 keyed presentations and order consistency the 60 items, both against the no-information run; unrelated-control SIDE is the absolute change in $2p-1$ relative to the unedited target-prompt run, averaged over the 20 unrelated-control items in both orders, where $p$ is the probability of the punctual or detail-oriented option. The KL limit is the 95th percentile under 20 random rank-4 frames plus $.01$. The constructed subspace meets all five on Mistral ($1/5$); Qwen2.5 misses only order consistency and OLMo-2 only unrelated-control SIDE (Figure~\ref{fig:hardening}c). The constructed subspace is selective on the audited items; it
meets every collateral criterion on one checkpoint and all but one on two more. The punctual half of the
unrelated-control block contrasts keeping a time buffer with cutting it close, which borders on risk; OLMo-2
misses the SIDE rule on both halves ($.20$ punctual, $.16$ detail-oriented), so its miss does not come from
that overlap.

\begin{figure}[htb]\centering
\includegraphics[width=\linewidth]{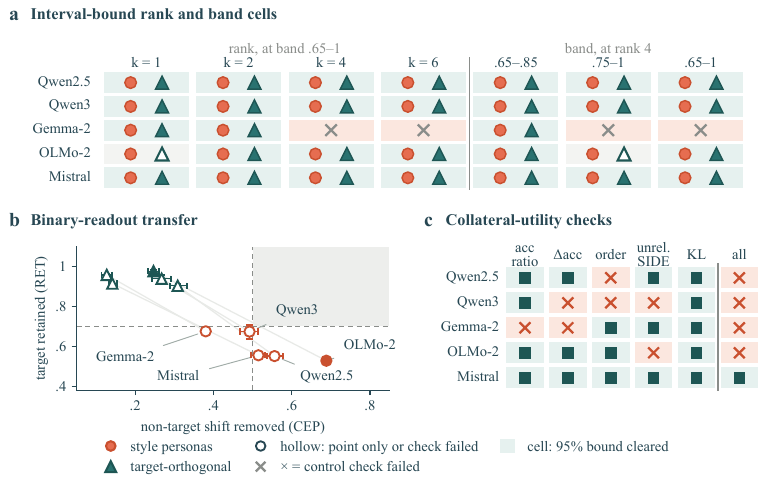}
\caption{\textbf{Prespecified robustness and boundary checks.}
\textbf{(a)} Interval-bound rank and band cells on the seven-point readout: circles denote the style-persona
subspace and triangles the constructed target-orthogonal subspace; filled symbols clear the required $95\%$
bound, hollow symbols pass only at the point estimate (in b, fail the control check), and crosses fail the control check; cell shading
marks where the subspace clears the bound.
\textbf{(b)} Binary-readout CEP--RET with item-paired intervals. Filled symbols identify the sole
model that passes the control check; hollow points remain descriptive. \textbf{(c)} Collateral-utility checks for the
constructed subspace; only \modelid{Mistral} passes all five.}
\label{fig:hardening}
\end{figure}

\subsubsection{Readout Boundary and Matched-Format Refitting}\label{app:s4}
Without refitting, the binary A/B and 1/2 runs are valid on all five primary checkpoints: full replacement reproduces the no-information reference at its ceiling, and the bare effect clears its floor. Subspaces fit on the seven-point readout and transported to these choices nevertheless pass the own-subspace positive control on OLMo-2 only (CEP $.57$ $[.55,.59]$; the other
four range from $.17$ to $.43$), below the prespecified $3/5$ rule (Figure~\ref{fig:hardening}b).
This is a failure of the transported intervention family, not evidence that binary-choice behavior
has no modifiable representation.

\paragraph{Refitting design.}
A separate exploratory analysis fits the same intervention inside each format. The panel is
Qwen3.5-9B, Qwen3-32B, Mistral-Small-3.1-24B, Gemma-2-27B, and OLMo-2-32B; Qwen3.5 replaces
Qwen2.5, so this is not the primary five-checkpoint panel. All use bf16, eager attention, and disabled
chat thinking. Of the $47$ seven-point items, the $13$ non-target and $30$ target items are scored; the four
ambiguous items are excluded. Fold 0 fits on $7$ non-target and $15$ target items and evaluates on the
complementary $6$ and $15$; fold 1 reverses the halves. The readouts are the seven-point scale (two
presentations) and binary choice (A/B and 1/2, both option orders). Presentations are averaged within
item before statistics, not counted as independent observations.

The table below uses the rank-$4$ constructed subspace, fitted within each format by the projection
above, over the question-to-answer span at every layer from relative depth $.65$ upward. Its center is the no-information reference
forward-pass state at the same scored item, layer, and position; it is not a train-mean center.
Conditions include the no-information reference, the target persona, punctual and detail-oriented unrelated-control personas, and
four directional non-target-preference conditions. No condition is a neutral declaration. The intervention requires
an item-matched no-information reference forward pass on the evaluation item, although the subspace fit uses only
training items.

\paragraph{Estimand and inference.}
Let $\pi_{p,i}^{0}$ and $\pi_{0,i}^{0}$ be the unedited persona and no-information reference answer distributions on
non-target item $i$, and let $\pi_{p,i}^{b}$ and $\pi_{0,i}^{b}$ be their edited counterparts. Every
$\pi$ is normalized inside that format's designated candidate answers. The reported distance reduction is
\begin{equation}\label{eq:dtv}
 \Delta\mathrm{TV}_b=\frac1{13}\sum_{i=1}^{13}
 \left[\operatorname{TV}(\pi_{p,i}^{0},\pi_{0,i}^{0})-
       \operatorname{TV}(\pi_{p,i}^{b},\pi_{0,i}^{b})\right].
\end{equation}
Target retention is the edited mean signed target effect divided by the unedited mean effect on the
$30$ target items at the \emph{same readout}; it is a ratio of means, not a mean of itemwise ratios.
This $\Delta\mathrm{TV}$ is an absolute distance reduction, not the normalized signed-shift fraction
CEP. Neither scale bounds the other, and the model panels differ as well. Pointwise intervals in the
refit analysis use $2{,}000$ item-bootstrap draws, with presentation averaged first, conditional on
the two fitted folds; they exclude refit uncertainty, family generalization, and multiplicity control.
The compact constructed-subspace table reports point estimates, not a test of differences between
methods or models. No pass rule is declared for refitting and no model is omitted on outcome.

The $.20$ statement in Section~\ref{sec:main} is the rounded maximum $.198$ in the binary column of
Table~\ref{tab:refitboundary}; it is a boundary on these tested constructed-subspace cells, not a general
upper bound on binary-readout interventions.

\begin{table}[!htb]\centering
\renewcommand{\arraystretch}{1.08}\footnotesize
\caption{\textbf{Readout-refit boundary for the constructed subspace.} Matched-format fitting and
evaluation, paired no-information reference, rank $4$, two item folds. Each entry is absolute leakage-TV
reduction / target-effect retention. The construction, fit split, reference, and estimand are specified above. Entries are point estimates;
this table has no pass/fail decision.}
\label{tab:refitboundary}
\begin{tabular*}{\linewidth}{@{\extracolsep{\fill}}lcc@{}}\toprule
Checkpoint & Seven-point & Binary choice \\
\midrule
Qwen3.5-9B & $.511/.945$ & $.068/.985$ \\
Qwen3-32B & $.243/.954$ & $.070/.997$ \\
Mistral-Small-3.1-24B & $.048/.955$ & $.001/1.000$ \\
Gemma-2-27B & $.004/.989$ & $.001/1.000$ \\
OLMo-2-32B & $.187/.977$ & $.198/1.000$ \\
\bottomrule\end{tabular*}
\end{table}

Paired-reference edits have zero reference drift by construction. Under this center, restoration
TV to the original reference is the post-intervention persona--reference TV; closing part of the gap is not
the same as returning to the original reference behavior. Random subspaces are near-null, which does not identify nuisance information.

\subsection{The Neutrality Gap: Design, Results, and Qualifications}\label{app:battery}\label{app:c4}
This section supplies the estimands, selection procedure, complete headline numbers, and controls for
Section~\ref{sec:declare}. Validation contrasts, dependence checks, and numerical sensitivity
calculations are included below. The held-out tail rule is fixed in the analysis plan; level splits, scale decompositions,
threshold sensitivity, and independent-item level tabulations are post hoc.
\subsubsection{Estimands, World Banks, and Prompt-Design Selection}\label{app:c4:design}

\paragraph{Estimands.} For a non-target query in world $w$ at declared level $z$, non-target sensitivity
$L_{\mathrm{sens}}$ compares the two target levels at that non-target value and query; its
normalized-choice version is the TV between the decoded two-option distributions
(Section~\ref{sec:framework}). Attribute-fidelity error $E_Z$ compares non-target responses with the task
ground truth, target efficacy is the signed probability effect on target queries, and retention is defined
against an adequate positive baseline effect; together these distinguish selective behavior, attribute
collapse, and loss of the target manipulation. The conservative per-world residual
$a_{w,z}=\min\{1,\,L_{\mathrm{sens}}(w,z)+u_{+}+u_{-}\}$ adds to normalized-choice TV the unassigned
answer mass $u_t$, the probability outside the legal answer codes at target level $t$. A world enters
the tail when any of its five non-target-level comparisons exceeds the sensitivity threshold, so the
conservative tail rate measures failure of the declared robust criterion, not a count of fully
observed semantic shifts, while the normalized-TV rate tests whether the pattern persists without the
widening. These estimands are related to, but are not renamed versions of, the field-audit measures of
Definition~\ref{def:surgical}: $L_U$ uses persona-minus-reference shifts, the surgicality ratio compares
regression coefficients (Appendix~\ref{app:sr}), and CEP/RET summarize erasure-induced removal and
retention under their own normalization, scales, and references.

\paragraph{World banks and prompt-design selection.} Structured worlds are drawn by one generator, with a
ground-truth answer, over four attribute pairs: risk weight $\to$ color preference, time preference $\to$
sound preference, social contact $\to$ layout preference, and time of day $\to$ color preference. The
explicit-evidence bank (Appendix~\ref{corrected-explicit-evidence-results}) samples all four. Candidate
prompt designs are selected on $162$ in-scope validation worlds and evaluated on $477$ in-scope held-out
worlds; both banks use two attribute pairs, risk weight $\to$ color preference and time preference
$\to$ sound preference (the \emph{certification pairs}, on which the selected prompt designs are evaluated), and
their $40$ non-target texts. Other declared pairs supply
separate collateral checks. The evaluation bank has $958$ scored test worlds: $477$ from the two certification pairs and $481$ from the other two pairs, used for collateral checks. Output validity uses all $958$; sensitivity uses the $477$ in-scope worlds. The separate development bank has $256$ worlds ($64$ per pair). The evaluation generator has $320$ worlds per pair; a salted hash split assigns $320$ to validation and $960$ to test, with two unscored dry-run prompts excluded. The evaluation records of the five primary checkpoints share one world set, so model and threshold comparisons stay paired. The non-target levels include an explicitly neutral level whose ground-truth answer is fixed by the
options' base constants rather than by any stated preference
(Appendix~\ref{identifiability-of-the-neutral-oracle-and-decomposition-of-the-mean-conditions}).
Validation selects strong relevance for OLMo-2 and Gemma-2, attribute scope for Qwen2.5, and strong relevance for the shared Qwen3/Mistral record (verbatim in Appendix~\ref{app:prompts:battery}). Each checkpoint has four candidates appended to the declaration-only prompt: attribute scope and relevance at its not-relevant (“The description above is not relevant to this question. Answer from what the question itself states”), weak, and strong levels. A candidate is feasible when validation mean sensitivity is at most $.05$, fidelity error at most $.10$, target retention at least $.90$, and collateral sensitivity at most $.05$. Candidates are ranked by feasibility on every checkpoint in their record, then by worst validation mean conservative sensitivity, then by fixed candidate order; for Qwen3 and Mistral, the larger of the two sensitivities is used. The least-infeasible candidate therefore minimizes that worst-case sensitivity; fidelity, retention, and collateral sensitivity enter only through feasibility. Every candidate fails this validation feasibility criterion, though selected candidates meet sensitivity and fidelity bounds at $z{=}{\pm1}$ (Table~\ref{tab:extreme-levels}); matched validation contrasts are descriptive because selection used those data.
The rule-lookup pilot is a successful rule reference with different tasks, not a matched
fourth condition here.

\subsubsection{Directional Declarations under Supplied Correlations}\label{corrected-explicit-evidence-results}

The explicit-evidence bank asks whether a declared non-target value overrides a correlation supplied in
context. It contains $64$ worlds per model, $16$ from each of the four attribute pairs---two of them
outside the certification pairs---with $42$ distinct non-target texts, and it was declared exploratory
before scoring, with no pass/fail decision. Each test prompt uses the battery frame without an added
instruction and $16$ in-context demonstrations drawn with balanced marginals and joint probabilities
$(1+\rho TZ)/4$: endpoint counts are 7/1/1/7 or reversed, realizing correlations of $.75$ and $-.75$,
and the zero-correlation table is 4/4/4/4; the $\rho$ labels remain the nominal settings $-.8$, $0$,
$+.8$. The non-target attribute appears in one of four information conditions: absent (bare), a direct
declaration at $z{=}{\pm1}$, a passing mention, or one recorded past choice
(Appendix~\ref{app:prompts:battery}); this bank has no neutral level. At each setting $D_p$ is the
signed non-target-option probability difference between target levels, and the endpoint contrast
$D_p(+.8)-D_p(-.8)$ has range $[-2,2]$; the margin contrast replaces probability by semantic log-odds,
and a retained fraction divides a condition's mean contrast by the bare condition's mean contrast. Means are formed
within and across worlds; code maps and non-target-value repetitions do not create independent worlds,
and the aggregation retains both code maps within each world (Appendix~\ref{app:repro}).

On the probability scale the direct declaration lowers the endpoint contrast from $.90$--$1.40$ to at
most $.07$, an attenuation of $94$--$100\%$; on the log-odds scale $5$--$31\%$ of the bare margin
contrast remains (Table~\ref{tab:explicit-evidence}). Gemma-2's small negative probability fraction is
reported as estimated, not clipped, and the model ordering is descriptive, not a multiple-comparison
test of adjacent pairs. The weak-evidence forms override less: a passing mention keeps $.03$--$.31$ of the bare
probability contrast and one past choice $.18$--$.44$, against at most $.06$ for the direct declaration. Text-clustered intervals for this bank are in
Appendix~\ref{text-level-units} and its log-odds decomposition in Appendix~\ref{bd-decomposition}.

\begin{table}[!htb]\centering
\renewcommand{\arraystretch}{1.08}\footnotesize
\caption{\textbf{Direct-declaration contrasts relative to the bare prompt} (explicit-evidence bank;
$64$ worlds per model over four attribute pairs; both code maps retained). Endpoint contrast
$D_p(+.8)-D_p(-.8)$ under the bare prompt and under the direct declaration; retained fractions on the
probability and margin scales with $95\%$ world-bootstrap intervals ($B{=}2{,}000$); median absolute
margin under the declaration. Exploratory.}
\label{tab:explicit-evidence}
\setlength{\tabcolsep}{4pt}
\begin{tabular*}{\linewidth}{@{\extracolsep{\fill}}lllll@{}}
\toprule
 & Endpoint contrast & Probability fraction & Margin fraction & Declared median \\
Model & bare / declared & {[}95\% interval{]} & {[}95\% interval{]} & absolute margin (nat) \\
\midrule
Mistral & 1.3963 / 0.0358 & 0.0256 {[}0.0109, 0.0443{]} & 0.0975 {[}0.0775, 0.1203{]} & 6.3125 \\
Qwen3 & 1.1406 / 0.0662 & 0.0581 {[}0.0235, 0.0954{]} & 0.3058 {[}0.2600, 0.3540{]} & 16.0000 \\
OLMo-2 & 1.3077 / 0.0459 & 0.0351 {[}0.0183, 0.0567{]} & 0.1965 {[}0.1688, 0.2264{]} & 9.3750 \\
Gemma-2 & 0.8951 / $-$0.0027 & $-$0.0031 {[}$-$0.0335, 0.0296{]} & 0.0703 {[}0.0421, 0.0979{]} & 10.2500 \\
Qwen2.5 & 1.3537 / 0.0075 & 0.0056 {[}$-$0.0000, 0.0182{]} & 0.0474 {[}0.0389, 0.0575{]} & 33.3750 \\
\bottomrule
\end{tabular*}
\end{table}

\subsubsection{Matched Validation Comparisons}\label{matched-validation-comparisons}

Each table uses the same $162$ in-scope validation worlds per model. ``Declaration'' is the
declaration-only condition and ``Selected'' that model's validation-selected candidate. Values use
normalized-choice TV and the deterministic ground-truth reference; bare prompts have no individual $Z$
fidelity target. The selected-minus-declaration interval is world-paired with $2{,}000$ bootstrap
replicates and is an exploratory, post-selection comparison. A declaration alone lowers normalized
sensitivity from $.60$--$.78$ to $.10$--$.16$, and the selected candidate changes it by $-.014$ to
$+.001$ further, except on Qwen2.5 ($-.076$ $[-.093,-.060]$).

\paragraph{All non-target-value levels.}\label{all-non-target-value-levels}

Table~\ref{tab:all-levels} reports matched validation results across all non-target levels.

\begin{table}[!htb]\centering
\renewcommand{\arraystretch}{1.08}\footnotesize
\caption{\textbf{Matched validation results across all non-target-value levels.} Relative to the declaration alone, the selected candidate changes sensitivity by at most $.015$ except on Qwen2.5 ($-.076$); both remain far below the bare prompt ($.60$--$.78$).}
\label{tab:all-levels}
\setlength{\tabcolsep}{3pt}
\begin{tabular*}{\linewidth}{@{\extracolsep{\fill}}lrrrlrr@{}}
\toprule
 & Bare & Declaration & Selected & Selected $-$ declaration & Selected & Selected \\
Model & sensitivity & sensitivity & sensitivity & [95\% interval] & fidelity error & target effect \\
\midrule
Mistral & 0.7196 & 0.1033 & 0.0891 & $-$0.0143 {[}$-$0.0225, $-$0.0063{]} & 0.1087 & 0.9830 \\
Qwen3 & 0.5968 & 0.1282 & 0.1205 & $-$0.0077 {[}$-$0.0124, $-$0.0035{]} & 0.1196 & 0.9042 \\
OLMo-2 & 0.6987 & 0.1131 & 0.1106 & $-$0.0025 {[}$-$0.0071, 0.0021{]} & 0.1141 & 0.9903 \\
Gemma-2 & 0.7750 & 0.1647 & 0.1653 & 0.0006 {[}$-$0.0097, 0.0112{]} & 0.1190 & 0.9520 \\
Qwen2.5 & 0.7161 & 0.1436 & 0.0674 & $-$0.0762 {[}$-$0.0926, $-$0.0597{]} & 0.1088 & 0.9971 \\
\bottomrule
\end{tabular*}
\end{table}

The all-level fidelity column includes the neutral stratum, whose ground-truth reference is not
identifiable from the prompt
(Appendix~\ref{identifiability-of-the-neutral-oracle-and-decomposition-of-the-mean-conditions}); the
extreme-level ($z{=}{\pm}1$) values in Appendix~\ref{extreme-non-target-values-only} are the interpretable fidelity
comparison. Because that reference does not depend on the prompt, the expected fidelity error at $z=0$ is $.5$
for any design, so the all-level fidelity condition sits near its $.10$ bound by construction; the selected
candidates fail feasibility on sensitivity ($.067$--$.165$) regardless.

\paragraph{Extreme non-target values only.}\label{extreme-non-target-values-only}

Non-target comparisons are restricted to $Z=-1$ or $Z=+1$. Target comparisons use available target rows
under the same prompt design, already with endpoint non-target values; the unchanged target column does not
imply an unmeasured five-level target grid. Table~\ref{tab:extreme-levels} reports the matched
validation comparison restricted to the two extreme non-target values.

\begin{table}[!htb]\centering
\renewcommand{\arraystretch}{1.08}\footnotesize
\caption{\textbf{Matched validation results at the extreme non-target values.} Both conditions are nearly insensitive here (sensitivity at most $.042$, fidelity error at most $.021$), so the residual sensitivity in Table~\ref{tab:all-levels} comes from the other levels.}
\label{tab:extreme-levels}
\setlength{\tabcolsep}{3pt}
\begin{tabular*}{\linewidth}{@{\extracolsep{\fill}}lrrlrrr@{}}
\toprule
 & Declaration & Selected & Selected $-$ declaration & Declaration & Selected & Selected \\
Model & sensitivity & sensitivity & [95\% interval] & fidelity error & fidelity error & target effect \\
\midrule
Mistral & 0.0043 & 0.0018 & $-$0.0026 {[}$-$0.0044, $-$0.0012{]} & 0.0045 & 0.0030 & 0.9830 \\
Qwen3 & 0.0134 & 0.0189 & 0.0055 {[}$-$0.0002, 0.0131{]} & 0.0085 & 0.0127 & 0.9042 \\
OLMo-2 & 0.0010 & 0.0042 & 0.0033 {[}0.0003, 0.0074{]} & 0.0092 & 0.0076 & 0.9903 \\
Gemma-2 & 0.0018 & 0.0055 & 0.0038 {[}$-$0.0003, 0.0097{]} & 0.0010 & 0.0029 & 0.9520 \\
Qwen2.5 & 0.0420 & 0.0062 & $-$0.0359 {[}$-$0.0553, $-$0.0175{]} & 0.0210 & 0.0031 & 0.9971 \\
\bottomrule
\end{tabular*}
\end{table}

\subsubsection{The Held-Out Tail Sits at the Neutral Level}\label{neutrality-and-the-held-out-failure-tail}

Each checkpoint's selected candidate, the declaration alone, and the bare prompt are scored once on the
$477$ held-out worlds, and the held-out split was opened once for each set of models evaluated. The sensitivity threshold ($.15$) and the tolerance ($.05$) were fixed before held-out evaluation. All five checkpoints violate the declared tail requirement under both scores (Table~\ref{tab:tail}), and the verdict holds at any threshold below $.95$ (Table~\ref{tab:tau-sweep}); the level split, threshold sweep, multiplicity
reallocations, and text-cluster re-expression below are post hoc.

\begin{table}[!htb]\centering
\renewcommand{\arraystretch}{1.08}\small
\caption{\textbf{Held-out prompt-design evaluation.} Sensitive-world share: the share of the $477$ worlds with at least one
comparison above the sensitivity threshold, under the conservative variant
and normalized-choice TV (sensitivity threshold $.15$, tolerance $.05$); brackets give $95\%$ intervals over
$40$ non-target texts. The first four models use the relevance instruction (strong); Qwen2.5 uses the attribute-scope instruction.
$z{=}0$ texts: non-target texts whose mean normalized TV at the neutral level exceeds $.15$. $L_{\mathrm{sens}}$
includes / excludes $z{=}0$. The level columns are post hoc diagnostics
(Appendix~\ref{neutrality-and-the-held-out-failure-tail}).}
\label{tab:tail}
\setlength{\tabcolsep}{4pt}
\begin{tabular*}{\linewidth}{@{\extracolsep{\fill}}lcccc@{}}
\toprule
Model & \shortstack{Conservative\\share} & \shortstack{Normalized-TV\\share [text CI]} & \shortstack{$z{=}0$ texts\\ / $40$} & \shortstack{$L_{\mathrm{sens}}$\\all / no $z{=}0$} \\
\midrule
\modelid{Mistral} & $.90$ & $.61$ $[.48,.74]$ & $27$ & $.081$ / $.014$ \\
\modelid{Qwen3} & $.64$ & $.58$ $[.46,.71]$ & $23$ & $.107$ / $.013$ \\
\modelid{OLMo-2} & $.62$ & $.62$ $[.50,.75]$ & $30$ & $.104$ / $.013$ \\
\modelid{Gemma-2} & $.76$ & $.76$ $[.65,.87]$ & $34$ & $.155$ / $.035$ \\
\modelid{Qwen2.5} & $.39$ & $.39$ $[.25,.52]$ & $20$ & $.061$ / $.009$ \\
\bottomrule
\end{tabular*}
\end{table}

\paragraph{Level split of the tail.} Table~\ref{tab:heldout-tail-detail} splits the union failures (worlds that fail at any non-target level) by
non-target level. Without $z{=}0$ the conservative union falls to $13$--$77$ of $477$ worlds
($3$--$16\%$), and at the two extreme levels to $0$--$46$ ($0$--$10\%$), whereas $180$--$431$ worlds
fail at the neutral level.

\begin{table}[!htb]\centering
\renewcommand{\arraystretch}{1.08}\small
\caption{Conservative and normalized-TV held-out failure-tail rates (every column uses the
conservative score except the normalized-TV union).}
\label{tab:heldout-tail-detail}
\setlength{\tabcolsep}{3pt}
\begin{tabular*}{\linewidth}{@{\extracolsep{\fill}}lrrrrr@{}}
\toprule
 & Conservative & Normalized-TV & Conservative & Nonzero-z & Extreme $\pm1$ \\
Model & union /477 & union /477 & $z=0$ count & union & union \\
\midrule
Mistral & 431 (90.4\%) & 293 (61.4\%) & 431 & 45 & 0 \\
Qwen3 & 303 (63.5\%) & 279 (58.5\%) & 281 & 74 & 46 \\
OLMo-2 & 296 (62.1\%) & 296 (62.1\%) & 288 & 43 & 8 \\
Gemma-2 & 364 (76.3\%) & 364 (76.3\%) & 364 & 77 & 8 \\
Qwen2.5 & 184 (38.6\%) & 184 (38.6\%) & 180 & 13 & 6 \\
\bottomrule
\end{tabular*}
\end{table}

Table~\ref{tab:heldout-levels} gives the normalized-TV counts at each level, the shares plotted in
Figure~\ref{fig:neutral}a.

\begin{table}[!htb]\centering
\renewcommand{\arraystretch}{1.08}\small
\caption{Held-out worlds whose normalized-choice TV exceeds $.15$ at each non-target level (selected
candidate; counts of $477$ worlds). A world can exceed the threshold at several levels, so the union is
not the row sum; it equals the normalized-TV union of Table~\ref{tab:heldout-tail-detail}.}
\label{tab:heldout-levels}
\setlength{\tabcolsep}{4pt}
\begin{tabular*}{\linewidth}{@{\extracolsep{\fill}}lrrrrrr@{}}
\toprule
Model & $z{=}{-1}$ & $z{=}{-.5}$ & $z{=}0$ & $z{=}{+.5}$ & $z{=}{+1}$ & Union \\
\midrule
Mistral & 0 & 8 & 285 & 35 & 0 & 293 \\
Qwen3 & 0 & 0 & 248 & 18 & 17 & 279 \\
OLMo-2 & 0 & 10 & 288 & 25 & 8 & 296 \\
Gemma-2 & 8 & 13 & 364 & 64 & 0 & 364 \\
Qwen2.5 & 0 & 0 & 180 & 13 & 6 & 184 \\
\bottomrule
\end{tabular*}
\end{table}

\paragraph{Selected added instructions versus the declaration alone.} The added instructions leave most neutral-level exceedances in place. Worlds exceeding the sensitivity threshold at $z{=}0$ number $285$ for Mistral's selected
candidate against $294$ for the declaration alone, $248$ against $287$ for Qwen3, $288$ against $289$
for OLMo-2, $364$ against $344$ for Gemma-2, and $180$ against $213$ for Qwen2.5. Qwen2.5's
attribute-scope instruction acts mainly at directional values: $13$ against $127$ worlds at
$z{=}{+}.5$ and $6$ against $35$ at $z{=}{+}1$.

\paragraph{Robustness to the sensitivity threshold.}\label{cell-threshold-sensitivity}

Table~\ref{tab:tau-sweep} scores the selected candidates' cached cells, post hoc, at sensitivity thresholds from
$.05$ to $.50$ with the tolerance held at $.05$; at $.15$ it reproduces every count of
Table~\ref{tab:heldout-tail-detail}. The verdict does not depend on the sensitivity threshold: the
smallest threshold at which a model's union tail rate would fall to $.05$ is $.953$--$1.000$ under
normalized TV and $.998$--$1.000$ under the conservative score, so every model violates the tail at
any threshold below $.95$, and a certified pass, which needs a Clopper--Pearson bound below the
tolerance, would need a threshold at least as high; at $.50$, $28$--$64\%$ of worlds still fail. On
normalized TV neither does its location: at every tested threshold the neutral level alone flags at
least $4.4$ times as many worlds as the four directional levels together. Under the conservative
score the unassigned-mass widening compresses that ratio at low thresholds, to at least $2.9$ from
$.10$ upward and to parity for \modelid{Mistral} at $.05$, where every world is flagged at the
neutral level and at some directional level. The directional levels are not free of exceedances either: under normalized TV, the tolerance
without $z{=}0$ would be met from $.06$ on \modelid{Qwen2.5} but only from $.32$--$.94$ on the other
four, while the two extreme levels stay at or below $.04$ at every tested threshold. The
declaration-only condition gives the same union verdict ($\tau^{*}\ge.98$ under both scores).

\begin{table}[!htb]\centering
\renewcommand{\arraystretch}{1.08}\footnotesize
\caption{\textbf{Sensitivity-threshold robustness of the held-out tail} (selected candidates, $477$ worlds,
normalized-choice TV). Each cell gives the share of worlds whose neutral-level comparison exceeds the
threshold / the same at any of the four directional levels; the neutral
shares at $.15$ match Figure~\ref{fig:neutral}. $\tau^{*}$: the smallest threshold at which the union-tail
point rate over all five levels reaches the $.05$ tolerance (normalized TV / conservative
score);
$\tau^{*}_{\mathrm{dir}}$: the same without $z{=}0$ (normalized TV).}
\label{tab:tau-sweep}
\setlength{\tabcolsep}{3pt}
\begin{tabular*}{\linewidth}{@{\extracolsep{\fill}}lcccccccc@{}}
\toprule
 & \multicolumn{6}{c}{Sensitivity threshold: neutral / directional share} & & \\
\cmidrule(lr){2-7}
Model & $.05$ & $.10$ & $.15$ & $.20$ & $.30$ & $.50$ & $\tau^{*}$ & $\tau^{*}_{\mathrm{dir}}$ \\
\midrule
\modelid{Mistral} & $.77$/$.14$ & $.66$/$.08$ & $.60$/$.07$ & $.50$/$.05$ & $.44$/$.05$ & $.30$/$.04$ & $.953$/$1.000$ & $.317$ \\
\modelid{Qwen3} & $.64$/$.06$ & $.55$/$.06$ & $.52$/$.06$ & $.52$/$.06$ & $.47$/$.05$ & $.47$/$.05$ & $1.000$/$1.000$ & $.560$ \\
\modelid{OLMo-2} & $.78$/$.13$ & $.66$/$.11$ & $.60$/$.09$ & $.59$/$.09$ & $.50$/$.07$ & $.43$/$.03$ & $1.000$/$1.000$ & $.385$ \\
\modelid{Gemma-2} & $.81$/$.18$ & $.79$/$.16$ & $.76$/$.16$ & $.74$/$.14$ & $.73$/$.14$ & $.60$/$.13$ & $1.000$/$1.000$ & $.940$ \\
\modelid{Qwen2.5} & $.45$/$.05$ & $.42$/$.03$ & $.38$/$.03$ & $.34$/$.03$ & $.29$/$.03$ & $.27$/$.02$ & $.998$/$.998$ & $.060$ \\
\bottomrule
\end{tabular*}
\end{table}

\paragraph{Post hoc multiplicity reallocations.}\label{additional-multiplicity-checks}

The analysis plan's \emph{certificate}, its finite-sample test of the declared conditions
(distinct from the bound of Proposition~\ref{prop:cert}), uses empirical-Bernstein bounds for bounded
world-level means and exact
Clopper--Pearson bounds for world-level exceedance indicators
\citep{maurer2009empirical, clopper1934confidence}. The sensitivity variable is bounded in
$[0,1]$. Each certificate tests eight conditions: mean conservative sensitivity; the maximum conservative cell residual over $2{,}385$ bank cells (a maximum without a confidence bound); the share of worlds with any cell above the sensitivity threshold; fidelity error; target effect; retention relative to the bare prompt; output validity; and collateral sensitivity. For mean, maximum, tail, fidelity, target effect, retention, validity, and collateral, the screening thresholds are respectively $\le.05$, $\le.15$, $\le.05$ (cell threshold $.15$), $\le.10$, $\ge.10$, $\ge.90$, $\ge.98$, and $\le.05$; strict thresholds are $\le.015$, $\le.05$, $\le.02$ (cell threshold $.05$), $\le.05$, $\ge.10$, $\ge.90$, $\ge.99$, and $\le.02$. All ten declared outcomes are violated, in every case through the maximum and
tail conditions. The declared per-condition $\alpha$ is $.0125$ in the record that covers Qwen3 and Mistral
jointly and $.025$ in each single-model record. Two post hoc checks tighten that allocation: the first
divides each declared $\alpha$ by eight, and the second uses $.05/(10\times 8)=.000625$ across all ten
model-by-threshold-set assessments (five models, each under a research-screening and a strict-application
threshold set) and eight conditions. All conservative tail lower bounds remain above $.31$ under the panel-wide calculation, well above the declared $.05$ tolerance. The reallocations turn some mean conditions from violated to inconclusive (Qwen2.5's
strict-application one under either reallocation, OLMo-2's screening one under the panel-wide allocation,
lower bound $.048$; Qwen2.5's screening one is inconclusive as declared), and every outcome stays violated
through the maximum and tail conditions, so the per-model multiplicity table is not reproduced.

\paragraph{Text-level units and clustered intervals.}\label{text-level-units}

Every world bank is drawn by the same generator from a fixed catalog of $20$ question stems per
construct, each with one option pair. Each block of $20$ worlds uses every stem once; later blocks
revisit the catalog in a fresh permutation with independent hidden ground-truth parameters (base constants
and non-target loadings). The rendered prompt depends only on text, presentation order, and declared
levels, so sensitivity is the same (up to batch-padding rounding) for worlds sharing text and order;
only the fidelity label differs (hidden base constant). In the two certification pairs each non-target
text has $13$--$18$ worlds in the evaluation bank of its pair, $8$--$16$ among the $477$ in-scope
held-out worlds, and $1$--$8$ validation worlds, from which candidates were selected; $39$ of the $40$
held-out texts also occur in validation (one printer-choice text has no validation world).
``Held-out'' therefore means a hash-partitioned set of parameter draws and worlds within a shared
catalog, not unseen texts, which the independent items supply
(Appendix~\ref{app:c4:independent}); the $64$ explicit-evidence worlds carry $42$ non-target texts,
and the neutral-wording confirmation scores that catalog under different input conditions rather
than new tasks. The certificate declares the world as cluster unit and computes empirical-Bernstein and
Clopper--Pearson bounds under independent-world assumptions from this catalog. Because blocks are permuted and hash-partitioned, coverage conditions are unestablished, so the nominal guarantees and world-bootstrap intervals are reported as issued, not as verified claims; the text-cluster analysis below replaces world-level independence with clustering by text.
Table~\ref{tab:textlevel} re-expresses the selected candidates' neutral-level results with the text as
cluster unit (text-cluster bootstrap, $B{=}2{,}000$), summarizing variation across the $40$ catalog texts; unseen texts are tested in Appendix~\ref{app:c4:independent}. Means and point estimates are unchanged, every declared
\emph{violated} outcome survives, and every tail rate stays above $.05$; the declared outcomes remain as
issued. For the declaration-only condition, neutral-level exceedance counts are $28$, $25$, $32$, $30$, and
$24$ texts in model order. For the selected candidates, same-sign counts sum to $120$ of $122$
model--text comparisons (world level $1{,}237$ of $1{,}240$).

\begin{table}[!htb]\centering
\renewcommand{\arraystretch}{1.08}\footnotesize
\caption{\textbf{Held-out neutral-level results with the question text as cluster unit} (selected
candidate, $40$ non-target texts, $477$ worlds). Exceedance counts at $z{=}0$: worlds above $.15$, and texts whose mean over their worlds is above $.15$; rates with $95\%$ text-cluster bootstrap intervals; parentheses give a post hoc one-sided
lower bound at level $.05/8$.}
\label{tab:textlevel}
\setlength{\tabcolsep}{3pt}
\begin{tabularx}{\linewidth}{@{}l >{\centering\arraybackslash}X >{\centering\arraybackslash}X >{\centering\arraybackslash}X >{\centering\arraybackslash}X >{\centering\arraybackslash}X >{\centering\arraybackslash}X@{}}
\toprule
Model & $z{=}0$ exceed.\newline worlds / texts & Tail, normalized TV\newline rate [CI] (lower) & Tail, conservative\newline rate [CI] & $L_{\mathrm{sens}}$ all\newline [CI] & $L_{\mathrm{sens}}$ no $z{=}0$\newline [CI] & Fidelity error,\newline directional [CI] \\
\midrule
Mistral & $285$ / $27$ & $.614$ $[.483,.739]$ ($.448$) & $.904$ $[.826,.967]$ & $.081$ $[.058,.105]$ & $.014$ $[.005,.026]$ & $.011$ $[.005,.020]$ \\
Qwen3   & $248$ / $23$ & $.585$ $[.456,.709]$ ($.423$) & $.635$ $[.509,.755]$ & $.107$ $[.079,.134]$ & $.013$ $[.001,.029]$ & $.016$ $[.000,.038]$ \\
OLMo-2    & $288$ / $30$ & $.621$ $[.497,.745]$ ($.454$) & $.621$ $[.498,.747]$ & $.104$ $[.084,.125]$ & $.013$ $[.005,.023]$ & $.014$ $[.003,.033]$ \\
Gemma-2 & $364$ / $34$ & $.763$ $[.655,.866]$ ($.618$) & $.763$ $[.660,.865]$ & $.155$ $[.125,.185]$ & $.035$ $[.015,.058]$ & $.018$ $[.008,.029]$ \\
Qwen2.5 & $180$ / $20$ & $.386$ $[.248,.515]$ ($.216$) & $.386$ $[.252,.519]$ & $.061$ $[.036,.089]$ & $.009$ $[.000,.027]$ & $.005$ $[.000,.013]$ \\
\bottomrule
\end{tabularx}
\end{table}

For the explicit-evidence bank, resampling its $42$ non-target texts leaves the direct-declaration
retained fractions of the margin contrast at $.097$ $[.077,.122]$, $.306$ $[.255,.362]$, $.196$
$[.166,.226]$, $.070$ $[.037,.101]$, and $.047$ $[.039,.057]$ (Mistral, Qwen3, OLMo-2, Gemma-2,
Qwen2.5; world-bootstrap intervals in Table~\ref{tab:explicit-evidence}), and the retained
probability-scale fractions at $.026$ $[.011,.045]$, $.058$ $[.022,.099]$, $.035$ $[.019,.056]$,
$-.003$ $[-.036,.031]$, and $.006$ $[-.000,.019]$.

\subsubsection{The Level Pattern on the Independent Items}\label{app:c4:independent}

The held-out worlds are separate parameter draws over the same $40$ non-target texts
(Appendix~\ref{text-level-units}), so they cannot show that the level pattern carries to unseen texts.
The independent items can: $200$ texts ($160$ non-target and $40$ target) over four attribute pairs
(risk weight $\to$ color, time preference $\to$ sound, risk weight $\to$ hue, time preference $\to$
layout), written by $21$ generator families without access to model outputs
(Appendix~\ref{app:suite}). The forced-choice condition of the report-versus-choice comparison
(Appendix~\ref{app:report-choice}) scores every non-target item under all four non-target states---absent,
$z{=}{+}1$, $z{=}{-}1$, and the neutral declaration (option-specific wording)---so the level split of Section~\ref{sec:declare} can
be computed on texts no prompt design was tuned on. This is a post hoc tabulation of cached scores, not a
certificate, and its design differs from the held-out evaluation in ways that travel with every number
below: the frame is the open frame (Appendix~\ref{app:prompts:tasks}), not the battery frame
with a selected added instruction; the neutral wording is the option-specific indifference sentence
(W3 of Appendix~\ref{neutral-wording-confirmation}), not W0; there are no $\pm.5$ levels and no
conservative unassigned-mass widening; the unit is the item, with generator-family cluster intervals
($21$ families, $B{=}2{,}000$); and the checkpoints are \modelid{Qwen3-32B}, \modelid{OLMo-2-32B},
\modelid{Gemma-2-27B}, \modelid{Mistral-Small-3.1-24B}, and \modelid{Qwen3.5-9B}, so
\modelid{Qwen2.5-32B} is not in this panel. Per item and state, normalized-choice TV between the two
target levels is computed within each order-by-map comparison and averaged over the four comparisons.

Table~\ref{tab:s2levels} reproduces the held-out level pattern: directional declarations leave at most $1$ of $320$ item--pole comparisons above $.15$, while the neutral declaration leaves
$51$--$81\%$ of items. The log-odds decomposition of Section~\ref{sec:declare}
(Appendix~\ref{bd-decomposition}) is restricted here, post hoc, to items whose bare margin $|\eta_w|$ is
at least $.5$ nat, so that margin ratios are defined ($132$, $83$, $155$, $47$, and $154$ of $160$ items
for \modelid{Mistral}, \modelid{OLMo-2}, \modelid{Qwen3}, \modelid{Qwen3.5-9B}, and \modelid{Gemma-2}; the
held-out decomposition uses no such filter). On those items directional declarations raise the shared
tendency $|\alpha_w|$ by a median $5.5$--$21.5$ nat and keep a median $.02$--$.10$ of the bare margin
($.29$ at \modelid{Qwen3.5-9B}'s negative pole), while the neutral declaration moves $|\alpha_w|$ by
$-0.2$ to $+1.2$ nat and keeps $.30$--$.73$ of it. Over all items the medians rise (\modelid{OLMo-2}'s
neutral ratio to $1.13$, \modelid{Qwen3.5-9B}'s negative pole to $.70$). Where both the neutral-level
and the bare shift exceed $.15$, the two share a sign in $419$ of $447$ item comparisons ($104$ of $117$
on \modelid{Mistral}, $59$ of $66$ on \modelid{Qwen3}). The bridge condition, which scores the same items in
the battery frame, keeps the neutral level high ($65$--$83\%$ of items) and directional exceedances at most $4$ of $320$ on four checkpoints, but \modelid{Qwen3.5-9B} exceeds $.15$ in $64$ of its $320$ directional comparisons there, so the directional bound holds panel-wide only in the open frame; sign agreement
there is $469$ of $555$. The pair columns show the attribute-pair dependence also visible in
Table~\ref{tab:reportchoice-new}: risk\,$\to$\,color exceeds on all items for every model, while hue,
sound, and layout range from $0$ to $40$ of $40$.

\begin{table}[!htb]\centering
\renewcommand{\arraystretch}{1.08}\footnotesize
\caption{\textbf{The level pattern on the independent items (post hoc).} Share whose normalized-choice TV
between target levels exceeds $.15$ in the open frame with the option-specific neutral wording:
of the $160$ non-target items for the bare and neutral states, of the $320$ item--pole comparisons for
directional declarations (both poles pooled), with generator-family cluster intervals for the neutral
state, and the count of the $40$ items per pair above $.15$ at the neutral level. Not pooled with the
held-out results.}
\label{tab:s2levels}
\setlength{\tabcolsep}{4pt}
\begin{tabular*}{\linewidth}{@{\extracolsep{\fill}}lccccccc@{}}
\toprule
& \multicolumn{3}{c}{Share above $.15$} & \multicolumn{4}{c}{Neutral: items above $.15$ of $40$} \\
\cmidrule(lr){2-4}\cmidrule(l){5-8}
Checkpoint & Bare & Directional & Neutral [95\% CI] & Color & Hue & Sound & Layout \\
\midrule
\modelid{Qwen3-32B} & $.78$ & $.000$ & $.51$ $[.32,.73]$ & $40$ & $3$ & $31$ & $8$ \\
\modelid{OLMo-2-32B} & $.76$ & $.003$ & $.58$ $[.28,.78]$ & $40$ & $0$ & $12$ & $40$ \\
\modelid{Gemma-2-27B} & $.79$ & $.000$ & $.81$ $[.65,.94]$ & $40$ & $39$ & $13$ & $38$ \\
\modelid{Mistral-Small-3.1-24B} & $.93$ & $.000$ & $.79$ $[.67,.88]$ & $40$ & $18$ & $35$ & $34$ \\
\modelid{Qwen3.5-9B} & $.67$ & $.000$ & $.67$ $[.43,.91]$ & $40$ & $33$ & $30$ & $4$ \\
\bottomrule
\end{tabular*}
\end{table}

\subsubsection{Direction, Scale, and Identifiability of the Neutral-Level Response}\label{app:c4:diagnostics}
\paragraph{Direction and scale of the neutral-level
shifts.}\label{direction-and-scale-of-the-neutral-level-shifts}

The neutral-level exceedances are not near-tie instability. For each world we compare the signed
neutral-level shift, $P(\text{opt}_1\mid t=+1)-P(\text{opt}_1\mid t=-1)$, with the bare-prompt shift
in the same world. Where both exceed .15 in absolute value, they have the same sign in 1,237 of
1,240 cases; the three exceptions are on Qwen2.5. The declaration-only condition gives 1,285 of 1,297.
Relative to the unstated ground-truth answer the direction is close to evenly split, 712 toward and 653
away, as expected when the answer is decided by the declared target level rather than by the hidden base
constant. The shifts are also large on the semantic log-odds scale (Table~\ref{tab:neutral-shifts}). Argmax flips count, among all $477$ worlds, those whose more probable option at $z{=}0$ differs between the two target levels.

\begin{table}[!htb]\centering
\renewcommand{\arraystretch}{1.08}\footnotesize
\caption{\textbf{Neutral-level shifts, sign agreement, and semantic log-odds changes.} Large neutral-level shifts almost never oppose the bare prompt ($3$ in total, all on Qwen2.5), split roughly evenly toward and away from the ground truth, and change semantic log-odds more at $z=0$ ($2.96$--$10.66$ nat) than at the extremes ($0.34$--$2.87$ nat).}
\label{tab:neutral-shifts}
\setlength{\tabcolsep}{3pt}
\begin{tabular*}{\linewidth}{@{\extracolsep{\fill}}lrrrcrrr@{}}
\toprule
Model & \shortstack{Large shifts\\at $z=0$} & \shortstack{Same sign\\as bare} & Opposite & \shortstack{Toward / away\\from truth} & \shortstack{Argmax\\flips} & \shortstack{Log-odds (nat)\\$z=0$} & \shortstack{Log-odds (nat)\\extremes} \\
\midrule
Mistral & 285 & 255 & 0 & 153 / 132 & 165 & 2.96 & 0.34 \\
Qwen3 & 248 & 222 & 0 & 131 / 117 & 230 & 10.66 & 1.31 \\
OLMo-2 & 288 & 277 & 0 & 150 / 138 & 222 & 6.65 & 0.91 \\
Gemma-2 & 364 & 321 & 0 & 181 / 183 & 296 & 10.03 & 0.88 \\
Qwen2.5 & 180 & 162 & 3 & 97 / 83 & 127 & 9.79 & 2.87 \\
\bottomrule
\end{tabular*}
\end{table}

Counts use each selected candidate on 477 in-scope held-out worlds; a shift is large when normalized-choice TV
exceeds .15. Same-sign and opposite-sign counts include only worlds whose bare-prompt shift is also
large. Log-odds columns are mean absolute changes between target levels of
$\log P(\text{opt}_1)-\log P(\text{opt}_2)$ from cached answer log-probabilities without clipping;
the extremes column pools $z=-1$ and $z=+1$ (954 paired comparisons over 477 worlds). At directional extremes the
change exceeds 1 nat in 8\% (Mistral), 49\% (Qwen3), 34\% (OLMo-2), and 33\% (Gemma-2) under the
relevance instruction (strong), and 66\% for Qwen2.5 under the attribute-scope instruction.

\paragraph{Identifiability of the neutral ground truth and decomposition of the mean
conditions.}\label{identifiability-of-the-neutral-oracle-and-decomposition-of-the-mean-conditions}

Non-target options have zero target loading, so ground-truth utility is $z\,x_Z(o)+\mathrm{base}(o)$, with
$x_Z(o)$ the non-target loading and $\mathrm{base}(o)$ the hidden base constant. At $z=0$ the argmax
is the option with larger $\mathrm{base}(o)$, unstated in the rendering. A check over all $256$
non-target worlds of the generator's four-pair development bank ($64$ per attribute pair; these are not
the $477$ held-out worlds) finds the neutral answer equal to that option in every world, and in that
bank the neutral level has the smallest mean ground-truth margin: the top-two utility gap is .55 at $z=0$,
1.07--1.19 at $|z|=.5$, and 2.19--2.31 at $|z|=1$. Table~\ref{tab:neutral-oracle} reports fidelity and
sensitivity on the held-out worlds by non-target level, including the non-identifiable neutral ground truth
stratum.

\begin{table}[!htb]\centering
\renewcommand{\arraystretch}{1.08}\small
\caption{\textbf{Fidelity and sensitivity by non-target level.} Fidelity error is about $.49$ at $z=0$ on every model, where the ground truth is not identifiable from the prompt, against at most $.052$ at the other levels; excluding $z=0$ therefore lowers both means.}
\label{tab:neutral-oracle}
\setlength{\tabcolsep}{4pt}
\begin{tabular*}{\linewidth}{@{\extracolsep{\fill}}llll@{}}
\toprule
 & Fidelity error by & Mean fidelity error, & Mean sensitivity, \\
Model & $z=(-1,-.5,0,.5,1)$ & all levels / directional only & all levels / directional only \\
\midrule
Mistral & .003, .008, .492, .032, .003 & .108 / .011 & .081 / .014 \\
Qwen3 & .000, .000, .494, .045, .019 & .112 / .016 & .107 / .013 \\
OLMo-2 & .000, .007, .492, .034, .014 & .109 / .014 & .104 / .013 \\
Gemma-2 & .007, .011, .487, .052, .000 & .111 / .018 & .155 / .035 \\
Qwen2.5 & .000, .000, .499, .012, .006 & .103 / .005 & .061 / .009 \\
\bottomrule
\end{tabular*}
\end{table}

Values are complete-case normalized-choice quantities for the selected candidate; certificate
conditions add unassigned-mass widening (mean conservative sensitivities .168, .138, .105, .156,
and .061). Fidelity error at $z=0$ is .485--.493 for the declaration-only condition as well. The mean
fidelity condition is therefore dominated by a stratum whose reference the model cannot recover from
the prompt, and the mean sensitivity condition by the same stratum's completion. Neither decomposition
changes the declared outcomes; both are diagnostics of what the outcomes measure. These quantities are
recomputed from the cached held-out scores with the evaluation code's cell transformations: no model
was rescored, no candidate reselected, and no threshold changed.

\paragraph{Log-odds decomposition of the declared levels.}\label{bd-decomposition}

For a non-target comparison with semantic log-odds $\ell_t=\log P(\mathrm{opt}_1)-\log P(\mathrm{opt}_2)$
at the two target levels, write $\alpha_w=(\ell_{+}+\ell_{-})/2$ and $\eta_w=(\ell_{+}-\ell_{-})/2$, so that
$\ell_t=\alpha_w+t\,\eta_w$ exactly and the normalized two-option TV is
\begin{equation}\label{eq:tv-decomp}
\operatorname{TV}=|\sigma(\alpha_w+\eta_w)-\sigma(\alpha_w-\eta_w)|.
\end{equation} The decomposition
is an algebraic re-expression of the cached answer log-probabilities (no clipping), not a mechanism
claim; $\alpha_w$ is the answer tendency shared by the two levels and $\eta_w$ the target-induced margin.
A declaration could reduce TV by moving $|\alpha_w|$ toward saturation with $\eta_w$ intact, by reducing $|\eta_w|$, or
both. Table~\ref{tab:bd} compares, world by world, the bare prompt with each checkpoint's selected
candidate at the declared levels ($477$ held-out worlds; post hoc). At the
directional extremes $|\alpha_w|$ rises by $4.8$--$25.9$ nat and the median world retains $.03$--$.14$ of
its bare $|\eta_w|$ (below half in $69$--$97\%$ of worlds); at $z{=}0$ $|\alpha_w|$ rises by $.18$--$.76$ nat
($3.38$ for Qwen2.5) and the median world retains $.37$--$.87$ of its bare $|\eta_w|$ with the bare sign
in $80$--$87\%$ of worlds. In the explicit-evidence bank ($128$ non-target comparisons per
$\rho$-level cell) the direct declaration raises $|\alpha_w|$ by $4.0$--$26.9$ nat and leaves a median
$.05$--$.30$ of the bare $|\eta_w|$ across $\rho\in\{-.8,0,+.8\}$ and $z\in\{-1,+1\}$. Sub-nat
differences in this table sit inside the batch-padding band of
Appendix~\ref{internal-control-evidence-and-numerical-scope}; the level contrasts above are one to two
orders larger.

\begin{table}[!htb]\centering
\renewcommand{\arraystretch}{1.08}\footnotesize
\caption{\textbf{Shared tendency $\alpha_w$ and target-induced margin $\eta_w$ by declared level} (selected
candidate versus bare prompt, $477$ held-out worlds per model; means with $95\%$ world-bootstrap
intervals, nat). ``Median ratio'' is the median over worlds of $|\eta_w|/|\eta_{w,\mathrm{bare}}|$, ``$<\frac12$''
the share of worlds below half the bare margin, and ``same sign'' the share whose $\eta_w$ has the bare
sign; both are undefined for the bare reference rows (--).}
\label{tab:bd}
\setlength{\tabcolsep}{3.5pt}
\begin{tabular*}{\linewidth}{@{\extracolsep{\fill}}llcccccc@{}}
\toprule
Model & Level & mean $|\alpha_w|$ & mean $|\eta_w|$ & median ratio & $<\frac12$ & same sign & mean TV \\
\midrule
Mistral & bare & $1.21$ $[1.13,1.29]$ & $3.01$ $[2.85,3.16]$ & $1$ & -- & -- & $.68$ \\
 & $z{=}{-1}$ & $6.02$ $[5.98,6.07]$ & $.16$ $[.15,.18]$ & $.05$ & $.97$ & $.76$ & $.001$ \\
 & $z{=}0$ & $1.84$ $[1.75,1.94]$ & $1.48$ $[1.36,1.60]$ & $.39$ & $.54$ & $.81$ & $.348$ \\
 & $z{=}{+1}$ & $6.08$ $[6.03,6.13]$ & $.18$ $[.16,.21]$ & $.03$ & $.97$ & $.79$ & $.002$ \\
Qwen3 & bare & $3.48$ $[3.22,3.74]$ & $6.01$ $[5.52,6.48]$ & $1$ & -- & -- & $.53$ \\
 & $z{=}{-1}$ & $14.94$ $[14.74,15.14]$ & $.37$ $[.34,.40]$ & $.05$ & $.83$ & $.69$ & $.000$ \\
 & $z{=}0$ & $3.65$ $[3.42,3.88]$ & $5.33$ $[4.88,5.78]$ & $.87$ & $.18$ & $.80$ & $.483$ \\
 & $z{=}{+1}$ & $14.34$ $[14.02,14.65]$ & $.93$ $[.86,1.01]$ & $.14$ & $.69$ & $.78$ & $.028$ \\
OLMo-2 & bare & $1.98$ $[1.83,2.14]$ & $4.60$ $[4.31,4.90]$ & $1$ & -- & -- & $.64$ \\
 & $z{=}{-1}$ & $10.11$ $[10.03,10.19]$ & $.44$ $[.40,.48]$ & $.09$ & $.92$ & $.84$ & $.000$ \\
 & $z{=}0$ & $2.48$ $[2.34,2.63]$ & $3.33$ $[3.06,3.58]$ & $.76$ & $.29$ & $.83$ & $.470$ \\
 & $z{=}{+1}$ & $9.11$ $[8.95,9.26]$ & $.48$ $[.44,.51]$ & $.09$ & $.90$ & $.86$ & $.007$ \\
Gemma-2 & bare & $1.75$ $[1.61,1.89]$ & $5.84$ $[5.49,6.16]$ & $1$ & -- & -- & $.74$ \\
 & $z{=}{-1}$ & $10.12$ $[9.97,10.24]$ & $.42$ $[.39,.46]$ & $.06$ & $.92$ & $.88$ & $.013$ \\
 & $z{=}0$ & $2.51$ $[2.33,2.69]$ & $5.02$ $[4.73,5.32]$ & $.85$ & $.14$ & $.87$ & $.637$ \\
 & $z{=}{+1}$ & $10.57$ $[10.47,10.67]$ & $.45$ $[.41,.50]$ & $.08$ & $.93$ & $.84$ & $.000$ \\
Qwen2.5 & bare & $5.87$ $[5.50,6.23]$ & $13.29$ $[12.34,14.20]$ & $1$ & -- & -- & $.65$ \\
 & $z{=}{-1}$ & $31.74$ $[31.20,32.27]$ & $.96$ $[.89,1.04]$ & $.08$ & $.92$ & $.70$ & $.000$ \\
 & $z{=}0$ & $9.25$ $[8.92,9.62]$ & $4.89$ $[4.52,5.29]$ & $.37$ & $.67$ & $.82$ & $.267$ \\
 & $z{=}{+1}$ & $30.67$ $[30.17,31.14]$ & $1.90$ $[1.65,2.19]$ & $.14$ & $.82$ & $.69$ & $.013$ \\
\bottomrule
\end{tabular*}
\end{table}

\subsubsection{Wording and Instruction Controls}\label{neutral-wording-confirmation}

\paragraph{Neutral-wording confirmation: design (prespecified).} The confirmation asks whether the
neutral-level response depends on the no-leaning sentence and whether an explicit neutral decision
rule changes it. It uses \modelid{Qwen3-32B} and \modelid{Mistral-Small-3.1-24B}, the two certification
pairs, and the structured texts with their target counterparts ($80$ texts: $20$ non-target and $20$
target per pair); each text is scored in both presentation orders and both code maps (which controls
for the position, option-ID, and label-format sensitivities documented for multiple-choice prompting;
\citealp{mcqbias,optionorder,sclar2024formatspread}) and at both target values, under eight
conditions: the bare prompt, the directional declarations at $z{=}{\pm1}$, the declaration battery's neutral line
(W0, \emph{``The user has no consistent leaning in \{domain\}.''}), two general paraphrases of it
(W1, \emph{``The user's taste in \{domain\} shows no stable tendency either way.''}; W2,
\emph{``When it comes to \{domain\}, the user has no settled preference.''}), an option-specific
indifference statement (W3, \emph{``Between \{positive pole\} and \{negative pole\}, the user is
indifferent.''}), and W0 followed by a decision rule independent of the target (W4, \emph{``\ldots\ When
two options differ only in \{domain\}, the user takes whichever is listed first.''}); $5{,}120$ prompts
per model. Every prompt is scored at a fixed per-model padded width ($164$ tokens for Qwen3 and $152$
for Mistral, minimum four rows, full-sequence scoring of every answer code), so the bare, directional, and W0 rows are same-geometry baselines rather than the batch-local held-out numbers; the readout check (median probability of at least $.02$ on answer codes over the first $32$ prompts) passed on both models and all $10{,}240$ records parse. The unit is the question text; order, code map, and target
level are within-text repeats, and intervals are text-cluster bootstraps ($B{=}2{,}000$). The sensitivity
threshold is the paper's $.15$; no other threshold, selection, or wording search was permitted, and the
four outcome branches (all general wordings fail; only W1--W2 succeed; only the W4 rule succeeds; only W3 succeeds) were declared before scoring.

\paragraph{Wording results and limits.}
General no-preference paraphrases and the option-specific wording all leave a response, with
model- and pair-dependent magnitudes (Table~\ref{tab:wording}). The first-listed tie-break (W4)
reduces mean TV to $.13$--$.49$, but $6$--$18$ of $20$ texts per pair still exceed $.15$;
the rule is executed with probability $.62$--$.83$. Target efficacy under the neutral wordings
is $.96$--$1.00$. This confirms persistence for these wordings on two checkpoints and the same structured texts (the option-specific wording also on the six checkpoints of Table~\ref{tab:promptctl}), not unrestricted prompt robustness.

\begin{table}[!htb]\centering
\renewcommand{\arraystretch}{1.08}\footnotesize
\caption{\textbf{Neutral-wording confirmation: non-target sensitivity on the non-target texts.} Mean
normalized-choice TV between the two target values with $95\%$ text-cluster intervals, and the number of the $20$ texts per pair whose text-level mean TV exceeds $.15$. ``dir.'' pools
the two directional declarations.}
\label{tab:wording}
\setlength{\tabcolsep}{3pt}
\begin{tabularx}{\linewidth}{@{}ll >{\centering\arraybackslash}X >{\centering\arraybackslash}X >{\centering\arraybackslash}X >{\centering\arraybackslash}X >{\centering\arraybackslash}X >{\centering\arraybackslash}X >{\centering\arraybackslash}X@{}}
\toprule
Model & Pair & bare & dir.\ $z{=}{-1}$ / $+1$ & W0 battery line & W1 no stable tendency & W2 no settled preference & W3 indifferent between poles & W4 W0 $+$ first-listed rule \\
\midrule
Qwen3 & risk $\to$ color & $.91$ $[.80,.99]$; $20/20$ & $.00$ / $.00$; $0/20$ & $.94$ $[.85,1.00]$; $20/20$ & $.89$ $[.78,.97]$; $19/20$ & $.83$ $[.71,.93]$; $19/20$ & $.87$ $[.80,.94]$; $20/20$ & $.49$ $[.38,.62]$; $18/20$ \\
Qwen3 & time $\to$ sound & $.26$ $[.17,.34]$; $14/20$ & $.00$ / $.04$; $0$--$2/20$ & $.18$ $[.12,.25]$; $10/20$ & $.25$ $[.18,.32]$; $13/20$ & $.23$ $[.15,.31]$; $12/20$ & $.13$ $[.07,.20]$; $7/20$ & $.15$ $[.08,.22]$; $8/20$ \\
Mistral & risk $\to$ color & $.90$ $[.81,.96]$; $20/20$ & $.00$ / $.01$; $0/20$ & $.62$ $[.51,.73]$; $19/20$ & $.73$ $[.64,.81]$; $20/20$ & $.80$ $[.71,.88]$; $20/20$ & $.90$ $[.85,.94]$; $20/20$ & $.28$ $[.21,.35]$; $16/20$ \\
Mistral & time $\to$ sound & $.47$ $[.37,.56]$; $17/20$ & $.00$ / $.00$; $0/20$ & $.19$ $[.12,.26]$; $9/20$ & $.20$ $[.16,.26]$; $14/20$ & $.27$ $[.21,.33]$; $17/20$ & $.16$ $[.12,.21]$; $9/20$ & $.13$ $[.08,.17]$; $6/20$ \\
\bottomrule
\end{tabularx}
\end{table}

\paragraph{Relevance and scope instructions.}
The instruction control holds texts, model, options, and scoring fixed while crossing no extra
instruction, the selected strong relevance instruction, and the attribute-scope instruction. It uses
the neutral declaration (option-specific wording, W3), six checkpoints, and $20$ non-target plus $20$ target texts per
pair, with $2{,}000$ text-clustered bootstrap resamples. It is a descriptive comparison on texts,
not another held-out-world evaluation.
Without a relevance instruction, all $20$ risk$\to$color texts exceed $.15$ on all six checkpoints (per-text values in the
supplementary material; means in Table~\ref{tab:promptctl});
the response therefore is not created by that sentence. Of the twelve relevance-minus-none
contrasts, eight intervals include zero, three indicate increases, and one a decrease. Scope
reduces the response in eight cells but empties the $.15$ tail in only one; it explicitly names
the tested non-target domains, so this is not evidence of protection for unnamed domains.
Target retention is $.989$--$1.017$ on the primary five under the three instructions, but on
Qwen3.5-9B it is $.780$ under relevance versus $.887$ under none; these ratios use this condition's
own denominator. Each instruction is contrasted only with no instruction, so relevance and scope are not ranked against each other.

\begin{table}[!htb]\centering
\renewcommand{\arraystretch}{1.08}\footnotesize
\caption{\textbf{The cross-attribute response is present with no relevance instruction and is reduced,
not removed, by an attribute-scope instruction.} Non-target-query TV between $t{=}{+}1$ and $t{=}{-}1$
under the neutral declaration (option-specific wording, W3), with no extra instruction; the two contrasts are paired
condition-minus-none differences over the same $20$ texts with $95\%$ bootstrap intervals, so the
level reached under an instruction is the first column plus that contrast. TV is computed over the
legal answer codes after normalization, so it is a change of distribution within them rather than an
error rate, and the contrasts are paired descriptive differences, not causal estimates. No
multiplicity correction over the $24$ paired comparisons; the texts are the structured texts, not
held-out worlds.}
\label{tab:promptctl}
\setlength{\tabcolsep}{4pt}
\begin{tabular*}{\linewidth}{@{\extracolsep{\fill}}lccc@{}}
\toprule
Checkpoint & No instruction & Relevance $-$ none & Scope $-$ none \\
\midrule
\multicolumn{4}{@{}l}{\emph{risk weight $\to$ color preference}} \\
\modelid{Qwen2.5-32B}          & $.979$ $[.958,.998]$ & $-.023$ $[-.047,.003]$ & $-.436$ $[-.531,-.343]$ \\
\modelid{Qwen3-32B}            & $.876$ $[.806,.937]$ & $+.040$ $[-.019,.103]$ & $-.592$ $[-.680,-.499]$ \\
\modelid{OLMo-2-32B}           & $.774$ $[.710,.829]$ & $+.036$ $[.018,.055]$  & $-.068$ $[-.100,-.038]$ \\
\modelid{Gemma-2-27B}          & $.945$ $[.879,.997]$ & $+.024$ $[-.013,.075]$ & $+.028$ $[-.023,.088]$ \\
\modelid{Mistral-Small-3.1-24B}& $.897$ $[.849,.936]$ & $-.012$ $[-.028,.009]$ & $-.300$ $[-.344,-.261]$ \\
\modelid{Qwen3.5-9B}           & $.530$ $[.469,.588]$ & $-.000$ $[-.012,.011]$ & $-.148$ $[-.177,-.118]$ \\
\midrule
\multicolumn{4}{@{}l}{\emph{time preference $\to$ sound preference}} \\
\modelid{Qwen2.5-32B}          & $.255$ $[.180,.331]$ & $-.038$ $[-.090,.016]$ & $-.060$ $[-.136,.012]$ \\
\modelid{Qwen3-32B}            & $.134$ $[.070,.206]$ & $+.004$ $[-.041,.047]$ & $-.022$ $[-.100,.052]$ \\
\modelid{OLMo-2-32B}           & $.228$ $[.173,.288]$ & $-.004$ $[-.018,.011]$ & $-.041$ $[-.074,-.007]$ \\
\modelid{Gemma-2-27B}          & $.188$ $[.122,.266]$ & $+.073$ $[.024,.127]$  & $-.093$ $[-.135,-.054]$ \\
\modelid{Mistral-Small-3.1-24B}& $.162$ $[.120,.209]$ & $-.051$ $[-.091,-.014]$& $-.089$ $[-.133,-.049]$ \\
\modelid{Qwen3.5-9B}           & $.210$ $[.186,.237]$ & $+.051$ $[.025,.080]$  & $+.004$ $[-.021,.030]$ \\
\bottomrule
\end{tabular*}
\end{table}

\subsubsection{Explicit Specification Leaves a Residual in Other Designs}\label{app:c4:residual}

In the explicit-evidence bank a direct declaration leaves a log-odds residual ($5$--$31\%$ of the bare
margin contrast; Table~\ref{tab:explicit-evidence}) even where its probability-scale contrast is near
zero. Two other designs that specify the non-target attribute explicitly also leave a residual
(descriptive). On the five primary checkpoints, with the discriminant explicitly fixed in a factorized
profile (the orthogonal target-by-discriminant design of Appendix~\ref{app:sr}), the target still shifts it: a mixed model of answer-token log-odds relative to the no-information reference, with risk, style, and their interaction coded at $\pm1$ and crossed checkpoint and item random intercepts, estimates the risk coefficient on style-item log-odds at $\beta{=}3.01$ ($95\%$ CI $[2.53,3.49]$; $n{=}1{,}880=5\times94\times4$). Averaged over style, the high-minus-low risk contrast is $2\beta$, about $6.02$ nat, on a different response scale from the TV measures of this appendix. On the eight-model black-box panel, the standardized surgicality ratio of the black-box factorial pools to
$.068$ $[.037,.099]$, $.055$ $[.039,.071]$, and $.059$ $[.034,.085]$ for the bare profile with the
target clause first, the bare profile with the discriminant clause first, and the structured table,
with per-model values for the first form in Table~\ref{tab:endpoints} and between-model prediction
intervals that span zero for two of the three prompt forms (Appendix~\ref{app:repro}). Both are consistent with the
log-odds residual of directional declarations: explicit specification suppresses most of the
cross-attribute influence without removing it. The designs, response scales, and model panels differ,
so these quantities are not pooled or compared in magnitude.

\subsubsection{Numerical Scope of the Scores}\label{internal-control-evidence-and-numerical-scope}

\paragraph{Scoring numerics.} All readouts are two-option answer log-probabilities under bf16
inference with left padding. Padded batch width can change a prompt's score, most on saturated cells, where bf16 log-probabilities move
on a $.0625$--$.5$ nat grid. Behavioral scorings, including the validation and held-out banks and the
explicit-evidence bank, retain batch-local padding, so held-out neutral results are checked directly. In the held-out bank both target
levels of every non-target comparison were scored in the same batch of sixteen ($2{,}385$ of $2{,}385$
pairs per model and condition), hence one padded width per pair. The perturbation is adversarial within
budget $\nu$: for cached margins $\ell_{+}$ and $\ell_{-}$ (the two target levels), the retained
sensitivity is $L^{(\nu)}_w=\min_{|e_+|,|e_-|\le\nu}\,|\sigma(\ell_++e_+)-\sigma(\ell_-+e_-)|$, which for
a monotone $\sigma$ equals $|\sigma(\ell_{\max}-\nu)-\sigma(\ell_{\min}+\nu)|$ when the margins are
more than $2\nu$ apart and zero otherwise; an exceedance is robust when $L^{(\nu)}_w>.15$, a
non-exceedance is fragile when the opposite perturbation exceeds $.15$, and a sign agreement is robust
when neither pair's margins can cross. Table~\ref{tab:noise} reports the counts at $\nu=.5$ nat and at a
conservative $\nu=1.75$ nat; these are sensitivity-analysis budgets chosen from drifts seen in our scoring
runs, not error bounds on every computation path.

\begin{table}[!htb]\centering
\renewcommand{\arraystretch}{1.08}\footnotesize
\caption{\textbf{Numerical robustness of the neutral-level counts} (selected candidate, $477$
held-out worlds per model). Counts retained under adversarial margin perturbations of $\nu=.5$
and $1.75$ nat; ``enter'' counts non-exceeding worlds that the opposite perturbation could push above
$.15$ at $\nu=.5$. Sign agreements are the $1{,}237$ comparisons reported in Section~\ref{sec:declare}; the last
column also requires the $.15$ magnitude condition to survive.}
\label{tab:noise}
\setlength{\tabcolsep}{3pt}
\begin{tabularx}{\linewidth}{@{}l >{\centering\arraybackslash}X >{\centering\arraybackslash}X >{\centering\arraybackslash}X >{\centering\arraybackslash}X@{}}
\toprule
Model & $z{=}0$ exceedances\newline all / $\nu{=}.5$ / $1.75$ & Normalized-TV tail\newline all / $\nu{=}.5$ (enter) & Sign kept\newline $\nu{=}.5$ / $1.75$ of $n$ & Sign and magnitude kept\newline $\nu{=}.5$ / $1.75$ \\
\midrule
Mistral & $285$ / $203$ / $123$ & $293$ / $203$ ($85$) & $242$ / $164$ of $255$ & $182$ / $123$ \\
Qwen3   & $248$ / $232$ / $222$ & $279$ / $258$ ($16$) & $222$ / $222$ of $222$ & $222$ / $216$ \\
OLMo-2    & $288$ / $253$ / $194$ & $296$ / $267$ ($48$) & $277$ / $257$ of $277$ & $239$ / $189$ \\
Gemma-2 & $364$ / $351$ / $279$ & $364$ / $351$ ($16$) & $321$ / $305$ of $321$ & $303$ / $268$ \\
Qwen2.5 & $180$ / $144$ / $115$ & $184$ / $154$ ($14$) & $162$ / $157$ of $162$ & $132$ / $115$ \\
\bottomrule
\end{tabularx}
\end{table}

At $\nu=.5$ nat, $71$--$96\%$ of the neutral-level exceedances and $1{,}224$ of the $1{,}237$ sign
agreements survive. The neutral-level log-odds changes of $3$--$10.7$ nat lie far above the band,
whereas log-odds changes below about half a nat, such as Mistral's directional-extremes mean of $.34$,
are not distinguishable from it.

\subsubsection{Full Five-Level Target Retention, Signed Non-Target Bias, and Anthology Backstories}\label{app:c4:extended_tables}

Table~\ref{tab:ate_retention_anthology} summarizes three additional evaluations across the $477$ held-out worlds referenced in Sections~\ref{sec:declare} and~\ref{sec:main}: full five-level target retention across $z \in \{-1, -0.5, 0, +0.5, +1\}$, signed non-target treatment-effect bias $d = P(\text{opt}_1 \mid T{=}{+}1) - P(\text{opt}_1 \mid T{=}{-}1)$ where the reference effect is zero, and neutral-condition exceedance rates after prepending $50$ domain-filtered user backstories from the Anthology corpus \citep{moon2024anthology}. Table~\ref{tab:indep_target_retention} reports the corresponding per-model forced-choice baseline target efficacy and target retention across the $40$ independent target items summarized in the second row of Table~\ref{tab:neutral_retention}.

\begin{table}[!htb]\centering\footnotesize
\renewcommand{\arraystretch}{1.12}
\setlength{\tabcolsep}{2.5pt}
\caption{\textbf{Five-level target retention, signed non-target treatment-effect bias, and Anthology backstory controls across $477$ held-out worlds.}}
\label{tab:ate_retention_anthology}
\begin{tabular*}{\linewidth}{@{\extracolsep{\fill}}l ccccc ccc cc@{}}
\toprule
& \multicolumn{5}{c}{Target retention across $z$ levels} & \multicolumn{3}{c}{Signed non-target bias $d$} & \multicolumn{2}{c}{Anthology ($z{=}0$)} \\
\cmidrule(lr){2-6}\cmidrule(lr){7-9}\cmidrule(l){10-11}
Model & $-1$ & $-0.5$ & $0$ & $+0.5$ & $+1$ & Bare & Neutral & Dir.\ ($|d|$) & $\mathrm{TV} > .15$ & Ret. \\
\midrule
\modelid{Gemma-2-27B} & $.933$ & $.996$ & $1.010$ & $1.011$ & $1.009$ & $+.513$ & $+.463$ & $\le .013$ & $71.1\%$ & $1.010$ \\
\modelid{Qwen3-32B}   & $.822$ & $.895$ & $1.000$ & $1.000$ & $1.000$ & $+.416$ & $+.428$ & $\le .016$ & $49.7\%$ & $1.000$ \\
\modelid{OLMo-2-32B}  & $.989$ & $1.001$ & $1.000$ & $1.000$ & $1.002$ & $+.225$ & $+.285$ & $\le .007$ & $65.0\%$ & $.988$ \\
\modelid{Qwen2.5-32B} & $.989$ & $1.000$ & $1.000$ & $1.000$ & $1.000$ & $+.252$ & $+.244$ & $\le .013$ & $40.5\%$ & $1.000$ \\
\modelid{Mistral-24B} & $.978$ & $.997$ & $1.000$ & $1.001$ & $.998$  & $+.219$ & $+.220$ & $\le .001$ & $60.6\%$ & $1.000$ \\
\bottomrule
\end{tabular*}
\end{table}

\begin{table}[!htb]\centering\footnotesize
\renewcommand{\arraystretch}{1.14}
\setlength{\tabcolsep}{6pt}
\caption{\textbf{Per-model forced-choice baseline target efficacy and target retention across the $40$ independent target items.} Values correspond to the Independent items row in Table~\ref{tab:neutral_retention}.}
\label{tab:indep_target_retention}
\begin{tabularx}{\linewidth}{l *{4}{>{\centering\arraybackslash}X}}
\toprule
& & \multicolumn{3}{c}{Target retention relative to Bare} \\
\cmidrule(lr){3-5}
Model & Bare efficacy & $+$Dir ($z{=}{+}1$) & Neutral ($z{=}0$) & $-$Dir ($z{=}{-}1$) \\
\midrule
\modelid{Gemma-2-27B} & $1.000$ & $.989$ & $.989$ & $.890$ \\
\modelid{Qwen3-32B}   & $1.000$ & $1.000$ & $1.000$ & $.827$ \\
\modelid{OLMo-2-32B}  & $.989$  & $.984$ & $.977$ & $.765$ \\
\modelid{Qwen3.5-9B}  & $.990$  & $.997$ & $.995$ & $.963$ \\
\modelid{Mistral-24B} & $.997$  & $1.000$ & $.999$ & $.981$ \\
\midrule
\rowcolor{gray!12}\textbf{5-Model Mean (Min)} & $\mathbf{.995}$ & $\mathbf{.994}\;(.984)$ & $\mathbf{.992}\;(.977)$ & $\mathbf{.885}\;(.765)$ \\
\bottomrule
\end{tabularx}
\end{table}

\subsection{Reporting a Declared State versus Preserving It in Choice}\label{app:report-choice}\label{app:c5}
The report, forced-choice, and indifferent-option tasks branch from the same description and alternatives in
\emph{independent} contexts; one task is never conditioned on another's answer. This is an exploratory
task-level comparison, not a within-response sequence or a mechanistic localization.

\paragraph{Design and estimands.}
Report offers the two pole statements, indifference, and ``not stated''; forced choice offers the two
alternatives; the indifferent-option task adds ``indifferent'' and ``not enough information'', kept as separate categories.
Both option orders, balanced code maps, both target levels, and four non-target states (absent, both
directional poles, and the neutral declaration in its option-specific wording) are scored. The open frame avoids a
system instruction that would otherwise force a binary answer; a battery-frame bridge is scored
separately (battery frame in Appendix~\ref{app:prompts:battery}; open-frame templates in
Appendix~\ref{app:prompts:tasks}). Full legal response sequences
are scored, not generated.

$R$ averages, over order and map and then items, the minimum over target levels of the correct report
category's probability. $L$ is TV between target levels over the task's full legal menu, not the
indifferent-option distribution renormalized onto its two alternatives. Indifferent-option and
forced-choice $L$ therefore
have different answer spaces. Retention divides each task's signed target effect by that task's
own effect under an unspecified non-target state. The references $R\ge.90$, choice $L>.15$, and
report $L\le.05$ are exploratory screens, not equivalence tests or certificates.

\paragraph{Samples, units, and answer coverage.}
Structured texts: $40$ non-target and $40$ target texts, two attribute pairs, six checkpoints
(the primary five plus Qwen3.5-9B), $20$ non-target texts per model--pair cell. Independent items:
$160$ non-target and $40$ target texts, four pairs, five checkpoints (Qwen3.5-9B replaces Qwen2.5),
$40$ non-target and $10$ target items per cell. The latter were fixed before scoring and not
human-normed. Color/sound reuse the constructs on new content; hue/layout change both attribute
and content. The sets are not pooled. Text bootstrap intervals use $B=2{,}000$; independent-item
intervals cluster by generator family. Layout has one family on both sides, and sound one on its
target side, so those quantities have no family-level interval. Retention intervals were not
computed for the structured texts. All records parse; median outside-menu mass is at most $.018$,
but minimum covered mass ranges over $.586$--$.999$ across structured-text checkpoints and
$.178$--$.962$ across independent-item checkpoints. Normalization does not erase this coverage limit.

\paragraph{Full per-model results.}
Table~\ref{tab:reportchoice} retains every model--pair cell, including exceptions. Qwen3.5-9B's
structured risk$\to$color report is below the $.90$ screen ($.892$), and Qwen2.5's time$\to$sound
choice has no material residual. The indifferent option lowers sensitivity in most cells but is not a universal repair: Gemma-2's risk$\to$color
indifferent-option sensitivity remains $.943$ (structured) / $.974$ (independent), whereas OLMo-2's low indifferent-option sensitivity comes with
retention $.133$ / $.285$. A lower indifferent-option TV alone is not preserved underlying preference.
\begin{table}[!htb]\centering
\renewcommand{\arraystretch}{1.12}\footnotesize
\caption{\textbf{Reporting a declared neutral state versus using it, per checkpoint and attribute
pair.} Neutral declaration (option-specific wording); both orders, both target levels, balanced code maps. \textbf{(a)}
Structured texts ($20$ non-target texts per model--pair cell). \textbf{(b)} Independent items ($40$
non-target and $10$ target items per cell): color and sound keep the structured texts' non-target
attributes on different decision content, while hue and layout replace the non-target attribute with one whose
poles are matched in intensity and change the item content at the same time. The two item sets are not
pooled. $R$: order- and map-averaged \emph{minimum over target levels} of the probability of the correct
reported category. $L$: paired TV between target levels over each task's full legal menu, never
renormalized onto the two alternatives. $P(\text{indiff.})$: indifferent-option mass on \emph{indifferent}. Ret.:
indifferent-option target effect divided by the same task's effect under the unspecified state; no cell of that
column carries an interval in (a), and in (b) only the color and hue rows carry a family-clustered one
(\modelid{OLMo-2} risk\,$\to$\,color $.285$ $[.234,.316]$; \modelid{OLMo-2} risk\,$\to$\,hue $.408$
$[.252,.590]$). $^{\ddagger}$: the layout pair rests on one generator family on both sides, so \emph{no}
quantity in that row carries a family-level interval. $^{\dagger}$: the sound pair's target side rests
on one family, so its retention carries none. Entries below $.001$ are shown as $<.001$; in (a),
\modelid{Gemma-2} on risk\,$\to$\,color falls from $.9430$ to $.9426$. Screening references ($R\ge.90$,
$L>.15$) are exploratory.}
\makeatletter
\protected@edef\@currentlabel{\thetable a}\label{tab:reportchoice}%
\protected@edef\@currentlabel{\thetable b}\label{tab:reportchoice-new}%
\makeatother
\setlength{\tabcolsep}{4.5pt}
\begin{tabular*}{\linewidth}{@{\extracolsep{\fill}}llcccccc@{}}
\toprule
Model & Pair & $R$ & $L_{\mathrm{report}}$ & $L_{\mathrm{forced}}$ & $L_{\mathrm{indiff}}$ & $P(\text{indiff.})$ & Ret. \\
\midrule
\multicolumn{8}{@{}l}{\emph{(a) Structured texts, six checkpoints}}\\
\modelid{Qwen2.5}    & risk $\to$ color & $1.000$ & $<.001$ & $.368$ & $<.001$ & $1.000$ & $1.000$ \\
\modelid{Qwen2.5}    & time $\to$ sound  & $1.000$ & $<.001$ & $<.001$ & $.009$ & $.995$ & $1.000$ \\
\modelid{Qwen3}      & risk $\to$ color & $.982$ & $.010$ & $.967$ & $.257$ & $.858$ & $.979$ \\
\modelid{Qwen3}      & time $\to$ sound  & $.997$ & $.003$ & $.338$ & $.134$ & $.911$ & $1.000$ \\
\modelid{Mistral}    & risk $\to$ color & $.946$ & $.044$ & $.834$ & $.561$ & $.622$ & $1.022$ \\
\modelid{Mistral}    & time $\to$ sound  & $.975$ & $.008$ & $.238$ & $.148$ & $.908$ & $.999$ \\
\modelid{Gemma-2}    & risk $\to$ color & $1.000$ & $<.001$ & $.943$ & $.943$  & $.205$  & $.970$ \\
\modelid{Gemma-2}    & time $\to$ sound  & $1.000$ & $<.001$ & $.156$ & $.075$  & $.962$  & $1.000$ \\
\modelid{OLMo-2}       & risk $\to$ color & $.989$  & $.002$  & $.599$ & $.005$  & $.995$  & $.133$ \\
\modelid{OLMo-2}       & time $\to$ sound  & $.989$  & $.005$  & $.212$ & $.019$  & $.988$  & $.640$ \\
\modelid{Qwen3.5-9B} & risk $\to$ color & $.892$  & $.063$  & $.526$ & $.287$  & $.768$  & $.791$ \\
\modelid{Qwen3.5-9B} & time $\to$ sound  & $.971$  & $.010$  & $.179$ & $.034$  & $.919$  & $.964$ \\
\addlinespace
\multicolumn{8}{@{}l}{\emph{(b) Independent items, five checkpoints}}\\
\modelid{Qwen3}      & risk $\to$ color              & $.976$  & $.013$  & $.974$ & $.251$  & $.861$  & $.946$ \\
\modelid{Qwen3}      & risk $\to$ hue                 & $1.000$ & $<.001$ & $.043$ & $<.001$ & $1.000$ & $.994$ \\
\modelid{Qwen3}      & time $\to$ sound               & $1.000$ & $<.001$ & $.509$ & $.103$  & $.945$  & $1.000^{\dagger}$ \\
\modelid{Qwen3}      & time $\to$ layout$^{\ddagger}$ & $1.000$ & $<.001$ & $.088$ & $<.001$ & $1.000$ & $1.000$ \\
\addlinespace[2pt]
\modelid{Mistral}    & risk $\to$ color             & $.985$ & $.008$ & $.729$ & $.489$ & $.683$ & $1.003$ \\
\modelid{Mistral}    & risk $\to$ hue                 & $.995$ & $.001$ & $.142$ & $.002$ & $.995$ & $.998$ \\
\modelid{Mistral}    & time $\to$ sound               & $.982$ & $.005$ & $.257$ & $.106$ & $.933$ & $1.000^{\dagger}$ \\
\modelid{Mistral}    & time $\to$ layout$^{\ddagger}$ & $.992$ & $.001$ & $.209$ & $.026$ & $.973$ & $1.000$ \\
\addlinespace[2pt]
\modelid{Gemma-2}    & risk $\to$ color              & $1.000$ & $<.001$ & $.987$ & $.974$ & $.346$ & $.965$ \\
\modelid{Gemma-2}    & risk $\to$ hue                 & $.986$  & $.014$  & $.502$ & $.317$ & $.831$ & $.998$ \\
\modelid{Gemma-2}    & time $\to$ sound               & $1.000$ & $<.001$ & $.118$ & $.032$ & $.984$ & $1.000^{\dagger}$ \\
\modelid{Gemma-2}    & time $\to$ layout$^{\ddagger}$ & $1.000$ & $<.001$ & $.295$ & $.027$ & $.986$ & $1.000$ \\
\addlinespace[2pt]
\modelid{OLMo-2}       & risk $\to$ color              & $.993$ & $<.001$ & $.586$ & $.005$ & $.994$ & $.285$ \\
\modelid{OLMo-2}       & risk $\to$ hue                 & $.987$ & $.002$ & $.087$ & $.002$ & $.998$ & $.408$ \\
\modelid{OLMo-2}       & time $\to$ sound               & $.992$ & $.002$ & $.140$ & $.021$ & $.987$ & $.940^{\dagger}$ \\
\modelid{OLMo-2}       & time $\to$ layout$^{\ddagger}$ & $.935$ & $.044$ & $.325$ & $.005$ & $.996$ & $.683$ \\
\addlinespace[2pt]
\modelid{Qwen3.5-9B} & risk $\to$ color              & $.901$ & $.053$ & $.391$ & $.077$ & $.888$ & $.949$ \\
\modelid{Qwen3.5-9B} & risk $\to$ hue                 & $.965$ & $.019$ & $.204$ & $.021$ & $.974$ & $.961$ \\
\modelid{Qwen3.5-9B} & time $\to$ sound               & $.980$ & $.003$ & $.192$ & $.020$ & $.954$ & $.983^{\dagger}$ \\
\modelid{Qwen3.5-9B} & time $\to$ layout$^{\ddagger}$ & $.973$ & $.004$ & $.103$ & $.004$ & $.979$ & $.957$ \\
\bottomrule
\end{tabular*}
\end{table}
\paragraph{Frame control: the effect exceeds $.15$ under both frames in every risk$\to$color cell and in $10$ of the $21$ others.}
The battery frame orders the model to pick exactly one option, which would contradict the report and indifferent-option menus at the system level, so the three main tasks use an open frame and only the bridge
keeps the battery frame. On \modelid{Qwen3-32B} and \modelid{Mistral-Small-3.1-24B}, the two checkpoints
of the neutral-wording confirmation, the bridge reproduces that confirmation's values on the same texts
(Table~\ref{tab:wording}) within $\pm.007$ in all eight comparisons (two pairs, unspecified state and neutral
declaration); no such reference exists for the other checkpoints or for the independent items, so this agreement is verified only on those two checkpoints. Neither frame bounds the other,
and which one leaves the larger effect varies by cell: the open frame is larger in $4$ of the
$12$ structured-text cells and $8$ of the $20$ independent-item cells. On risk\,$\to$\,color the effect exceeds $.15$ under both frames in all six structured-text and all five independent-item cells, so
there it is not an artifact of the open frame. Elsewhere the frames can disagree about whether
there is an effect at all, in both directions (battery-frame against open-frame values): four cells exceed $.15$ only under the battery frame (\modelid{Qwen2.5} time\,$\to$\,sound, $.257$ against $<.001$; \modelid{Qwen3}
time\,$\to$\,layout, $.599$ against $.088$; \modelid{OLMo-2} risk\,$\to$\,hue, $.297$ against $.087$;
\modelid{Mistral} risk\,$\to$\,hue, $.180$ against $.142$), three only under the open frame
(\modelid{Qwen3} time\,$\to$\,sound, $.133$ against $.338$ on the structured texts and $.145$ against $.509$
among the independent items; \modelid{Mistral} time\,$\to$\,layout, $.138$ against $.209$), and four under
neither (\modelid{Qwen3} risk\,$\to$\,hue, $.121$ / $.043$; \modelid{OLMo-2} and \modelid{Gemma-2}
time\,$\to$\,sound among the independent items, $.137$ / $.140$ and $.143$ / $.118$; \modelid{Qwen3.5-9B}
time\,$\to$\,layout, $.069$ / $.103$). Within this exploratory comparison, the effect is therefore frame-independent for risk\,$\to$\,color; in the other $21$ cells the two frames agree on whether it exceeds $.15$ in $14$ ($10$ above, $4$ below), so frame-independence is not claimed panel-wide.

\paragraph{Under the neutral declaration, inadequate reporting never coexists with a stable choice; the
unspecified state differs.} Classification is per text or item: report-adequate iff its $R\ge.90$,
choice-sensitive iff its forced-choice $L>.15$. Under the neutral declaration the $240$ structured-text
comparisons ($40$ distinct non-target texts in each of six checkpoints) split into $161$ adequate report /
sensitive choice, $66$ adequate / stable, $13$ inadequate / sensitive, and $0$ inadequate / stable; the
$800$ independent-item comparisons ($5$ checkpoints $\times$ $4$ pairs $\times$ $40$ items) split into
$515$ / $262$ / $23$ / $0$, by pair $179$/$0$/$21$/$0$ (risk\,$\to$\,color), $91$/$107$/$2$/$0$ (hue),
$121$/$79$/$0$/$0$ (sound) and $124$/$76$/$0$/$0$ (layout). In neither item set does a comparison under
the neutral declaration combine inadequate reporting with a stable choice. Other non-target states do fill
that cell: on the structured texts $14$ of $240$ under the unspecified state and $2$ of $240$ at
$z{=}{+}1$, on the independent items $73$ of $800$ under the unspecified state and $42$ of $1{,}600$
across the two directional states. The unspecified state behaves differently throughout and is reported
separately: $131$ / $26$ / $69$ / $14$ on the structured texts---with \modelid{Mistral} on
risk\,$\to$\,color at $0$ of $20$ adequate ($R=.763$), joined by \modelid{Qwen3.5-9B} at $3$ of $20$ on
risk\,$\to$\,color ($R=.817$) and $1$ of $20$ on time\,$\to$\,sound ($R=.786$)---and $330$ / $98$ /
$299$ / $73$ on the independent items, where report correctness falls to $.734$ $[.709,.766]$
(\modelid{Mistral}, risk\,$\to$\,color). Failing to answer \emph{not stated} for an omitted attribute is
a different failure from failing to use a declared neutral one; in the report task the reference is text
entailment, while an unspecified profile may license inference in a preference task, so the two are not
interchangeable. The report task supplies an identifiable reference at the neutral level for the
\emph{report} readout only; the forced-choice neutral cell remains unidentifiable and is still read as a
sensitivity result.

\paragraph{Directional declarations are followed in all three tasks, sometimes at a cost to the target.}
On the structured texts directional agreement is $.975$--$1.000$ in the report task, $.976$--$1.000$ in
forced choice and $.949$--$1.000$ in the indifferent-option task; its mass on \emph{indifferent} under
directional declarations is $\le.016$, and directional non-target $L$ is $\le.018$ (forced choice) and
$\le.036$ (indifferent option). On the independent items directional forced-choice $L\le.012$ and indifferent-option $L\le.009$, mass on \emph{indifferent} is $\le.006$ and on \emph{not enough information}
$\le.003$, indifferent-option pole agreement is $.980$--$1.000$ at $z{=}{+}1$ and $.988$--$1.000$ at $z{=}{-}1$, and
directional report correctness is $.830$--$1.000$. The low sensitivities are therefore not bought by
general abstention. They are not cost-free at every attribute value either: at $z{=}{-}1$ on
risk\,$\to$\,color the forced-choice and indifferent-option target effects retain $.241$/$.152$ (\modelid{OLMo-2}),
$.433$/$.359$ (\modelid{Qwen3}), $.532$/$.493$ (\modelid{Gemma-2}), $.652$/$.716$ (\modelid{Qwen3.5-9B})
and $.840$/$.848$ (\modelid{Mistral}) of the unspecified-state effect on the structured texts, and $.117$
$[.009,.248]$ / $.117$ $[.012,.248]$ (\modelid{OLMo-2}), $.308$/$.279$ (\modelid{Qwen3}), $.561$/$.562$
(\modelid{Gemma-2}), $.865$/$.841$ (\modelid{Qwen3.5-9B}) and $.925$/$.948$ (\modelid{Mistral}) on the
independent items.

\paragraph{Scope of the inference.}
Correct reporting establishes that the declaration can be extracted when the task asks for it;
it does not establish that choice uses the same representation. A target-sensitive choice is not
proof that an internal neutral state has been overwritten.

\paragraph{A prespecified internal state swap failed its positive control.}\label{app:report-choice:internal}
On \modelid{Mistral-Small-3.1-24B}, a four-state readout fitted on held-out material (the two
non-certification pairs, social contact\,$\to$\,layout and time of day\,$\to$\,color, and three neutral
wordings, never the sentence used above) recovers the non-target
state at the selected site with $.996$ balanced accuracy on held-out development data and $.887$ on
transfer. The readout is rank $3$ at layer $31$ of $40$, the final task-instruction token, selected from four depths and two positions by development balanced accuracy alone. Its four states are unspecified, the two directional poles, and an explicit neutral declaration. Transfer applies the readout unchanged to certification pairs with option-specific wording: unspecified and neutral are each correct on $960/960$ repeated prompts, while every directional error confuses the poles. Gap closure is $[P_{\mathrm{patched}}(\mathrm{donor})-P_{\mathrm{clean}}]/[P_{\mathrm{donor}}-P_{\mathrm{clean}}]$, with $P_{\mathrm{clean}}$ from the unpatched recipient run. The positive control swaps $z=-1$ and $z=+1$ in otherwise identical item, task, target-level, and presentation conditions ($480$ pairs); it required median gap closure $\ge.5$ on report and forced-choice tasks and a rank-matched random subspace below half the fitted effect. The observed median closure is $.000$ on all three tasks, as for the random control. A full-vector replacement at the same site also transferred nothing, at RMS displacement $.0201$ versus $.00497$ for the rank-$3$ edit at unit edit strength. The state is readable at this site but not used there: even full replacement transfers nothing, so
the answer depends on this information, if at all, through other positions or earlier layers. The
unspecified\,$\leftrightarrow$\,neutral swap was not run, and no layer, rank, or position was re-selected
after the control failed.

\paragraph{Full-span non-target key--value swaps and progressive Path~B head blocking.}
Whereas single-token patching at layer $31$ above transfers nothing, replacing key--value caches across the full length-aligned non-target declaration span $\mathcal{S}_Z$ across all layers achieves a median gap closure of $1.000$ across all five models on both binary choice and preference reports. In pathway blocking (Table~\ref{tab:mech_kv}), Path~A masks within-prefix attention $\mathcal{S}_Z \to \mathcal{S}_T$, whereas Path~B masks query-time attention $\{\mathcal{S}_{\text{opt}}, \mathcal{S}_{\text{task}}\} \to \mathcal{S}_T$ from both candidate-option and trailing task-instruction tokens back to the target clause. For progressive Path~B head blocking (Figure~\ref{fig:pathb_dose} and Table~\ref{tab:pathb_dose}), we rank attention heads out of sample via even/odd two-fold cross-fitting: each fold ranks heads on $40$ training non-target items under explicit neutrality and evaluates each blocking budget on the disjoint $40$ non-target items and $20$ independent target items before pooling across both folds. Specifically, head $(l,h)$ is ranked by its mean absolute target-flip direct logit contribution from $\mathcal{S}_T$ along the post-RMSNorm binary option-difference vector $\mathbf{u}_{\mathrm{eff},l}$ at the final prediction token:
\begin{equation}
\operatorname{Score}_{\mathrm{off}}(l,h) = \mathbb{E}_{i \in \mathrm{train}}\!\left[\tfrac{1}{2}\big|w_{l,h}^{\mathcal{S}_T}(i, T{=}{+}1) - w_{l,h}^{\mathcal{S}_T}(i, T{=}{-}1)\big|\right],
\end{equation}
where $w_{l,h}^{\mathcal{S}_T} = \big(\sum_{s \in \mathcal{S}_T} \alpha_{l,h,-1,s}\mathbf{v}_{l,h,s}\big)^{\!\top} W_{O,l,h}^{\top}\mathbf{u}_{\mathrm{eff},l}$.
Matched random-head ($10\%$ heads: pooled $\mathrm{TV}$ $.135$--$.470$, target retention $.970$--$1.000$) and random equal-length prefix-span controls ($\mathrm{TV}$ $.162$--$.499$, retention $.995$--$1.002$) leave drift intact. Table~\ref{tab:pathb_dose} reports the complete per-model pooled non-target $\mathrm{TV}$ and independent target retention across the four attribute pairs ($80$ non-target items and $40$ independent target items per model) plotted in Figure~\ref{fig:pathb_dose}, and Table~\ref{tab:step1_p11_diagnostics} summarizes the corresponding hidden-state cosine alignments, output log-odds biases, and signed $2\times 2$ value-overwrite main effects referenced in Sections~4.1 and~4.2.

\begin{table}[!htb]\centering\footnotesize
\providecommand{\ms}[2]{$#1_{\textcolor{gray}{#2}}$}
\providecommand{\bms}[2]{$\mathbf{#1}_{\textcolor{gray}{#2}}$}
\caption{Dose--retention trade-off under progressive Path~B head blocking versus Value-only overwrite across $80$ non-target and $40$ independent target items per model, reporting pooled non-target $\mathrm{TV}$ with independent target retention in gray subscripts.}
\label{tab:pathb_dose}
\vspace{4.5pt}
\renewcommand{\arraystretch}{1.18}
\setlength{\tabcolsep}{5pt}
\begin{tabularx}{\textwidth}{l >{\hsize=0.92\hsize\centering\arraybackslash}X >{\hsize=0.94\hsize\centering\arraybackslash}X >{\hsize=0.98\hsize\centering\arraybackslash}X >{\hsize=0.98\hsize\centering\arraybackslash}X >{\hsize=1.06\hsize\centering\arraybackslash}X >{\hsize=1.12\hsize\centering\arraybackslash}X}
\toprule
& & \multicolumn{4}{c}{\textbf{Progressive Path~B Head Blocking} $\cdot$ $\mathrm{TV}_{\textcolor{gray}{\text{Ret.}}}$} & \textbf{Value Swap} \\
\cmidrule(lr){3-6}\cmidrule(lr){7-7}
\textbf{Model} & \textbf{Neutral} & \textbf{Top-$5\%$} & \textbf{Top-$10\%$} & \textbf{Top-$20\%$} & \textbf{All Heads} & \textbf{Value-only} \\
\midrule
\modelid{Gemma-2-27B} & \ms{.505}{1.000} & \ms{.327}{.947} & \ms{.286}{.899} & \ms{.153}{.746} & \bms{.141}{.016} & \bms{.004}{\mathbf{.987}} \\
\modelid{Qwen3-32B}   & \ms{.400}{1.000} & \ms{.337}{1.000} & \ms{.204}{.720} & \ms{.073}{.355} & \bms{.049}{.023} & \bms{.004}{\mathbf{.992}} \\
\modelid{Mistral-24B} & \ms{.295}{1.000} & \ms{.126}{.765} & \ms{.061}{.620} & \ms{.062}{.524} & \bms{.049}{.061} & \bms{.036}{\mathbf{.994}} \\
\modelid{OLMo-2-32B}  & \ms{.253}{.986}  & \ms{.230}{.963} & \ms{.177}{.781} & \ms{.145}{.572} & \bms{.053}{.011} & \bms{.024}{\mathbf{.916}} \\
\modelid{Qwen2.5-32B} & \ms{.188}{1.000} & \ms{.110}{.847} & \ms{.077}{.526} & \ms{.049}{.391} & \bms{.034}{.036} & \bms{.024}{\mathbf{1.000}} \\
\midrule
\rowcolor{gray!12}\textbf{5-Model Mean} & \bms{.328}{.997} & \ms{.226}{.904} & \ms{.161}{.709} & \ms{.096}{.518} & \bms{.065}{.029} & \bms{.018}{\mathbf{.978}} \\
\bottomrule
\end{tabularx}
\end{table}

\begin{table}[!htb]\centering\footnotesize
\renewcommand{\arraystretch}{1.14}
\setlength{\tabcolsep}{3.5pt}
\caption{\textbf{Hidden-state alignment trajectories, output log-odds decompositions, and signed $2\times 2$ value-overwrite main effects.} Pre-option ($\mathcal{S}_Z$ probe) versus post-option ($\mathcal{S}_{\text{task}}$) cosine alignment between target-induced activation differences and the non-target choice axis; output log-odds on risk$\to$color with final-layer RMS normalization held fixed ($\bar{\ell}_{z=\pm 1}$ directional bias, $\bar{\ell}_{z=0}$ neutral baseline bias, and target-induced shift $\Delta_T\ell$ under Neutral versus the bipolar mean across $z=\pm 1$ under Directional); and pooled $4$-strata signed item-level main-effect removal under Value-only overwrite (values slightly above $100\%$ reflect tiny sign-reversed residuals near zero).}
\label{tab:step1_p11_diagnostics}
\begin{tabular*}{\linewidth}{@{\extracolsep{\fill}}l ccc cccc c@{}}
\toprule
& \multicolumn{3}{c}{Choice-axis cosine} & \multicolumn{4}{c}{Risk$\to$color log-odds ($\text{nat}$)} & \multicolumn{1}{c}{$2\times 2$ signed} \\
\cmidrule(lr){2-4}\cmidrule(lr){5-8}\cmidrule(l){9-9}
Model & Pre-opt & Color & Hue & $\bar{\ell}_{z=\pm 1}$ & $\bar{\ell}_{z=0}$ & $\Delta_T\ell_{\text{neut}}$ & $\Delta_T\ell_{\text{dir}}$ & Val.\ red. \\
\midrule
\modelid{Gemma-2-27B} & $.046$ & $.808$ & $.764$ & $+11.3 / {-10.9}$ & $+0.62$  & $+7.54$ & $+0.17$ & $100.5\%$ \\
\modelid{Qwen3-32B}   & $.084$ & $.638$ & $.223$ & $+20.6 / {-20.8}$ & $+0.52$  & $+9.85$ & $+0.72$ & $99.2\%$  \\
\modelid{Mistral-24B} & $.096$ & $.765$ & $.321$ & $+6.4 / {-6.2}$   & $-0.31$  & $+2.99$ & $+0.09$ & $98.8\%$  \\
\modelid{OLMo-2-32B}  & $.055$ & $.857$ & $.574$ & $+8.8 / {-8.8}$   & $-1.57$  & $+2.95$ & $+0.36$ & $104.4\%$ \\
\modelid{Qwen2.5-32B} & $.038$ & $.950$ & $.246$ & $+37.1 / {-37.0}$ & $-11.15$ & $+4.38$ & $+0.42$ & $88.5\%$  \\
\bottomrule
\end{tabular*}
\end{table}

\subsection{Mitigation: Supplied Directions and Their Limits}\label{app:c6}
A non-target direction can supply an alternative to the completion; a prohibition supplies none.
The comparisons below test this distinction without treating a replacement direction as neutrality.
\subsubsection{Instructions That Supply a Direction Contain More Than Prohibition or Relabeling}

\paragraph{A scope-and-default instruction contains most informative cells; a prohibition or a domain label contains few.}
Table~\ref{tab:containment} lists the per-instruction outcomes behind Section~\ref{sec:main}: the
scope-and-default instruction contains more of the cross-attribute influence than a negation-only instruction or a prefix
that labels each question as not a risk decision, with the other instructions in between (wordings in
Appendices~\ref{app:prompts:boundary} and~\ref{app:prompts:negation}). The domain-label prefix adds no
system-prompt clause and is applied to rating-scale and A/B questions only, so its two informative
survey cells carry no instruction. The cells cross the two black-box models of the dense audit
(\modelid{GPT-5.5} and \modelid{Kimi-K2.6}) with three personas (risk, flamboyant, frugal; wordings in
Appendix~\ref{app:prompts:boundary}) and eight tasks: $12$ aesthetic seven-point ladders pooled as one task and
five single ladders (message style, greeting warmth, playlist energy, food spice, snack portion) as separate
tasks, each scored as the share of endpoint ratings over $30$ draws; a survey of $15$ five-point preference
statements answered as one array by $60$ simulated respondents (share of extreme responses); and a synthetic A/B
task over $20$ known-zero pairs (share choosing the marked variant; three shown in
Appendix~\ref{app:prompts:knownzero}). Target retention is read on a $34$-item risk-attitude battery. A
model--task--persona cell is informative
when the bare prompt moves the non-target rate by at least $.10$ from a reference rate below $.80$, which
leaves $31$ cells per instruction. A cell counts as contained when the instruction removes at least $70\%$ of
the bare-prompt influence (residual gate $|G|\le.30$ on the logit scale, or probability-scale repair of at least
$.70$ where the reference rate is at most $.02$; Appendix~\ref{app:sr}) while keeping at least $.60$ of the
target effect; over-correction ($G<-.30$) and dilution (influence removed but target retention below
$.60$) count as failures. Containment depends on the task as well as the instruction: on the forced-choice
synthetic A/B task even the scope-and-default instruction contains only $3$ of its $5$ informative cells. On
the logit scale, six synthetic A/B cells whose bare rate reaches $1.0$ count as contained while removing
$15$--$60\%$ of the probability-scale influence (one under scope-and-default). Dropping over-corrections
and dilutions from the denominator would give $.82$, $.60$, $.50$, and $.45$ for the first four rows.
The counts are descriptive, over distinct cells. The flamboyant persona's own wording (``showy, expressive,
attention-drawing options'') covers these style tasks, and Table~\ref{tab:directions} codes its pole as
definitionally aligned, so its cells test whether an instruction's scope overrides a persona's own content. On the
cross-attribute cells of the risk and frugal personas ($18$ informative), scope-and-default and the default
alone each contain $14$, the task-specific instruction $12$, the scope rule alone $8$, the negation-only
instruction $5$, and the domain-label prefix $4$: supplying a direction contains most cross-attribute
influence, prohibition and relabeling little, and the scope rule adds containment where the persona's own
wording covers the task (flamboyant: $9$ against $0$ of $13$).

\begin{table}[!htb]\centering
\renewcommand{\arraystretch}{1.08}\small
\caption{\textbf{Instruction-type containment on informative behavioral cells.} Outcomes of the $31$
informative cells per instruction wording (two black-box models of the dense audit, eight tasks, three personas; distinct
cells) and the contained share.}
\label{tab:containment}
\begin{tabular*}{\linewidth}{@{\extracolsep{\fill}}lccccc@{}}
\toprule
Instruction wording & Contained & Leaks & Over-corrects & Dilutes & Contained share \\
\midrule
Scope $+$ ordinary default & $23$ & $5$ & $3$ & $0$ & $.74$ \\
Task-specific & $15$ & $10$ & $2$ & $4$ & $.48$ \\
Ordinary default only & $14$ & $14$ & $0$ & $3$ & $.45$ \\
Scope only & $13$ & $16$ & $2$ & $0$ & $.42$ \\
Negation-only & $5$ & $26$ & $0$ & $0$ & $.16$ \\
Domain-label prefix & $4$ & $27$ & $0$ & $0$ & $.13$ \\
\bottomrule
\end{tabular*}
\end{table}

\paragraph{Scope without a default.}
The attribute-scope instruction (Scope $-$ none in Table~\ref{tab:promptctl}, under
the neutral declaration in its option-specific wording) lowers neutral-level sensitivity in eight of twelve cells but removes every exceedance of $.15$ in only one. The attribute-scope instruction names both
tested non-target attributes. Its effect is therefore not an unrestricted relevance gate, and it does
not preserve neutrality in the tested panel.

\subsubsection{A Rule-Lookup Reference}\label{app:lookup}

\paragraph{Design.} This exploratory pilot, declared before scoring, asks how low non-target sensitivity
can go when the prompt supplies the decision rule. It is a separate task, not a matched condition of the
prompt-design evaluation, so it gives no matched estimate of adding a rule to the evaluated prompt designs. Each
prompt states both attributes in words, the target leaning and a directional non-target preference, and
then a verbal rule naming the attribute that decides the question: \emph{``Rule for this question: choose
the option that matches the user's stated [attribute]. The user's [other attribute] has no bearing on
this question; do not let it influence the choice.''} The ground truth is a lookup with no base term and no
arithmetic: the correct option carries the stated pole of the named attribute. The \emph{lookup} task
uses the structured worlds' own questions and option texts (the target scenario's for the target
question, the non-target scenario's for the non-target question). The \emph{multi-attribute} task uses
generic scenes whose two options each combine a target pole with the opposite non-target pole; the same
pair serves both questions, so the two rules disagree whenever $t$ and $z$ share a sign, and following
the target on the non-target question is leakage. Each checkpoint sees $64$ structured worlds, $16$ from
each of four attribute pairs (risk weight$\to$color, time preference$\to$sound, social
contact$\to$layout, time of day$\to$color), crossed with directional values only ($t,z\in\{\pm1\}$),
both option orders, and two answer-code maps ($4{,}096$ prompts per checkpoint); the pilot therefore
says nothing directly about the neutral level. Readouts per task and question are rule agreement (the
top choice carries the named attribute's pole), also split by that attribute's pole;
$P(\text{ground truth})$, the probability of the ground-truth option; on the non-target question, sensitivity, the
TV between $t{=}{+}1$ and $t{=}{-}1$ within the same $z$, code map, and order; and on the target
question, the target effect $P(\text{target-positive option}\mid t{=}{+}1)-P(\cdot\mid t{=}{-}1)$.
Intervals are $95\%$ world-bootstrap intervals ($B{=}2{,}000$). Results are read against a $.90$
reference level for rule execution, with no pass/fail criterion. Table~\ref{tab:lookup} reports the lookup task; the multi-attribute task is summarized after it.

\begin{table}[!htb]\centering
\renewcommand{\arraystretch}{1.24}\footnotesize
\caption{\textbf{With the deciding attribute named, the non-target choice is insensitive to the
target.} Lookup task of the explicit-rule pilot ($64$ worlds per checkpoint; directional values only).
Target question: rule agreement overall and by target pole, probability of the ground-truth option, and
target effect. Non-target question: rule agreement and sensitivity (TV between the two target values).
Point estimates; the text gives upper interval endpoints for non-target sensitivity.
Exploratory.}
\label{tab:lookup}
\setlength{\tabcolsep}{4pt}
\begin{tabular*}{\linewidth}{@{\extracolsep{\fill}}lccccccc@{}}
\toprule
& \multicolumn{5}{c}{Target question} & \multicolumn{2}{c}{Non-target question} \\
\cmidrule(lr){2-6}\cmidrule(l){7-8}
& \multicolumn{3}{c}{Rule agreement} & & & & \\
\cmidrule(lr){2-4}
Checkpoint & all & $t{=}{+}1$ & $t{=}{-}1$ & $P(\text{ground truth})$ & Target effect & Rule agreement & Sensitivity \\
\midrule
\modelid{Mistral} & $.979$ & $1.000$ & $.959$ & $.972$ & $.943$ & $1.000$ & $.001$ \\
\modelid{Qwen3} & $.916$ & $.992$ & $.840$ & $.910$ & $.820$ & $.989$ & $.003$ \\
\modelid{OLMo-2} & $.893$ & $.984$ & $.801$ & $.896$ & $.793$ & $.998$ & $.003$ \\
\modelid{Gemma-2} & $.910$ & $.996$ & $.824$ & $.909$ & $.818$ & $.991$ & $.002$ \\
\modelid{Qwen2.5} & $.971$ & $1.000$ & $.941$ & $.972$ & $.944$ & $.999$ & $.002$ \\
\bottomrule
\end{tabular*}
\end{table}

\paragraph{Reading the reference.}
Non-target sensitivity is $.001$--$.003$ with upper $95\%$ interval endpoints at most $.008$,
at target effects $.79$--$.94$. OLMo-2's target ground-truth probability is $.896$, below the $.90$
reference, and negative-pole target rule agreement is $.80$--$.96$ versus $.984$--$1.000$ at
the positive pole. These task errors remain visible in Table~\ref{tab:lookup}. In the multi-attribute task, where the two rules conflict whenever $t$ and $z$ share a sign, non-target rule agreement is $.949$--$1.000$ and non-target sensitivity $.000$--$.100$ (upper $95\%$ endpoints at most $.124$) at target effects $.77$--$.98$. Sensitivity is higher than in the lookup task on four of five checkpoints, but the named rule still decides the non-target choice. As in the lookup task, target errors remain: \modelid{OLMo-2}'s target ground-truth probability is $.886$, and negative-pole target rule agreement is $.79$--$.98$ against $1.000$ at the positive pole. Both tasks name a deciding attribute and supply a direction, so neither tests the neutral level.

\subsubsection{Dialogue Instruction Used in the Conclusion}\label{app:c6:dialogue}
Section~\ref{sec:conclusion} cites this dialogue comparison. It is not matched to the
original field-audit collection or to the containment comparison of Appendix~\ref{app:c6}: instructions, items, and API access can differ.

\paragraph{Black-box panel: over ten turns the repair lowers leakage on all eight models but keeps it below the floor throughout on three.} This check asks whether a system-prompt persona fades over a dialogue \citep{li2024instability} and
whether the scope-and-default repair survives it. Each dialogue has ten
assistant turns: measurement turns $1$, $5$, and $10$ (before any filler, after three filler
exchanges, and after seven), with seven fixed shopping-assistant questions that mention no trait, answered by the
model in between. At each measurement turn the dialogue branches (item answers never re-enter context) into six
forced-choice aesthetic non-target items, each between a default option and an attention-grabbing marked
one (phone wallpaper, ringtone, reading font, lock screen, keyboard theme, app icons), and two target
(risk) items; $n{=}10$ dialogues per condition, four conditions (bare target prompt, scoped target
prompt, bare opposite-pole prompt, no-information reference), $1{,}240$ calls per model. The scoped prompt
appends \emph{``This disposition applies only to decisions that themselves involve risk or money. For
other choices, this person has ordinary, understated preferences.''} Prespecified labels use the bootstrap CI of $\mathrm{leak}(10)/\mathrm{leak}(1)$:
\emph{amplifies} means the CI lies wholly above $1$; \emph{persists} means turn-$10$ leakage passes
the effect-floor-then-BH decision and the ratio CI contains $1$; \emph{washes out} means turn-$10$
leakage fails the floor after turn-$1$ leakage passed. The analyzer additionally labels leakage
that remains above the floor with a ratio CI wholly below $1$ as \emph{persists (attenuating)};
this is the Nemotron case below. For the scoped arm, the repair \emph{persists} only if scoped leakage stays below the
$.05$ floor at all three measurement turns with retention ${\ge}.90$. Results (Table~\ref{tab:armcmt}):
bare-prompt non-target leakage \textbf{persists on $8/8$} models---none washes out; the only attenuation
(\modelid{Nemotron-3-Super}, ratio CI $[.269,.690]$) still leaves turn-$10$ leakage at $.450$, nine times
the floor. In point estimates, the scoped prompt lowers leakage below the bare prompt's on $8/8$ models at every measurement turn, but the repair \textbf{persists on only $3/8$}. On four models scoped leakage is already above the floor at measurement turn $1$ ($.183$--$.433$, against bare $.700$--$.933$); of these, \modelid{Qwen3.5-397B} rises monotonically ($.217{\to}.450{\to}.517$), while \modelid{Kimi-K2.6} and \modelid{MiMo-V2.5-Pro} fall below the floor by turn $10$. \modelid{GPT-5.5} holds at turn $1$ ($.000$) but not at turn $5$ ($.200$). Retention stays $.90$--$1.05$ everywhere, so the shortfall is in containment, not in the target effect. The opposite-pole control never
moves toward the marked pole ($8/8$); its drift is away from it, by up to $-.267$ (\modelid{GLM-5.2},
\modelid{MiMo-V2.5-Pro}). Because measurement turn $1$ precedes any filler, the four turn-$1$ shortfalls occur before any conversation; they therefore do not require accumulated dialogue history. The three passes are one-sided: their scoped leakage moves toward the default pole that the instruction states (to $-.067$, $-.183$, and $-.100$), so they fail an absolute two-sided $.05$ rule, a distinction retained with the analyzer correction in Appendix~\ref{app:repro}.

\begin{table}[!htb]\centering
\renewcommand{\arraystretch}{1.35}
\caption{\textbf{Multi-turn persistence and repair on the eight black-box models (substitute API providers).}
Leakage = reference-adjusted signed shift toward the marked pole on the six non-target items at turns
$1/5/10$ (turn $1$ precedes any filler); ratio = $\mathrm{leak}(10)/\mathrm{leak}(1)$ with item-clustered
bootstrap CI; retention = target rate at turn $10$ (scoped condition) against bare turn $1$. Repair
\emph{persists} = scoped leakage below the $.05$ floor at all three turns and retention ${\ge}.90$;
otherwise the first failing measured turn is shown.}
\label{tab:armcmt}
\footnotesize
\begin{tabularx}{\linewidth}{@{}l l l >{\raggedright\arraybackslash}X c l@{}}
\toprule
Model & Bare $1/5/10$ & Ratio [CI] & Scoped $1/5/10$ & Ret. & Repair \\
\midrule
DeepSeek-V4-Flash & $.800/.800/.817$ & $1.02$ $[.89,1.16]$ & $.333/.317/.367$ & $1.05$ & fails (t1) \\
GPT-5.5 & $1.00/.867/.917$ & $.92$ $[.78,1.00]$ & $.000/.200/.067$ & $1.00$ & fails (t5) \\
Kimi-K2.6 & $.833/.767/.717$ & $.86$ $[.50,1.20]$ & $.433/.117/.000$ & $1.00$ & fails (t1) \\
Gemini-3-Flash-Preview & $1.00/1.00/.933$ & $.93$ $[.83,1.00]$ & $.000/.000/{-}.067$ & $1.00$ & persists \\
GLM-5.2 & $.633/.567/.550$ & $.87$ $[.45,1.37]$ & ${-}.133/{-}.150/{-}.183$ & $1.00$ & persists \\
Qwen3.5-397B & $.933/.933/.917$ & $.98$ $[.94,1.00]$ & $.217/.450/.517$ & $1.00$ & fails (t1) \\
Nemotron-3-Super & $.900/.567/.450$ & $.50$ $[.27,.69]$ & ${-}.100/{-}.017/{-}.083$ & $.90$ & persists \\
MiMo-V2.5-Pro & $.700/.615/.667$ & $.95$ $[.79,1.14]$ & $.183/.183/.017$ & $.95$ & fails (t1) \\
\bottomrule
\end{tabularx}
\end{table}

\subsubsection{Presentation-Randomization Audit on Simulated Investment Decisions}\label{app:c6:epr}

Table~\ref{tab:epr_summary} reports the per-model results of the $60$-item personal investment presentation-randomization audit (Section~3.3; $16{,}120$ confirmatory evaluations per model across aligned, swapped, and uncolored plain presentation conditions). Counterbalancing brochure colors ($\tau_{\text{bal}} = \tfrac{1}{2}(\tau_{\text{aligned}} + \tau_{\text{swapped}})$) leaves $\tau_{\text{bal}}/\tau_{\text{plain}} < 0.90$ on four of five models ($.722$--$.881$) and pulls presentation-balanced retention relative to the color-labeled Bare baseline ($S_3 = \tau_{\text{bal}}(\text{N1})/\tau_{\text{bal}}(\text{bare})$) below $0.90$ on two models (\modelid{OLMo-2-32B} $.804$, \modelid{Gemma-2-27B} $.886$; \modelid{Qwen3-32B} is $.908$ because color labels already depress its $\tau_{\text{bal}}(\text{bare})$ to $.817$). Across all $60$ items, objective comprehension accuracy is $99.15\%$ for return volatility ($\mathrm{CV}$) and $89.22\%$ for expected value ($\mathrm{CM}$), and subjective investment risk judgments ($\mathrm{CP}$) shift by $0.00$ percentage points across brochure colors on four of five models.

\begin{table}[!htb]\centering\footnotesize
\renewcommand{\arraystretch}{1.14}
\setlength{\tabcolsep}{3pt}
\caption{\textbf{Estimated target risk treatment effects across brochure color pairings under explicit color neutrality.} Target risk effect $\tau$ is the high-minus-low risk persona difference in choosing the higher-volatility fund across $60$ items; full gap is $\tau_{\text{aligned}} - \tau_{\text{swapped}} = 2\Delta(\text{N1})$, $\tau_{\text{bal}} = \tfrac{1}{2}(\tau_{\text{aligned}} + \tau_{\text{swapped}})$, and $S_3 = \tau_{\text{bal}}(\text{N1}) / \tau_{\text{bal}}(\text{bare})$.}
\label{tab:epr_summary}
\resizebox{\linewidth}{!}{%
\begin{tabular}{l cccc cc cc cc}
\toprule
& \multicolumn{4}{c}{Explicit color neutrality ($\text{N1}$)} & \multicolumn{2}{c}{Balanced retention} & \multicolumn{2}{c}{Robustness gaps} & \multicolumn{2}{c}{Checks} \\
\cmidrule(lr){2-5}\cmidrule(lr){6-7}\cmidrule(lr){8-9}\cmidrule(l){10-11}
Model & $\tau_{\text{plain}}$ & $\tau_{\text{aligned}}$ & $\tau_{\text{swapped}}$ & Full gap & $\tau_{\text{bal}}/\tau_{\text{plain}}$ & $S_3$ ($\tau_{\text{bare,bal}}$) & Equal-EV & Anti-coup.\ ($\text{N2}$) & Neut.\ rep. & Risk shift \\
\midrule
\modelid{OLMo-2-32B}  & $.990$ & $.982$ & $.446$ & $+53.6\text{ pp}$ & $.722$ & $.804\;(.889)$ & $+37.7\text{ pp}$ & $+52.4\text{ pp}$ & $73.0\%$  & $+9.2\text{ pp}$ \\
\modelid{Qwen3-32B}   & $.955$ & $.991$ & $.492$ & $+49.9\text{ pp}$ & $.777$ & $.908\;(.817)$ & $+36.2\text{ pp}$ & $+30.8\text{ pp}$ & $97.5\%$  & $0.0\text{ pp}$ \\
\modelid{Gemma-2-27B} & $.970$ & $.938$ & $.682$ & $+25.6\text{ pp}$ & $.835$ & $.886\;(.913)$ & $+24.5\text{ pp}$ & $+10.0\text{ pp}$ & $97.0\%$  & $0.0\text{ pp}$ \\
\modelid{Qwen2.5-32B} & $1.000$ & $1.000$ & $.984$ & $+1.6\text{ pp}$  & $.992$ & $.995\;(.997)$ & $+0.6\text{ pp}$  & $+0.0\text{ pp}$  & $100.0\%$ & $0.0\text{ pp}$ \\
\modelid{Mistral-24B} & $.972$ & $.840$ & $.872$ & $-3.2\text{ pp}$  & $.881$ & $.946\;(.905)$ & $-0.5\text{ pp}$  & $-3.5\text{ pp}$  & $64.6\%$  & $+0.1\text{ pp}$ \\
\midrule
\rowcolor{gray!12}\textbf{5-Model Mean} & $\mathbf{.977}$ & $\mathbf{.950}$ & $\mathbf{.695}$ & $\mathbf{+25.5\text{ pp}}$ & $\mathbf{.842}$ & $\mathbf{.910\;(.904)}$ & $\mathbf{+19.7\text{ pp}}$ & $\mathbf{+17.9\text{ pp}}$ & $\mathbf{86.4\%}$ & $\mathbf{+1.9\text{ pp}}$ \\
\bottomrule
\end{tabular}%
}
\end{table}

\subsection{Provenance, Corrections, and Analysis Status}\label{app:repro}
This appendix retains the provenance limitations and deviations that affect the main claims.
Experimental caches and runners are not part of this document.

\subsubsection{API Collection and Serving Limits}
The full-bank cache records $26{,}880$ valid responses per black-box model (three prompt conditions, four factor cells,
$94$ non-target and $18$ target items, $20$ draws; Table~\ref{tab:repro}); this is not the smaller headline item set's sample
count. Temperature is omitted and inherits the provider default, with a recorded $128$-token limit.
Option order uses deterministic client seeds; the APIs expose no sampling seed or checkpoint hash.
Request-side model IDs and API providers, not exact weight revisions, identify the panel.
\begin{table}[!htb]\centering
\renewcommand{\arraystretch}{1.08}
\caption{\textbf{Full-bank cache provenance.} Temperature is recorded as \nolinkurl{null} (parameter not
sent), the recorded output limit is $128$ tokens, and the system prompt is fixed across rows. Think: the
thinking setting recorded in the cache (notes below).}
\label{tab:repro}
\footnotesize
\setlength{\tabcolsep}{3pt}
\begin{tabularx}{\linewidth}{@{}l >{\raggedright\arraybackslash}p{0.25\linewidth} Y c r@{}}
\toprule
Model & Provider-facing ID & API provider & Think & Valid records \\
\midrule
DeepSeek-V4-Flash & \nolinkurl{deepseek-v4-flash} & DeepSeek API & on & $26{,}880$ \\
GPT-5.5 & \nolinkurl{gpt-5.5} & commercial gateway A & on & $26{,}880$ \\
Kimi-K2.6 & \nolinkurl{kimi-k2.6} & commercial gateway A & off & $26{,}880$ \\
GLM-5.2 & \nolinkurl{glm-5.2:cloud} & Ollama API & off & $26{,}880$ \\
Nemotron-3-Super & \nolinkurl{nemotron-3-}\allowbreak\nolinkurl{super:cloud} & Ollama API & off & $26{,}880$ \\
Qwen3.5-397B & \nolinkurl{qwen3.5:397b-cloud} & Ollama API & off & $26{,}880$ \\
MiMo-V2.5-Pro & \nolinkurl{mimo-v2.5-pro} & commercial gateway B & default & $26{,}880$ \\
Gemini-3-Flash-Preview & \nolinkurl{gemini-3-}\allowbreak\nolinkurl{flash-preview} & Ollama API & off & $26{,}880$ \\
\bottomrule
\end{tabularx}
\end{table}

\paragraph{Decoding and anonymity.}
Commercial gateways A--D are anonymized. A serving gateway does not identify weight openness.
``Off'' denotes a sent disable flag; ``on/default'' denotes no such flag, not a verified internal
reasoning mode. MiMo's caches record thinking enabled with no disable flag, while the runner's
default is disabled thinking with a larger output limit; invocation flags were not serialized, so
this discrepancy cannot be resolved at the per-record level.

\paragraph{Provider substitution.}
Original API access later became unavailable for six models. New calls use commercial gateways C/D, including for the two models whose original access remained available; among reported results, only the ten-turn dialogue comparison (Appendix~\ref{app:c6:dialogue}) uses them. The backend that serves an account can differ
even on a shared gateway. Some substitute providers also change available thinking controls.
Comparisons within each serving configuration (provider and backend) are retained, but comparisons to the original cached results
use different serving configurations and are not interpreted as model drift. Only GPT-5.5 and Kimi-K2.6 returned
response-side version strings in a version probe; the other six rely on request strings
and access windows.

\subsubsection{Endpoint Classification and Between-Model Heterogeneity}\label{app:endpoints}

\begin{table}[!htb]\centering
\renewcommand{\arraystretch}{1.08}
\caption{\textbf{Endpoint classification of the eight black-box models} (descriptive metadata; not
used for group inference). Standardized surgicality ratio, bare condition with the $R$ clause first, with $95\%$
bootstrap intervals, over the complete designed panel ($26{,}880/26{,}880$ valid records per model).}
\label{tab:endpoints}
\small
\begin{tabular*}{\linewidth}{@{\extracolsep{\fill}}l l r l@{}}
\toprule
Model & Endpoint type & Std.\ ratio & CI \\
\midrule
GLM-5.2 & open-weight via API & $.022$ & $[-.017, .069]$ \\
DeepSeek-V4-Flash & open-weight via API & $.027$ & $[.004, .056]$ \\
GPT-5.5 & proprietary closed API & $.051$ & $[.007, .116]$ \\
MiMo-V2.5-Pro & open-weight via API & $.057$ & $[.028, .087]$ \\
Qwen3.5-397B & open-weight via API & $.104$ & $[.025, .202]$ \\
Gemini-3-Flash-Preview & proprietary closed API & $.137$ & $[.038, .233]$ \\
Kimi-K2.6 & open-weight via API & $.153$ & $[.030, .285]$ \\
Nemotron-3-Super & open-weight via API & $.163$ & $[.082, .246]$ \\
\midrule
Random-effects pool ($8$) & pooled & $.068$ & $[.037, .099]$ \\
\bottomrule
\end{tabular*}
\end{table}

\paragraph{Between-model heterogeneity and prediction intervals.} DerSimonian--Laird heterogeneity
statistics for the three pooled conditions of Appendix~\ref{app:c4:residual} (each an $R\!\to\!S$ ratio under one prompt form): bare
$R$-then-$S$ clause order $\tau_{\mathrm{DL}}{=}.033$, $I^2{=}63\%$, $Q(7){=}19.1$, $p{=}.008$; bare
$S$-then-$R$ order $\tau_{\mathrm{DL}}{=}.009$, $I^2{=}15\%$, $Q(7){=}8.3$, $p{=}.31$; structured table
$\tau_{\mathrm{DL}}{=}.028$, $I^2{=}85\%$, $Q(7){=}46.3$, $p{<}.0001$. The corresponding $95\%$ prediction
intervals are $[-.021,.157]$, $[.025,.085]$, and $[-.017,.136]$: the pooled residual is positive on
average across the designed panel, but for the two heterogeneous conditions a new model family could plausibly
show a near-zero residual, so the pooled point values should not be read \mbox{as a universal constant.}

\subsubsection{Open-Weight Checkpoints and Scoring}
\begin{table}[!htb]\centering
\renewcommand{\arraystretch}{1.08}
\caption{\textbf{Open-weight checkpoint registry} (Hugging Face repository ids; display names as used in
figures and text; the five primary checkpoints first). All checkpoints are scored at release precision:
bf16 throughout.}
\label{tab:owcheckpoints}
\footnotesize
\begin{tabularx}{\linewidth}{@{}l Y Y@{}}
\toprule
Display name & Instruct checkpoint & Base checkpoint \\
\midrule
Qwen3-32B & \nolinkurl{Qwen/Qwen3-32B} & none public at this scale; other families' bases predict it ($R^2{=}.658$) \\
Qwen2.5-32B & \nolinkurl{Qwen/Qwen2.5-32B-Instruct} & \nolinkurl{Qwen/Qwen2.5-32B} \\
Mistral-Small-3.1-24B & \texttt{mistralai/\allowbreak Mistral-Small-\allowbreak 3.1-24B-\allowbreak Instruct-2503} &
\texttt{mistralai/\allowbreak Mistral-Small-\allowbreak 3.1-24B-\allowbreak Base-2503} \\
Gemma-2-27B & \nolinkurl{google/gemma-2-27b-it} & \nolinkurl{google/gemma-2-27b} \\
OLMo-2-32B & \nolinkurl{allenai/OLMo-2-0325-32B-Instruct} & \nolinkurl{allenai/OLMo-2-0325-32B} \\
\midrule
Gemma-3-27B & \nolinkurl{google/gemma-3-27b-it} & \nolinkurl{google/gemma-3-27b-pt} \\
Llama-3.1-8B & \nolinkurl{meta-llama/Llama-3.1-8B-Instruct} & \nolinkurl{meta-llama/Llama-3.1-8B} \\
GPT-OSS-20B & \nolinkurl{openai/gpt-oss-20b} & none released \\
Gemma-4-31B & \nolinkurl{google/gemma-4-31B-it} & \nolinkurl{google/gemma-4-31B} \\
Qwen3.5-9B & \nolinkurl{Qwen/Qwen3.5-9B} & \nolinkurl{Qwen/Qwen3.5-9B-Base} (base not used; the checkpoint is outside the 24--32B primary panel and distinct from
\modelid{Qwen3.5-397B} in the black-box panel) \\
\bottomrule
\end{tabularx}
\end{table}
\paragraph{Implementation of the open-weight readout.} Checkpoints are loaded with Hugging Face
\texttt{transformers} in \texttt{bfloat16}, with left padding and single-GPU placement unless memory forces
\texttt{device\_map=auto}. The semantic-control and representation analyses use eager attention for the Gemma and GPT-OSS families and SDPA elsewhere; the readout-refit analysis loads
all five of its checkpoints with eager attention (Appendix~\ref{app:s4}). The seven-point readout is a
next-token logit read at the answer position over the token ids of the digits $1$--$7$ (both bare and
space-prefixed variants), taken after the chat template with the answer prefix \texttt{\{"choice": };
base checkpoints use the raw-LM seven-point readout with no chat template. Nothing is sampled. Geometry activations
are cached at every decoder layer at six aligned positions in \texttt{float16}, and in \texttt{float32}
for every Gemma checkpoint, whose late residual stream overflows \texttt{float16}. The prompt-design
evaluation and the report--choice comparison read answer log-probabilities over the legal responses
instead (Appendices~\ref{internal-control-evidence-and-numerical-scope} and~\ref{app:report-choice}); the
prompt-design evaluation's recorded environment is one NVIDIA RTX PRO 6000 Blackwell GPU (96\,GB),
\texttt{bfloat16}, PyTorch 2.11.0 (CUDA 12.8 build), and \texttt{transformers} 5.10.2.

\subsubsection{Corrections and Deviations That Qualify the Claims}
\paragraph{Explicit-evidence aggregation.}
The bank was declared exploratory before scoring (2026-09-11), with no pass/fail rule. Its first
summary omitted answer-code map from the grouping key, overwriting one map. The 2026-09-12
correction retains both maps within each world, is invariant to record order, and uses the original
$64$ worlds without new calls. Declared-value margin fractions change by at most $.012$.
Tables~\ref{tab:explicit-evidence} and~\ref{tab:rho-means} use the corrected aggregation.

\paragraph{Neutral wording and independent items.}
Table~\ref{tab:wording} uses counts of texts above $.15$, correcting an earlier condition-level share under that label; mean TV and thresholds are unchanged. Independent-item scores were resumed from hash-checked completed records, not counted twice. Qwen3.5-9B was scored in a separate run; its addition was not in the original structured-text or independent-item declarations; it has the lowest report correctness in both sets ($.892$, $.901$) and leaves the open-frame independent-item neutral and directional ranges unchanged; in the battery-frame bridge it is the one checkpoint whose directional comparisons often exceed $.15$ ($64$ of $320$). The
neutrality level split is post hoc, and the structured-text comparison adds model coverage rather
than new texts.

\paragraph{Containment and dialogue rules.}
The containment counts use a reconciled cell rule ($|G|\le.30$, $.70$ repair where the reference rate
is at most $.02$, a $.10$ informative floor, and the prespecified $.60$ retention floor) that was set
after the two earlier planned classifications existed. The probability-scale classifier requires
repair of at least $.70$ and retention of at least $.60$; it excludes absolute persona shifts below $.10$
and reference rates at least $.80$, and calls an opposite-sign residual of magnitude at least $.05$
an over-correction. The logit-gate classifier uses $|G|\le.25$, a minimum absolute persona signal of
$.25$ logit units, and the same retention floor; it excludes saturated reference rates and saturated
persona rates with probability shifts below $.10$. The reconciled counts are therefore
\emph{descriptive}; two re-analysis files carrying the same cell statistics are collapsed to distinct cells. The dialogue
analyzer was corrected from an absolute to the prespecified signed $.05$ floor after interim
numbers had been seen; both readings are retained and differ only on the three models whose scoped leakage moves toward the default pole (Table~\ref{tab:armcmt}).

\paragraph{Representation analysis.}
The program was specified before its first run (2026-09-02), with later changes:
reference-centered erasure replaced uncentered erasure after a smoke test damaged the no-information
reference run; the side-effect readout was moved to the punctual-persona run; uninformative principal angles were dropped and the trivially met decodability criterion replaced; Gemma-2 activations were recollected in float32 after float16 overflow. The random-control
check excluding Gemma-4-31B-IT (random CEP $\le.2$, RET $\ge.8$) was added \emph{after that checkpoint's run}; its random frame scores $.85/{-}.49$ against at most $.01$ CEP on every other checkpoint, so the exclusion does not hinge on these cutoffs. Fixed rank/band, binary-transfer, and collateral-utility rules
were committed before full-panel jobs, but some model revision fields were recovered afterward.
Post-run hash manifests and deterministic reconstructions do not retroactively preregister runs.

\paragraph{Reproduction materials.}
Recomputing statistics requires fixed item banks and splits, prompt renderers, selection rules,
answer scores or sufficient statistics, and aggregation scripts. Cached-score recomputation differs from new collection, which for six of the eight black-box models cannot reuse the original serving configuration, and from activation-level experiments, which require separate caches.

\subsection{Limits of the Conclusions}\label{app:limits}
\paragraph{Scope and external validity.}
Cross-model field prevalence rests on one canonical attribute pair; the other pairs are tested on open-weight checkpoints. Neutral-declaration, report--choice, and internal results also cover open-weight checkpoints only. They score exact answer-token probabilities for every world or item, and the internal results also read hidden states; black-box APIs expose no hidden states and do not uniformly expose exact answer probabilities. Extending the neutral-declaration test to the black-box panel therefore requires a sampled-choice estimator with its own calibration, because sampling error enters the sensitivity estimate at the $.15$ threshold and the log-odds decomposition that separates a suppressed target effect from output saturation (Appendix~\ref{bd-decomposition}) cannot be applied at comparable precision. On the black-box panel, the full-bank factorial states an ordinary default for the non-target attribute (``infer ordinary neutral everyday preferences''; Appendix~\ref{app:prompts:factorized}) rather than a declaration of no leaning, which leaves a residual influence above the declared $\epsilon{=}.015$ on all eight models (Appendix~\ref{app:c4:residual}; Table~\ref{tab:criterion}). The construct specification is author-coded; stricter and adversarial recodings
keep the two-model audit's risk--control gap at or above $.87$ (Appendix~\ref{app:c1}). All materials are English and
text-only. We use no human benchmark, because selectivity is defined by what the prompt states rather
than by how people behave, and we evaluate no demographic population model or unrestricted task
distribution. The full-bank items are LLM-written and judge-normed, so the measured influence is established for such items; their generator's panel membership shows no sign of inflating its own result (Appendix~\ref{app:prompts:probes}). Pole norming does not establish human validity, which the selectivity measures do not require. Simulated
choices are not evidence that real risk-seeking people share the inferred preferences.

\paragraph{Neutrality is an invariance claim.}
A neutral declaration differs from leaving an attribute unstated or directing it toward an ordinary/default pole.
At the structured world's neutral level, an unstated base constant determines the ground-truth answer,
making it unidentifiable from the prompt. Neutrality therefore concerns target-induced movement,
not fidelity to that hidden answer. Forced choice and the indifferent-option task have different answer
spaces; the indifferent option may reduce both measured sensitivity and the target effect, and correct
reporting need not imply preserved choice. 

\paragraph{Generalization and statistical status.}
World-level holdout reuses texts; the independent items reproduce the level pattern on new texts, within
controlled generator families and without cluster intervals for single-family quantities. Their
directional bound is open-frame-specific: in the battery frame, directional exceedances stay at most $4$
of $320$ on four checkpoints but reach $64$ on \modelid{Qwen3.5-9B}. Paraphrase and tie-break
confirmation covers two checkpoints, and the option-specific wording six; report/choice remains
exploratory. Tail tolerances were fixed before held-out evaluation, not preregistered, and nominal
independent-world coverage is unestablished; post hoc checks keep the verdict at any threshold below
$.95$ and with texts as clusters, but create no new confirmatory guarantee.

\paragraph{Internal and mitigation claims.}
Dependence is established for the tested centered question-to-answer erasures in a late rank/band region
on the seven-point readout; transported to binary choice, the own-subspace control passes on one of five
checkpoints (Appendix~\ref{app:s4}), so this is not a universal causal decomposition. Locally selective
constructed subspaces need not be persona-induced or preserve general utility. Readout refits include no
neutral declaration, and the direct neutral-state intervention failed its positive control
(Appendix~\ref{app:report-choice:internal}). Directions, defaults, and explicit rules can limit
cross-attribute influence without preserving an absence of preference. These results neither prove
impossibility nor identify a complete mechanism for the neutrality gap.

\end{document}